%% file: paper_arxiv.tex
\documentclass{article}

\RequirePackage{fancyhdr}
\RequirePackage{color}
\RequirePackage{algorithm}
\RequirePackage{algorithmic}
\RequirePackage{natbib}
\RequirePackage{forloop}

\usepackage[margin=1in]{geometry}
\usepackage{fancyhdr}
\input{imports}

\title{Training Neural Networks for and by Interpolation}

\author{Leonard Berrada$^1$\footnote{Work performed while at University of Oxford. Corresponding address: \texttt{lberrada@google.com}.} , Andrew Zisserman$^2$ and M. Pawan Kumar$^{2}$ \\
  $^1$DeepMind, London\\
  $^2$Department of Engineering Science, University of Oxford\\
}

\begin{document}

\maketitle

\input{core_arxiv}

\bibliography{../bibliography/standardstrings,../bibliography/oval}
\bibliographystyle{icml2020}

\pagebreak
\appendix
\input{appendix}

\end{document}

%% file: imports.tex
\input{math_commands.tex}

\usepackage[utf8]{inputenc} % allow utf-8 input
\usepackage[T1]{fontenc}    % use 8-bit T1 fonts
\usepackage{hyperref}       % hyperlinks
\usepackage{url}            % simple URL typesetting
\usepackage{booktabs}       % professional-quality tables
\usepackage{amsfonts}       % blackboard math symbols
\usepackage{nicefrac}       % compact symbols for 1/2, etc.
\usepackage{microtype}      % microtypography

\usepackage{amsmath}
\usepackage{amsfonts}
\usepackage{amssymb}
\usepackage{mleftright}
\usepackage{xparse}

\usepackage{float}
\usepackage{algorithm}
\usepackage{graphicx}
\usepackage{subcaption}
\usepackage{dirtytalk}
\usepackage{multirow}
\usepackage{bbm}
\usepackage{bm}
\usepackage{mdframed}
\usepackage{wrapfig}

\usepackage{url}
\usepackage{booktabs}
\usepackage{amsfonts}
\usepackage{nicefrac}
\usepackage{microtype}
\usepackage{natbib}
\usepackage{apalike}
\usepackage{dsfont}
\usepackage{xcolor}
\usepackage{siunitx}
\usepackage{appendix}
\usepackage{xspace}

\usepackage{tikz}
\usepackage{pgfplots}
\pgfplotsset{compat=1.12}
\usepackage{adjustbox}
\usepgfplotslibrary{groupplots}

\allowdisplaybreaks

\newcolumntype{R}[2]{%
    >{\adjustbox{angle=#1,lap=\width-(#2)}\bgroup}%
    l%
    <{\egroup}%
}
\setcitestyle{citesep={,}}  % separate multiple citations with a comma

\usepackage{float}
\usepackage{stmaryrd}
\usepackage{amsthm}

\usepackage{thmtools}
\usepackage{thmbox}
\usepackage{thm-restate}

\newmdenv[topline=false,rightline=false,bottomline=false,nobreak=false]{leftlinebox}

\declaretheorem[name=Theorem]{theorem}

\declaretheorem[name=Definition]{definition}

\declaretheorem[name=Lemma]{lemma}

\NewDocumentCommand{\evalat}{sO{\big}mm}{%
  \IfBooleanTF{#1}
   {\mleft. #3 \mright|_{#4}}
   {#3#2|_{#4}}%
}

\newcommand{\eqcomment}[1]{\quad \text{(#1)}}

\newcommand{\w}{\mathbf{w}}

\newcommand{\wbar}{\underline{\mathbf{w}}}

\newcommand{\wstar}{\mathbf{w_\star}}

\newcommand{\Z}{\mathcal{Z}}

\newcommand{\fstar}{f_\star}

\newcommand{\vwhat}{\hat{\vw}}

\newsavebox{\leftbox}
\newsavebox{\rightbox}

\newcommand{\Arrow}[1]{%
\parbox{#1}{\tikz{\draw[->](0,0)--(#1,0);\draw[<-](0,-15)--(#1,-15);}}
}

%% file: math_commands.tex
\usepackage{amsmath,amsfonts,bm}

\def\eqref#1{equation~\ref{#1}}
\def\1{\bm{1}}

\def\vu{{\bm{u}}}
\def\vv{{\bm{v}}}
\def\vw{{\bm{w}}}

\DeclareMathAlphabet{\mathsfit}{\encodingdefault}{\sfdefault}{m}{sl}
\SetMathAlphabet{\mathsfit}{bold}{\encodingdefault}{\sfdefault}{bx}{n}

\newcommand{\E}{\mathbb{E}}

\DeclareMathOperator*{\argmin}{argmin}

%% file: core_arxiv.tex
\section*{Abstract}
\addcontentsline{toc}{section}{Abstract}

In modern supervised learning, many deep neural networks are able to interpolate the data: the empirical loss can be driven to near zero on all samples simultaneously.
In this work, we explicitly exploit this interpolation property for the design of a new optimization algorithm for deep learning, which we term Adaptive Learning-rates for Interpolation with Gradients (ALI-G).
ALI-G retains the two main advantages of Stochastic Gradient Descent (SGD), which are (i) a low computational cost per iteration and (ii) good generalization performance in practice.
At each iteration, ALI-G exploits the interpolation property to compute an adaptive learning-rate in closed form.
In addition, ALI-G clips the learning-rate to a maximal value, which we prove to be helpful for non-convex problems.
Crucially, in contrast to the learning-rate of SGD, the maximal learning-rate of ALI-G does not require a decay schedule, which makes it considerably easier to tune.
We provide convergence guarantees of ALI-G in various stochastic settings.
Notably, we tackle the realistic case where the interpolation property is satisfied up to some tolerance.
We provide experiments on a variety of architectures and tasks:
(i) learning a differentiable neural computer;
(ii) training a wide residual network on the SVHN data set;
(iii) training a Bi-LSTM on the SNLI data set;
and (iv) training wide residual networks and densely connected networks on the CIFAR data sets.
ALI-G produces state-of-the-art results among adaptive methods, and even yields comparable performance with SGD, which requires manually tuned learning-rate schedules.
Furthermore, ALI-G is simple to implement in any standard deep learning framework and can be used as a drop-in replacement in existing code.

\section{Introduction}
Training a deep neural network is a challenging optimization problem: it involves minimizing the average of many high-dimensional non-convex functions.
In practice, the main algorithms of choice are Stochastic Gradient Descent (SGD) \citep{Robbins1951} and adaptive gradient methods such as AdaGrad \citep{Duchi2011} or Adam \citep{Kingma2015}.
It has been observed that SGD tends to provide better generalization performance than adaptive gradient methods \cite{Wilson2017}.
However, the downside of SGD is that it requires the manual design of a learning-rate schedule, which is widely regarded as an onerous and time consuming task.
In this work, we alleviate this issue with the design of an adaptive learning-rate algorithm that needs minimal tuning for good performance.
Indeed, we postulate that by using the same descent direction as SGD while automatically adapting its learning-rate, the resulting algorithm can offer similar generalization performance while requiring considerably less tuning.

In this work, we build on the following two ideas.
First, an adaptive learning-rate can be computed for the non-stochastic gradient direction when the minimum value of the objective function is known \citep{Polyak1969,Shor1985,Brannlund1995,Nedic2001, Nedic2001a}.
And second, one such minimum value is usually approximately known for interpolating models: for instance, it is close to zero for a model trained with the cross-entropy loss.
By carefully combining these two ideas, we create a stochastic algorithm that (i) provably converges fast in convex or Restricted Secant Inequality (RSI) settings, and (ii) obtains state-of-the-art empirical results with neural networks.
We refer to this algorithm as Adaptive Learning-rates for Interpolation with Gradients (ALI-G).

Procedurally, ALI-G is close to many existing algorithms, such as Deep Frank-Wolfe \citep{Berrada2019}, {\sc aProx} \citep{Asi2019} and $L_4$ \citep{Rolinek2018}.
And yet uniquely, thanks to its careful design and analysis, the learning-rate of ALI-G effectively requires a single hyper-parameter that does not need to be decayed.
Since ALI-G is easy to implement in any deep learning framework, we believe that it can prove to be a practical and reliable optimization tool for deep learning.

\paragraph{Contributions.}
We summarize the contributions of this work as follows: \\
- We design an adaptive learning-rate algorithm that uses a single hyper-parameter and does need any decaying schedule.
In contrast, the closely related {\sc aProx} \citep{Asi2019} and $L_4$ \citep{Rolinek2018} use respectively two and four hyper-parameters for their learning-rate.
\\
- We provide convergence rates of ALI-G in various stochastic convex settings.
Importantly, our theoretical results take into account the error in the estimate of the minimum objective value.
To the best of our knowledge, our work is the first to establish convergence rates for interpolation with approximate estimates. \\
- We prove that using a maximal learning-rate helps with convergence for a class of non-convex problems. \\
- We demonstrate state-of-the-art results for ALI-G on learning a differentiable neural computer; training variants of residual networks on the SVHN and CIFAR data sets; and training a Bi-LSTM on the Stanford Natural Language Inference data set.

\section{The Algorithm}

\subsection{Problem Setting}
\label{sec:pb_setting}

\paragraph{Loss Function.}
We consider a supervised learning task where the model is parameterized by $\vw \in \mathbb{R}^p$.
Usually, the objective function can be expressed as an expectation over $z \in \Z$, a random variable indexing the samples of the training set:
\begin{equation}
    f(\vw) \triangleq \mathbb{E}_{z \in \Z}[\ell_z(\vw)],
\end{equation}
where each $\ell_z$ is the loss function associated with the sample $z$.
We assume that each $\ell_z$ is non-negative, which is the case for the large majority of loss functions used in machine learning.
For instance, suppose that the model is a deep neural network with weights $\vw$ performing classification.
Then for each sample $z$, $\ell_z(\vw)$ can represent the cross-entropy loss, which is always non-negative.
Other non-negative loss functions include the structured or multi-class hinge loss, and the $\ell_1$ or $\ell_2$ loss functions for regression.

\paragraph{Regularization.}
It is often desirable to employ a regularization function $\phi$ in order to promote generalization.
In this work, we incorporate such regularization as a constraint on the feasible domain: $\Omega = \left\{ \vw \in \mathbb{R}^p: \ \phi(\vw) \leq r \right\}$ for some value of $r$.
In the deep learning setting, this will allow us to assume that the objective function can be driven close to zero without unrealistic assumptions about the regularization.
Our framework can handle any constraint set $\Omega$ on which Euclidean projections are computationally efficient.
This includes the feasible set induced by $\ell_2$ regularization: $\Omega = \left\{ \vw \in \mathbb{R}^p: \  \| \vw \|_2^2 \leq r \right\}$, for which the projection is given by a simple rescaling of $\vw$.
Finally, note that if we do not wish to use any regularization, we define $\Omega = \mathbb{R}^p$ and the corresponding projection is the identity.

\paragraph{Problem Formulation.}
The learning task can be expressed as the problem $(\mathcal{P})$ of finding a feasible vector of parameters $\wstar \in \Omega$ that minimizes $f$:
\begin{equation} \tag{$\mathcal{P}$} \label{eq:main_problem}
    \wstar \in \argmin\limits_{\vw \in \Omega} f(\vw).
\end{equation}
Also note that $\fstar$ refers to the minimum value of $f$ over $\Omega$: $\fstar \triangleq \min_{\vw \in \Omega} f(\vw)$.

\paragraph{Interpolation.} 
We say that the problem (\ref{eq:main_problem}) satisfies the interpolation assumption if there exist a solution $\wstar$ that simultaneously minimizes all individual loss functions:
\begin{equation}
    \forall z \in \Z, \: \ell_z(\wstar) = 0.
\label{eq:def_interpolation}
\end{equation}
The condition (\ref{eq:def_interpolation}) can be equivalently expressed as $\fstar = 0$.
We also point out that in some cases, it can be more realistic to relax (\ref{eq:def_interpolation}) to $\forall z \in \Z, \: \ell_z(\wstar) \leq \varepsilon$ for a small positive $\varepsilon$.

\subsection{The Polyak Step-Size}
\label{sec:polyak_step_size}

Before outlining the ALI-G algorithm, we begin with a brief description of the Polyak step-size, from which ALI-G draws some fundamental ideas.

\paragraph{Setting.}
We assume that $\fstar$ is known and we use non-stochastic updates: at each iteration, the full objective $f$ and its derivative are evaluated.
We denote by $\nabla f(\vw)$ the first-order derivative of $f$ at $\vw$ (e.g. $\nabla f(\vw)$ can be a sub-gradient or the gradient).
In addition, $\| \cdot \|$ is the standard Euclidean norm in $\mathbb{R}^p$, and $\Pi_{\Omega}(\vw)$ is the Euclidean projection of the vector $\vw \in \mathbb{R}^p$ on the set $\Omega$.

\paragraph{Polyak Step-Size.}
At time-step $t$, using the Polyak step-size \citep{Polyak1969,Shor1985,Brannlund1995,Nedic2001,Nedic2001a} yields the following update:
\begin{equation} \label{eq:polyak_step}
    \vw_{t+1} = \Pi_\Omega\left(\vw_{t} - \gamma_t \nabla f(\vw_t) \right), \text{ where } \gamma_t \triangleq \tfrac{f(\vw_t) - \fstar}{\|\nabla f(\vw_t) \|^2},
\end{equation}
where we loosely define $\frac{0}{0}=0$ for simplicity purposes.

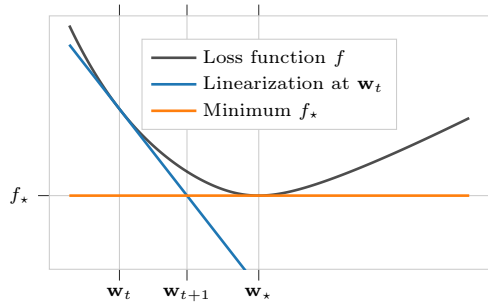
\begin{figure}[H]
\centering
\scriptsize
\input{include/plot_step_simple_arxiv.tex}
\caption{\em Illustration of the Polyak step-size in 1D. In this case, and further assuming that $\fstar = 0$, the algorithm coincides with the Newton-Raphson method for finding roots of a function.}
\label{fig:nr_step_simple}
\end{figure}

\paragraph{Interpretation.}
It can be shown that $\vw_{t+1}$ lies on the intersection between the linearization of $f$ at $\vw_t$ and the horizontal plane $z=\fstar$ (see Figure \ref{fig:nr_step_simple}, more details in Proposition 1 in the appendix).
Note that since $\fstar$ is the minimum of $f$, the Polyak step-size $\gamma_t$ is necessarily non-negative.

\paragraph{Limitations.}
Equation (\ref{eq:polyak_step}) has two major short-comings that prevent its applicability in a machine learning setting.
First, each update requires a full evaluation of $f$ and its derivative.
Stochastic extensions have been proposed in \citep{Nedic2001,Nedic2001a}, but they still require frequent evaluations of $f$.
This is expensive in the large data setting, and even computationally infeasible when using massive data augmentation.
Second, when applying this method to the non-convex setting of deep neural networks, the method sometimes fails to converge.

Therefore we would like to design an extension of the Polyak step-size that (i) is inexpensive to compute in a stochastic setting (e.g. with a computational cost that is independent of the total number of training samples), and (ii) converges in practice when used with deep neural networks.
The next section introduces the ALI-G algorithm, which achieves these two goals in the interpolation setting.

\subsection{The ALI-G Algorithm}

We now present the ALI-G algorithm.
For this, we suppose that we are in an interpolation setting: the model is assumed to be able to drive the loss function to near zero on all samples simultaneously.

\paragraph{Algorithm.} The main steps of the ALI-G algorithm are provided in Algorithm \ref{algo:prox}. ALI-G iterates over three operations until convergence. First, it computes a stochastic approximation of the learning objective and its derivative (line 3). Second, it computes a step-size decay parameter $\gamma_t$ based on the stochastic information (line 4). Third, it updates the parameters by moving in the negative derivative direction by an amount specified by the step-size and projecting the resulting vector on to the feasible region (line 5).

\begin{algorithm}[ht]
\caption{\em The ALI-G algorithm}\label{algo:prox}
\begin{algorithmic}[1]
\REQUIRE maximal learning-rate $\eta$, initial feasible $\vw_0 \in \Omega$, small constant $\delta > 0$
\STATE $t=0$
\WHILE {not converged}
\STATE Get $\ell_{z_t}(\vw_t)$, $\nabla \ell_{z_t}(\vw_t)$ with $z_t$ drawn i.i.d.
\STATE $\gamma_t = \min \left\{ \frac{\ell_{z_t}(\vw_t)}{\| \nabla \ell_{z_t}(\vw_t) \|^2 + \delta}, \eta \right\}$
\STATE $\vw_{t+1} = \Pi_{\Omega}\left(\vw_t - \gamma_t \nabla \ell_{z_t}(\vw_t) \right)$
\STATE $t=t+1$
\ENDWHILE
\end{algorithmic}
\end{algorithm}

\paragraph{Comparison with the Polyak Step-Size.} There are three main differences to the update in equation (\ref{eq:polyak_step}).
First, each update only uses the loss $\ell_{z_t}$ and its derivative rather than the full objective $f$ and its derivative. Second, the learning-rate $\gamma_t$ is clipped to $\eta$, the maximal learning-rate hyper-parameter.
We emphasize that $\eta$ remains constant throughout the iterations, therefore it is a single hyper-parameter and does not need a schedule like SGD learning-rate.
Third, the minimum $\fstar$ has been replaced by the lower-bound of $0$.
All these modifications will be justified in the next section.

\paragraph{The ALI-G$^\infty$ Variant.}
When ALI-G uses no maximal learning-rate, we refer to the algorithm as ALI-G$^\infty$, since it is equivalent to use an infinite maximal learning-rate.
Note that ALI-G$^\infty$ requires no hyper-parameter for its step-size.

\paragraph{Momentum.}
In some of our experiments, we accelerate ALI-G with Nesterov momentum.
The update step at line 5 of algorithm \ref{algo:prox} is then replaced by (i) a velocity update $\vv_{t} = \mu \vv_{t-1} - \gamma_t \nabla \ell_{z_t}(\vw_t)$ and (ii) a parameter update $\vw_{t+1} = \Pi_{\Omega}\left(\vw_t + \mu \vv_{t} \right)$.

\section{Justification and Analysis}

\subsection{Stochasticity}

By definition, the interpolation setting gives $\fstar = 0$, which we used in ALI-G to simplify the formula of the learning-rate $\gamma_t$.
More subtly, the interpolation property also allows the updates to rely on the stochastic estimate $\ell_{z_t}(\vw_t)$ rather than the exact but expensive $f(\vw_t)$.
Intuitively, this is possible because in the interpolation setting, we know the minimum of the loss function for each individual training sample.
Recall that ALI-G$^\infty$ is the variant of ALI-G that uses no maximal learning-rate.
The following result formalizes the convergence guarantee of ALI-G$^\infty$ in the stochastic convex setting.

\begin{theorem}[Convex and Lipschitz] \label{th:alig_cvx}
We assume that $\Omega$ is a convex set, and that for every $z \in \Z$, $\ell_z$ is convex and $C$-Lipschitz.
Let $\wstar$ be a solution of (\ref{eq:main_problem}), and assume that the interpolation property is approximately satisfied: $\forall z \in \Z, \: \ell_z(\wstar) \leq \varepsilon$, for some interpolation tolerance $\varepsilon \geq 0$.
Then ALI-G$^\infty$ applied to $f$ satisfies:
\begin{equation}
\E \left[f\left(\tfrac{1}{T+ 1} \sum\limits_{t=0}^T \vw_t \right) \right]
    \leq \tfrac{ \| \vw_{0} - \wstar \| \sqrt{C^2 + \delta}}{\sqrt{T + 1}} + \varepsilon \sqrt{\left(\tfrac{C^2}{\delta} + 1 \right)}.
\end{equation}
\end{theorem}
In other words, by assuming interpolation, ALI-G provably converges while requiring only $\ell_{z_t}(\vw_t)$ and $\nabla \ell_{z_t}(\vw_t)$ (stochastic estimation per sample) to compute its learning-rate.
In contrast, the Polyak step-size would require $f(\vw_t)$ and $\nabla f(\vw_t)$ to compute the learning-rate (deterministic computation over all training samples).
This is because the Polyak step-size exploits the knowledge of $\fstar$ only, which is weaker information than knowing the minimum of all individual loss functions $\ell_z$ (as ALI-G does in the interpolation setting).
This difference induces a major computational advantage of ALI-G over the usual Polyak step-size.

We emphasize that in Theorem \ref{th:alig_cvx}, our careful analysis explicitly shows the dependency of the convergence result on the interpolation tolerance $\varepsilon$.
It is reassuring to note that convergence is exact when the interpolation property is exactly satisfied ($\varepsilon = 0$).

In the appendix, we also establish convergence rates of $\mathcal{O}(1 / T)$ for smooth convex functions, and $\mathcal{O}(\exp(-\alpha T / 8 \beta ))$ for $\alpha$-strongly convex and $\beta$-smooth functions.
Similar results can be proved when using a maximal learning-rate $\eta$: the convergence speed then remains unchanged provided that $\eta$ is large enough, and it is lowered when $\eta$ is small.
We refer the interested reader to the appendix for the formal results and their proofs.

\paragraph{Interpolation and Gradient Variance.}
In the literature, most convergence results of SGD depend on the variance of the gradient, which we denote by $\upsilon$ here.
The reader may have noticed that our convergence results depends only the interpolation tolerance $\varepsilon$ rather than $\upsilon$.
We briefly compare how these two quantities help convergence in their own distinct ways.
The gradient variance $\upsilon$ globally characterizes how much the gradient direction can differ across individual samples $z$, at any point $\vw$ of the parameter space.
In particular, a low value for $\upsilon$ implies that the loss functions $\ell_z$ agree in the steepest descent direction at any point of the trajectory $\vw_0, ..., \vw_T$. 
In contrast, the interpolation tolerance $\varepsilon$ locally characterizes the behavior of all loss functions near a global minimum $\wstar$ only.
More specifically, a low value for $\varepsilon$ ensures that all loss functions $\ell_z$ agree in a common minimizer $\wstar$.
Thus these two mechanisms are distinct ways of ensuring convergence of SGD.
Importantly, a low interpolation tolerance $\varepsilon$ does not necessarily imply a low gradient variance $\upsilon$ and vice-versa.

\subsection{Maximal Learning-Rate}

\paragraph{Non-Convexity.}
The Polyak step-size may fail to converge when the objective is non-convex, as figure \ref{fig:rsi} illustrates: in this (non-convex) setting, gradient descent with Polyak step-size oscillates between two symmetrical points because its step-size is too large.
A similar behavior can be observed on the non-convex problem of training deep neural networks.

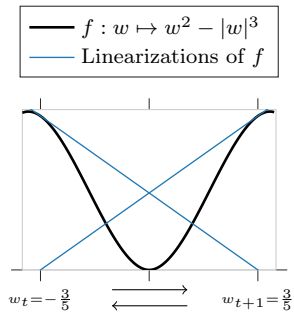
\begin{figure}[h]
\centering
\footnotesize
\input{include/rsi_plot_arxiv.tex}
\caption{\em
A simple example where the Polyak step-size oscillates due to non-convexity.
On this problem, ALI-G converges whenever its maximal learning-rate is lower than $10$.
}
\label{fig:rsi}
\end{figure}

In order to analyze the convergence of ALI-G in a non-convex setting, we introduce the Restricted Secant Inequality (RSI) \cite{Zhang2013}:
\begin{definition}
Let $\phi: \mathbb{R}^p \to \mathbb{R}$ be a lower-bounded differentiable function achieving its minimum at $\wstar$.
We say that $\phi$ satisfies the RSI if there exists $\alpha > 0$ such that:
\begin{equation}
    \forall \vw \in \mathbb{R}^p, \: \nabla \phi (\w)^\top (\vw - \wstar) \geq \alpha \| \vw - \wstar \|^2.
\end{equation}
\end{definition}

The RSI does not require convexity and is a weaker assumption in the sense that all strongly convex functions satisfy the RSI \cite{Zhang2013}.
In particular, the example in figure \ref{fig:rsi} does satisfy the RSI (proof in the appendix).
In other words, the example above shows that the Polyak step-size can fail to converge under the RSI assumption.
In contrast, we prove that with an appropriate maximal learning-rate, ALI-G converges (exponentially fast) on all interpolating problems that satisfy the RSI:

\begin{restatable}{theorem}{thrsismooth}\label{th:alig_rsi}
We assume that $\Omega = \mathbb{R}^p$, and that for every $z \in \Z$, $\ell_z$ is $\beta$-smooth and satisfies the RSI with constant $\mu$.
Let $\wstar$ be a solution of (\ref{eq:main_problem}) such that $\forall z \in \Z, \: \ell_z(\wstar) = 0$.
Further assume that $\frac{1}{2 \beta} \leq \eta \leq \frac{2 \mu}{\beta^2}$.
Then if we apply ALI-G with a maximal learning-rate of $\eta$ to $f$, we have:
\begin{equation}
f(\w_{T+1}) - \fstar
    \leq  \tfrac{\beta}{2} \exp \left( \tfrac{-(2\mu - \eta \beta^2)T }{2\beta} \right) \| \w_{0} - \wstar \|^2.
\end{equation}
\end{restatable}

Note that the above theorem assumes perfect interpolation, that is, the tolerance $\epsilon = 0$.
Nonetheless, it demonstrates the importance of a maximal learning rate, which does not need a manual decaying schedule.
It is currently an open question whether a similar result to theorem \ref{th:alig_rsi} can be proved with some interpolation tolerance $\varepsilon > 0$ on the value of all $\ell_z (\wstar)$.

\paragraph{Proximal Interpretation.}
Interestingly, using a maximal learning-rate can be seen as a natural extension of SGD when using a non-negative loss function:

\begin{restatable}{proposition}{thproxstep}[Proximal Interpretation] \label{th:prox_step}
Suppose that $\Omega = \mathbb{R}^p$ and let $\delta = 0$.
We consider the update performed by SGD: $\vw_{t+1}^{\text{SGD}} = \vw_t - \eta_t \nabla \ell_{z_t}(\vw_t)$; and the update performed by ALI-G: $\vw_{t+1}^{\text{ALI-G}} = \vw_t - \gamma_t \nabla \ell_{z_t}(\vw_t)$, where $\gamma_t = \min\left\{\frac{\ell_{z_t}(\vw_t)}{\| \nabla \ell_{z_t}(\vw_t)\|^2}, \eta \right\}$.
Then we have:
\begin{equation}
\begin{split}
&\vw_{t+1}^{\text{SGD}} = \argmin_{\vw \in \mathbb{R}^p} \Big\{ \tfrac{1}{2 \eta_t} \| \vw - \vw_t \|^2 + \ell_{z_t}(\vw_t) + \nabla \ell_{z_t}(\vw_t)^\top (\vw - \vw_t) \Big\}, \\
&\vw_{t+1}^{\text{ALI-G}} = \argmin_{\vw \in \mathbb{R}^p} \Big\{ \tfrac{1}{2 \eta} \| \vw - \vw_t \|^2 + \max \left\{\ell_{z_t}(\vw_t) + \nabla \ell_{z_t}(\vw_t)^\top (\vw - \vw_t), 0 \right\}  \Big\}. \label{eq:prox_pb}
\end{split}
\end{equation}
\end{restatable}

In other words, at each iteration, ALI-G solves a proximal problem in closed form in a similar way to SGD.
In both cases, the loss function $\ell_{z_t}$ is locally approximated by a first-order Taylor expansion at $\vw_t$.
The difference is that ALI-G also exploits the fact that $\ell_{z_t}$ is non-negative.
This allows ALI-G to use a constant value for $\eta$ in the interpolation setting, while the learning-rate $\eta_t$ of SGD needs to be manually decayed.

\section{Related Work}

\paragraph{Interpolation in Deep Learning.}
As mentioned in the introduction, recent works have successfully exploited the interpolation assumption to prove convergence of SGD in the context of deep learning \citep{Ma2018a,Vaswani2019,Zhou2019}.
Such works are complementary to ours in the sense that they provide a convergence analysis of an existing algorithm for deep learning.
In a different line of work, \cite{Liu2019b} propose to exploit interpolation to prove convergence of a new acceleration method for deep learning. 
However, their experiments suggest that the method still requires the use of a hand-designed learning-rate schedule.

\paragraph{Adaptive Gradient Methods.}
Similarly to ALI-G, most adaptive gradient methods also rely on tuning a single hyper-parameter, thereby providing a more pragmatic alternative to SGD that needs a specification of the full learning-rate schedule.
While the most popular ones are Adagrad \citep{Duchi2011}, RMSPROP \citep{Tieleman2012}, Adam \citep{Kingma2015} and AMSGrad \citep{Reddi2018}, there have been many other variants
\citep{Zeiler2012,Orabona2015,Defossez2017,Levy2017,Mukkamala2017,Zheng2017,Bernstein2018,Chen2018,Shazeer2018,Zaheer2018,Chen2019,Loshchilov2019,Luo2019}.
However, as pointed out in \citep{Wilson2017}, adaptive gradient methods tend to give poor generalization in supervised learning.
In our experiments, the results provided by ALI-G are significantly better than those obtained by the most popular adaptive gradient methods.
Recently, \citet{Liu2019a} have proposed to \say{rectify} Adam with a learning-rate warmup, which partly bridges the gap in generalization performance between Adam and SGD.
However, their method still requires a learning-rate schedule, and thus remains difficult to tune on new tasks.

\paragraph{Adaptive Learning-Rate Algorithms.}
\citet{Vaswani2019a} show that one can use line search in a stochastic setting for interpolating models while guaranteeing convergence.
This work is complementary to ours, as it provides convergence results with weaker assumptions on the loss function, but is less practically useful as it requires up to four hyper-parameters, instead of one for ALI-G.
Less closely related methods, included second-order ones, adaptively compute the learning-rate without using the minimum \citep{Schaul2013,Martens2015,Tan2016,Zhang2017a,Baydin2018,Wu2018,Li2019,Henriques2019}, but do not demonstrate competitive generalization performance against SGD with a well-tuned hand-designed schedule.

\paragraph{$L_4$ Algorithm.}
The $L_4$ algorithm \citep{Rolinek2018} also uses a modified version of the Polyak step-size.
However, the $L_4$ algorithm computes an online estimate of $\fstar$ rather  than relying on a fixed value.
This requires three hyper-parameters, which are in practice sensitive to noise and crucial for empirical convergence of the method.
In addition, $L_4$ does not come with convergence guarantees.
In contrast, by utilizing the interpolation property and a maximal learning-rate, our method is able to (i) provide reliable and accurate minimization with only a single hyper-parameter, and (ii) offer guarantees of convergence in the stochastic convex setting.

\paragraph{Frank-Wolfe Methods.}
The proximal interpretation in Proposition \ref{th:prox_step} allows us to draw additional parallels to existing methods.
In particular, the formula of the learning-rate $\gamma_t$ may remind the reader of the Frank-Wolfe algorithm \citep{Frank1956} in some of its variants \citep{Locatello2017}, or other dual methods \citep{Lacoste-Julien2013,Shalev-Shwartz2016}.
This is because such methods solve in closed form the dual of problem (\ref{eq:prox_pb}), and problems in the form of (\ref{eq:prox_pb}) naturally appear in dual coordinate ascent methods \citep{Shalev-Shwartz2016}.

When no regularization is used, ALI-G and Deep Frank-Wolfe (DFW) \citep{Berrada2019} are procedurally identical algorithms.
This is because in such a setting, one iteration of DFW also amounts to solving (\ref{eq:prox_pb}) in closed-form -- more generally, DFW is designed to train deep neural networks by solving proximal linear support vector machine problems approximately.
However, we point out the two fundamental advantages of ALI-G over DFW:
(i) ALI-G can handle arbitrary (lower-bounded) loss functions, while DFW can only use convex piece-wise linear loss functions;
and (ii) as seen previously, ALI-G provides convergence guarantees in the convex setting.

\paragraph{SGD with Polyak's Learning-Rate.}
\citep{Oberman2019} extend the Polyak step-size to rely on a stochastic estimation of the gradient $\nabla \ell_{z_t}(\vw_t)$ only, instead of the expensive deterministic gradient $\nabla f(\vw_t)$.
However, they still require to evaluate $f(\vw_t)$, the objective function over the entire training data set, in order to compute its learning-rate, which makes the method impractical.
In addition, since they do not do exploit the interpolation setting nor the fact that regularization can be expressed as a constraint, they also require the knowledge of the optimal objective function value $\fstar$.
We also refer the interested reader to the recent analysis of \cite{Loizou2020}, which appeared after this work and provides a set of improved theoretical results.

\paragraph{{\sc aProx} Algorithm.}
\citep{Asi2019} have recently introduced the {\sc aProx} algorithm, a family of proximal stochastic optimization algorithms for convex problems.
Notably, the {\sc aProx} \say{truncated model} version is similar to ALI-G.
However, there are four clear advantages of our work over \citep{Asi2019} in the interpolation setting, in particular for training neural networks.
First, our work is the first to empirically demonstrate the applicability and usefulness of the algorithm on varied modern deep learning tasks -- most of our experiments use several orders of magnitude more data and model parameters than the small-scale convex problems of \citep{Asi2019}.
Second, our analysis and insights allow us to make more aggressive choices of learning rate than \citep{Asi2019}.
Indeed, \citep{Asi2019} assume that the maximal learning-rate is exponentially decaying, even in the interpolating convex setting.
In contrast, by avoiding the need for an exponential decay, the learning-rate of ALI-G requires only one hyper-parameters instead of two for {\sc aProx}.
Third, our analysis takes into account the interpolation tolerance $\varepsilon \geq 0$ rather than unrealistically assuming the perfect case $\varepsilon = 0$ (that would require infinite weights when using the cross-entropy loss for instance).
Fourth, our analysis proves fast convergence in function space rather than iterate space.

\section{Experiments}

We empirically compare ALI-G to the optimization algorithms most commonly used in deep learning.
Our experiments span a variety of architectures and tasks: (i) learning a differentiable neural computer;
(ii) training wide residual networks on SVHN;
(iii) training a Bi-LSTM on the Stanford Natural Language Inference data set;
and (iv) training wide residual networks and densely connected networks on the CIFAR data sets.
Note that the tasks of training wide residual networks on SVHN and CIFAR-100 are part of the DeepOBS benchmark \citep{Schneider2019}, which aims at standardizing baselines for deep learning optimizers.
In particular, these tasks are among the most difficult ones of the benchmark because the SGD baseline benefits from a manual schedule for the learning rate.
Despite this, our set of experiments demonstrate that ALI-G obtains competitive performance with SGD.
In addition, ALI-G significantly outperforms adaptive gradient methods.

The code to reproduce our results is publicly available\footnote{\url{https://github.com/oval-group/ali-g}}.
In the TensorFlow \citep{Abadi2015} experiment, we use the official and publicly available implementation of $L_4$\footnote{\url{https://github.com/martius-lab/l4-optimizer}}.
In the PyTorch \citep{Paszke2017} experiments, we use our implementation of $L_4$, which we unit-test against the official TensorFlow implementation.
In addition, we employ the official implementation of DFW\footnote{\url{https://github.com/oval-group/dfw}} and we re-use their code for the experiments on SNLI and CIFAR.
All experiments are performed either on a 12-core CPU (differentiable neural computer), on a single GPU (SVHN, SNLI, CIFAR) or on up to 4 GPUs (ImageNet).
We emphasize that all methods approximately have the same cost per iteration.
Consequently, faster convergence in terms of number of iterations or epochs translates into faster convergence in terms of wall-clock time.

\subsection{Differentiable Neural Computers}

\paragraph{Setting.}
The Differentiable Neural Computer (DNC) \citep{Graves2016} is a recurrent neural network that aims at performing computing tasks by learning from examples rather than by executing an explicit program.
In this case, the DNC learns to repeatedly copy a fixed size string given as input.
Although this learning task is relatively simple, the complex architecture of the DNC makes it an interesting benchmark problem for optimization algorithms.

\begin{figure}[H]
\centering
\includegraphics[width=0.65\linewidth]{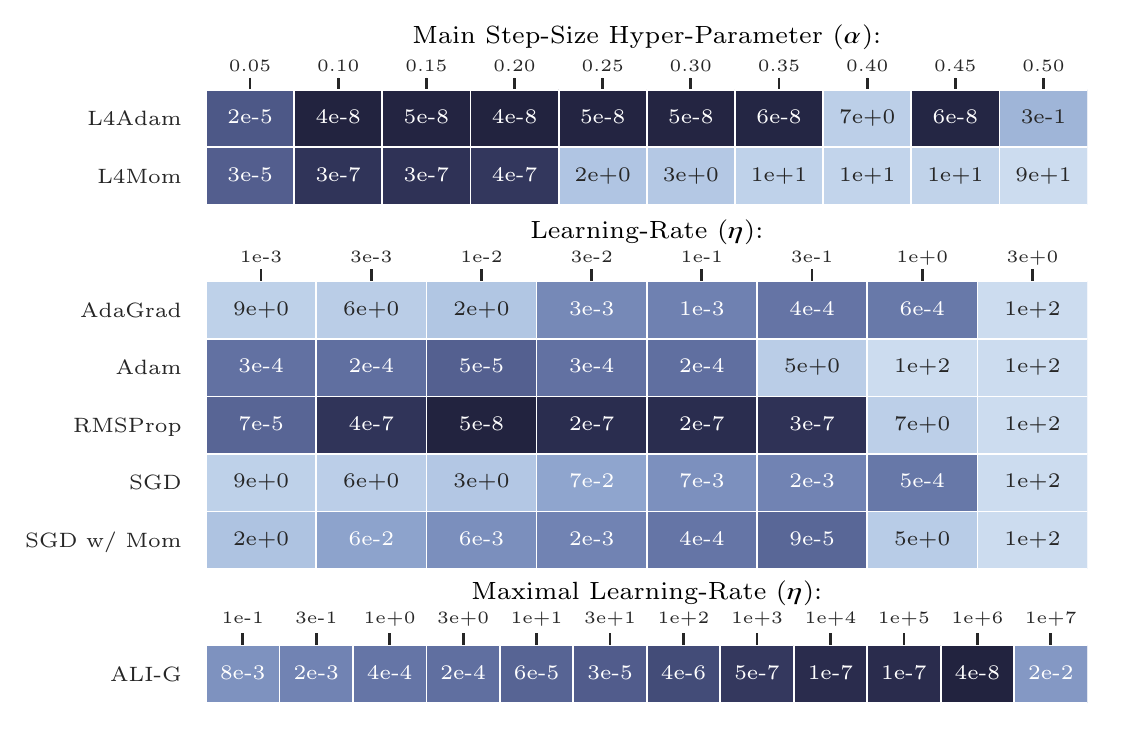}
\caption{
    \em Final objective function when training a Differentiable Neural Computer for $10$k steps (lower is better).
    The intensity of each cell is log-proportional to the value of the objective function (darker is better).
    ALI-G obtains good performance for a very large range of $\eta$ ($10^{-1} \leq \eta \leq 10^6$).
    }
\label{fig:dnc}
\end{figure}

\paragraph{Methods.}
We use the official and publicly available implementation of DNC\footnote{\url{https://github.com/deepmind/dnc}}.
We vary the initial learning rate as powers of ten between $10^{-4}$ and $10^{4}$ for each method except for $L_4$Adam and $L_4$Mom.
For $L_4$Adam and $L_4$Mom, since the main hyper-parameter $\alpha$ is designed to lie in $(0, 1)$, we vary it between $0.05$ and $0.095$ with a step of $0.1$.
The gradient norm is clipped for all methods except for ALI-G, $L_4$Adam and $L_4$Mom (as recommended by \citep{Rolinek2018}).

\paragraph{Results.} We present the results in Figure \ref{fig:dnc}.
ALI-G provides accurate optimization for any $\eta$ within $[10^{-1}, 10^6]$, and is among the best performing methods by reaching an objective function of $4.10^{-8}$.
On this task, RMSProp, $L_4$Adam and $L_4$Mom also provide accurate and robust optimization.
In contrast to ALI-G and the $L_4$ methods, the most commonly used algorithms such as SGD, SGD with momentum and Adam are very sensitive to their main learning-rate hyper-parameter.
Note that the difference between well-performing methods is not significant here because these reach the numerical precision limit of single-precision float numbers.

\subsection{Wide Residual Networks on SVHN}

\paragraph{Setting.}
The SVHN data set contains 73k training samples, 26k testing samples and 531k additional easier samples.
From the 73k difficult training examples, we select 6k samples for validation; we use all remaining (both difficult and easy) examples for training, for a total of 598k samples.
We train a wide residual network 16-4 following \citep{Zagoruyko2016}.

\paragraph{Method.}
For SGD, we use the manual schedule for the learning rate of \citep{Zagoruyko2016}.
For $L_4$Adam and $L_4$Mom, we cross-validate the main learning-rate hyper-parameter $\alpha$ to be in $\{0.0015, 0.015, 0.15\}$ ($0.15$ is the value recommended by \citep{Rolinek2018}).
For other methods, the learning rate hyper-parameter is tuned as a power of 10.
The $\ell_2$ regularization is cross-validated in $\{0.0001, 0.0005\}$ for all methods but ALI-G.
For ALI-G, the regularization is expressed as a constraint on the $\ell_2$-norm of the parameters, and its maximal value is set to $50$.
SGD, ALI-G and BPGrad use a Nesterov momentum of 0.9.
All methods use a dropout rate of 0.4 and a fixed budget of 160 epochs, following \citep{Zagoruyko2016}.

\begin{table}[ht]
\centering
\begin{tabular}{lc|lc}
    \toprule
    \multicolumn{4}{c}{Test Accuracy on SVHN (\%)} \\
    \midrule
    Adagrad & 98.0 & Adam & 97.9 \\
    AMSGrad & 97.9 & BPGrad & 98.1 \\
    DFW & 98.1 &$L_4$Adam &{\bf 98.2} \\
    $L_4$Mom & 19.6 & ALI-G & 98.1 \\
    \cmidrule(lr){1-2} \cmidrule(lr){3-4}
    {\color{red} SGD} &98.3 &{\color{red} SGD$^\dagger$} & 98.4 \\
    \bottomrule
    \end{tabular}
\caption{\em
    In red, SGD benefits from a hand-designed schedule for its learning-rate.
    In black, adaptive methods, including ALI-G, have a single hyper-parameter for their learning-rate.
    $SGD^\dagger$ refers to the performance reported by \citep{Zagoruyko2016}.
    }
\label{tab:svhn}
\end{table}

\paragraph{Results.}
The results are presented in Table \ref{tab:svhn}.
On this relatively easy task, most methods achieve about 98\% test accuracy.
Despite the cross-validation, $L_4$Mom does not converge on this task.
Even though SGD benefits from a hand-designed schedule, ALI-G and other adaptive methods obtain close performance to it.

\subsection{Bi-LSTM on SNLI}

\paragraph{Setting.}
We train a Bi-LSTM of 47M parameters on the Stanford Natural Language Inference (SNLI) data set \citep{Bowman2015}.
The SNLI data set consists in 570k pairs of sentences, with each pair labeled as entailment, neutral or contradiction.
This large scale data set is commonly used as a pre-training corpus for transfer learning to many other natural language tasks where labeled data is scarcer \citep{Conneau2017} -- much like ImageNet is used for pre-training in computer vision.
We follow the protocol of \citep{Berrada2019}; we also re-use their results for the baselines.

\paragraph{Method.}
For $L_4$Adam and $L_4$Mom, the main hyper-parameter $\alpha$ is cross-validated in $\{0.015, 0.15\}$ -- compared to the recommended value of 0.15, this helped convergence and considerably improved performance.
The SGD algorithm benefits from a hand-designed schedule, where the learning-rate is decreased by 5 when the validation accuracy does not improve.
Other methods use adaptive learning-rates and do not require such schedule.
The value of the main hyper-parameter $\eta$ is cross-validated as a power of ten for the ALI-G algorithm and for previously reported adaptive methods.
Following the implementation by \citep{Conneau2017}, no $\ell_2$ regularization is used.
The algorithms are evaluated with the Cross-Entropy (CE) loss and the multi-class hinge loss (SVM), except for DFW which is designed for use with an SVM loss only.
For all optimization algorithms, the model is trained for 20 epochs, following \citep{Conneau2017}.

\begin{table}[h]
\centering
\begin{tabular}{lcc|lcc}
    \toprule
    \multicolumn{6}{c}{Test Accuracy on SNLI (\%)} \\
    \midrule
    & CE & SVM & & CE & SVM \\
    \cmidrule(lr){2-2} \cmidrule(lr){3-3} \cmidrule(lr){5-5} \cmidrule(lr){6-6}
    Adagrad$^*$ &83.8 &84.6 &Adam$^*$ &84.5 &85.0 \\
    AMSGrad$^*$ &84.2 &85.1 &BPGrad$^*$ &83.6 &84.2 \\
    DFW$^*$ & - &{\bf 85.2} & $L_4$Adam &83.3 &82.5 \\
    $L_4$Mom &83.7 &83.2 & {\color{blue}ALI-G$^\infty$} &84.6 &84.7 \\
    ALI-G & {\bf 84.8} &{\bf 85.2} & & \\
    \cmidrule(lr){1-3} \cmidrule(lr){4-6}
    {\color{red} SGD$^*$} &84.7 &85.2 &{\color{red} SGD$^\dagger$} &84.5 & - \\
    \bottomrule
    \end{tabular}
\caption{\em
    In red, SGD benefits from a hand-designed schedule for its learning-rate.
    In black, adaptive methods have a single hyper-parameter for their learning-rate.
    In blue, {\color{blue} ALI-G$^\infty$} does not have any hyper-parameter for its learning-rate.
    With an SVM loss, DFW and ALI-G are procedurally identical algorithms
    -- but in contrast to DFW, ALI-G can also employ the CE loss.
    Methods in the format $X^*$ re-use results from \citep{Berrada2019}.
    $SGD^\dagger$ is the result from \citep{Conneau2017}.
    }
\label{tab:snli}
\end{table}

\paragraph{Results.} We present the results in Table \ref{tab:snli}.
ALI-G$^\infty$ is the only method that requires no hyper-parameter for its learning-rate.
Despite this, and the fact that SGD employs a learning-rate schedule that has been hand designed for good validation performance, ALI-G$^\infty$ is still able to obtain results that are competitive with SGD.
Moreover, ALI-G, which requires a single hyper-parameter for the learning-rate, outperforms all other methods for both the SVM and the CE loss functions.

\subsection{Wide Residual Networks and Densely Connected Networks on CIFAR}

\paragraph{Setting.}
We follow the methodology of \citep{Berrada2019}, and we reproduce their results.
We test two architectures: a Wide Residual Network (WRN) 40-4 \citep{Zagoruyko2016} and a bottleneck DenseNet (DN) 40-40 \citep{Huang2017a}.
We use 45k samples for training and 5k for validation.
The images are centered and normalized per channel.
We apply standard data augmentation with random horizontal flipping and random crops.
AMSGrad was selected in \citep{Berrada2019} because it was the best adaptive method on similar tasks, outperforming in particular Adam and Adagrad.
In addition to the baselines from \citep{Berrada2019}, we also provide the performance of $L_4$Adam, $L_4$Mom, AdamW \citep{Loshchilov2019} and Yogi \citep{Zaheer2018}.

\paragraph{Method.}
All optimization methods employ the cross-entropy loss, except for the DFW algorithm, which is designed to use an SVM loss.
For DN and WRN respectively, SGD uses the manual learning rate schedules from \citep{Huang2017a} and \citep{Zagoruyko2016}.
Following \citep{Berrada2019}, the batch-size is cross-validated in $\{64, 128, 256\}$ for the DN architecture, and $\{128, 256, 512\}$ for the WRN architecture.
For $L_4$Adam and $L_4$Mom, the learning-rate hyper-parameter $\alpha$ is cross-validated in $\{0.015, 0.15\}$.
For AMSGrad, AdamW, Yogi, DFW and ALI-G, the learning-rate hyper-parameter $\eta$ is cross-validated as a power of 10 (in practice $\eta \in \{0.1, 1\}$ for ALI-G).
SGD, DFW and ALI-G use a Nesterov momentum of 0.9.
Following \citep{Berrada2019}, for all methods but ALI-G and AdamW, the $\ell_2$ regularization is cross-validated in $\{0.0001, 0.0005\}$ on the WRN architecture, and is set to $0.0001$ for the DN architecture.
For AdamW, the weight-decay is cross-validated as a power of 10.
For ALI-G, $\ell_2$ regularization is expressed as a constraint on the norm on the vector of parameters; its maximal value is set to $100$ for the WRN models, $80$ for DN on CIFAR-10 and $75$ for DN on CIFAR-100.
For all optimization algorithms, the WRN model is trained for 200 epochs and the DN model for 300 epochs, following respectively \citep{Zagoruyko2016} and \citep{Huang2017a}.

\paragraph{Results.}
We present the results in Table \ref{tab:cifar}.
In this setting again, ALI-G obtains competitive performance with manually decayed SGD.
ALI-G largely outperforms AMSGrad, AdamW and Yogi.

\begin{table}[ht]
\centering
\begin{tabular}{lcccc}
    \toprule
    \multicolumn{5}{c}{Test Accuracy on CIFAR (\%)} \\
    \midrule
    &\multicolumn{2}{c}{CIFAR-10} &\multicolumn{2}{c}{CIFAR-100} \\
    \cmidrule(lr){2-3} \cmidrule(lr){4-5}
    & WRN & DN & WRN & DN \\
    \cmidrule(lr){2-2} \cmidrule(lr){3-3} \cmidrule(lr){4-4} \cmidrule(lr){5-5}
    AMSGrad & 90.8 & 91.7 &  68.7 &  69.4 \\
    AdamW & 92.1 & 92.6 &  69.6 &  69.5 \\
    Yogi & 91.2 & 92.1 &  68.7 &  69.6 \\
    DFW & 94.2 & 94.6 & {\bf 76.0} &  73.2 \\
    $L_4$Adam & 90.5 & 90.8 &  61.7 &  60.5 \\
    $L_4$Mom & 91.6 & 91.9 &  61.4 &  62.6 \\
    ALI-G & {\bf 95.2} & {\bf 95.0} &  75.8 &  {\bf 76.3} \\
    \midrule
    {\color{red} SGD} & 95.3 & 95.1 &  77.8 &  76.3 \\
    {\color{red} SGD$^\dagger$} & 95.4 & - &  78.8 & - \\
    \bottomrule
    \end{tabular}
\caption{\em
    In red, SGD benefits from a hand-designed schedule for its learning-rate.
    In black, adaptive methods, including ALI-G, have a single hyper-parameter for their learning-rate.
    $SGD^\dagger$ refers to the result from \citep{Zagoruyko2016}.
    Each reported result is an average over three independent runs;
    the standard deviations are reported in Appendix (they are at most $0.3$ for ALI-G and SGD).
    }
\label{tab:cifar}
\end{table}

\subsection{Comparing Training Performance on CIFAR-100}

In this section, we empirically assess the performance of ALI-G and its competitors in terms of training objective on CIFAR-100.
In order to have comparable objective functions, the $\ell_2$ regularization is deactivated.
The learning-rate is selected as a power of ten for best final objective value, and the batch-size is set to its default value.
For clarity, we only display the performance of SGD, Adam, Adagrad and ALI-G (DFW does not support the cross-entropy loss).
The $L_4$ methods diverge in this setting.
Here SGD uses a constant learning-rate to emphasize the need for adaptivity.
Therefore all methods use one hyper-parameter for their learning-rate.
All methods use a fixed budget of 200 epochs for WRN-CIFAR-100 and 300 epochs for DN-CIFAR-100.
As can be seen, ALI-G provides better training performance than the baseline algorithms on all tasks.

\begin{figure}[h]
\centering
\footnotesize
\input{include/cifar_train_wrn_cifar100_arxiv.tex}
\input{include/cifar_train_dn_cifar100_arxiv.tex}
\caption{\em
    Objective function over the epochs on CIFAR-100 (smoothed with a moving average over 5 epochs).
    ALI-G reaches a value that is an order of magnitude better than the baselines.
}
\label{fig:cifar100_training}
\vspace{-10pt}
\end{figure}
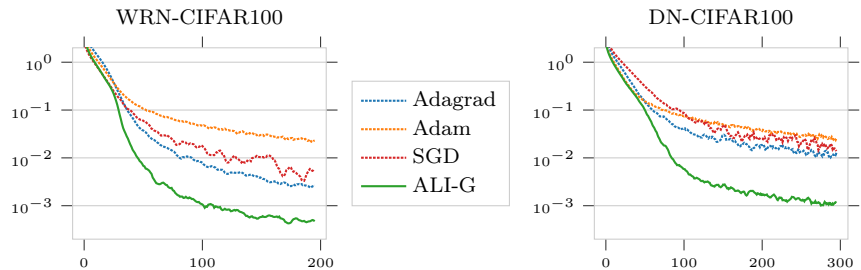

\subsection{Training at Large Scale}

We demonstrate the scalability of ALI-G by training a ResNet-18 \citep{He2016} on the ImageNet data set.
In order to satisfy the interpolation assumption, we employ a loss function tailored for top-5 classification \cite{Lapin2016}, and we do not use data augmentation.
Our focus here is on the training objective and accuracy.
ALI-G uses the following training setup: a batch-size of 1024 split over 4 GPUs, a $\ell_2$ maximal norm of 400 for $\w$, a maximal learning-rate of 10 and no momentum.
SGD uses the state-of-the-art hyper-parameters and learning-rate schedule from \cite{He2016}.
As can be seen in figure \ref{fig:imagenet_training}, ALI-G reaches 99\% top-5 accuracy in 12 epochs (faster than SGD), and minimizes the objective function as well as SGD with its custom schedule.

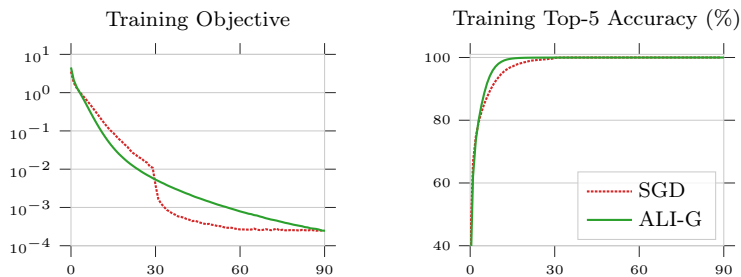
\begin{figure}[H]
\centering
\footnotesize
\input{include/imagenet_obj_arxiv.tex}
\hspace{20pt}
\input{include/imagenet_acc_arxiv.tex}
\caption{\em
    Training a ResNet-18 on ImageNet.
    The final performance of ALI-G is as good as that of SGD, even though SGD benefits from a custom learning-rate schedule.
    In addition, ALI-G reaches a high training accuracy faster than SGD.
}
\label{fig:imagenet_training}
\end{figure}

\section{Discussion}

We have introduced ALI-G, an optimization algorithm that automatically adapts the learning-rate in the interpolation setting.
ALI-G provides convergence guarantees in the stochastic setting, including for a class of non-convex problems.
By using the same descent direction as SGD, it offers comparable generalization performance while requiring significantly less tuning.
In future work, it would be interesting to extend ALI-G to the non-interpolating setting by adapting the minimum $\fstar$ online while requiring few hyper-parameters.

\subsection*{Acknowledgements}

This work was supported by the EPSRC grants AIMS CDT EP/L015987/1, Seebibyte EP/M013774/1, EP/P020658/1 and TU/B/000048, and by YouGov. 
We also thank the Nvidia Corporation for the GPU donation.

%% file: include/plot_step_simple_arxiv.tex
% This file was created by tikzplotlib v0.9.3.
\begin{tikzpicture}

\definecolor{color0}{rgb}{0.12156862745098,0.466666666666667,0.705882352941177}
\definecolor{color1}{rgb}{1,0.498039215686275,0.0549019607843137}

\begin{axis}[
axis line style={white!80!black},
compat=newest,
height=0.3\textwidth,
legend cell align={left},
legend style={at={(0.5,2)}, anchor=north, draw=white!80.0!black},
legend style={fill opacity=1, draw opacity=1, text opacity=1, at={(0.5,0.91)}, anchor=north, draw=white!80!black},
tick align=outside,
width=0.45\textwidth,
x grid style={white!80!black},
xmajorgrids,
xmajorticks=false,
xmajorticks=true,
xmin=-2.19950000047684, xmax=2.18950001001358,
xtick style={color=white!15!black},
xtick={-1.5,-0.109999999403954,-0.825963318347931},
xticklabels={\(\displaystyle \w_t\),\(\displaystyle \wstar\),\(\displaystyle \w_{t+1}\)},
y grid style={white!80!black},
ymajorgrids,
ymajorticks=false,
ymajorticks=true,
ymin=-3, ymax=21,
ytick style={color=white!15!black},
ytick={3.9776074886322},
yticklabels={\(\displaystyle \fstar\)}
]
\addplot [line width=1pt, black, opacity=0.7]
table {%
-2 20.0856609344482
-1.99000000953674 19.8858127593994
-1.98000001907349 19.6879539489746
-1.97000002861023 19.4920635223389
-1.96000003814697 19.2981243133545
-1.95000004768372 19.1061134338379
-1.94000005722046 18.9160137176514
-1.92999994754791 18.7278022766113
-1.91999995708466 18.5414733886719
-1.9099999666214 18.3569889068604
-1.89999997615814 18.1743507385254
-1.88999998569489 17.9935207366943
-1.87999999523163 17.8145008087158
-1.87000000476837 17.6372528076172
-1.86000001430511 17.4617767333984
-1.85000002384186 17.2880420684814
-1.8400000333786 17.1160430908203
-1.83000004291534 16.9457473754883
-1.82000005245209 16.777156829834
-1.80999994277954 16.6102352142334
-1.79999995231628 16.4449825286865
-1.78999996185303 16.2813720703125
-1.77999997138977 16.1193904876709
-1.76999998092651 15.959023475647
-1.75999999046326 15.8002529144287
-1.75 15.6430625915527
-1.74000000953674 15.4874382019043
-1.73000001907349 15.33336353302
-1.72000002861023 15.1808233261108
-1.71000003814697 15.0298023223877
-1.70000004768372 14.8802852630615
-1.69000005722046 14.7322587966919
-1.67999994754791 14.5857019424438
-1.66999995708466 14.4406127929688
-1.6599999666214 14.2969617843628
-1.64999997615814 14.1547498703003
-1.63999998569489 14.0139484405518
-1.62999999523163 13.8745565414429
-1.62000000476837 13.7365465164185
-1.61000001430511 13.5999193191528
-1.60000002384186 13.4646482467651
-1.5900000333786 13.3307323455811
-1.58000004291534 13.1981449127197
-1.57000005245209 13.0668849945068
-1.55999994277954 12.9369297027588
-1.54999995231628 12.8082733154297
-1.53999996185303 12.6809005737305
-1.52999997138977 12.554799079895
-1.51999998092651 12.4299554824829
-1.50999999046326 12.3063592910767
-1.5 12.1839962005615
-1.49000000953674 12.06285572052
-1.48000001907349 11.942925453186
-1.47000002861023 11.8241920471191
-1.46000003814697 11.7066469192505
-1.45000004768372 11.590274810791
-1.44000005722046 11.4750690460205
-1.42999994754791 11.361011505127
-1.41999995708466 11.2481002807617
-1.4099999666214 11.1363134384155
-1.39999997615814 11.025652885437
-1.38999998569489 10.9160947799683
-1.37999999523163 10.8076400756836
-1.37000000476837 10.7002668380737
-1.36000001430511 10.593976020813
-1.35000002384186 10.4887466430664
-1.3400000333786 10.384578704834
-1.33000004291534 10.2814521789551
-1.32000005245209 10.1793670654297
-1.30999994277954 10.0783033370972
-1.29999995231628 9.97826099395752
-1.28999996185303 9.87922382354736
-1.27999997138977 9.7811861038208
-1.26999998092651 9.68413734436035
-1.25999999046326 9.58806800842285
-1.25 9.49296951293945
-1.24000000953674 9.39883327484131
-1.23000001907349 9.30564880371094
-1.22000002861023 9.21340847015381
-1.21000003814697 9.12210559844971
-1.20000004768372 9.03172969818115
-1.19000005722046 8.9422721862793
-1.17999994754791 8.85372352600098
-1.16999995708466 8.76608276367188
-1.1599999666214 8.6793327331543
-1.14999997615814 8.59347343444824
-1.13999998569489 8.50849056243896
-1.12999999523163 8.42438316345215
-1.12000000476837 8.34113788604736
-1.11000001430511 8.25875377655029
-1.10000002384186 8.17721652984619
-1.0900000333786 8.09652709960938
-1.08000004291534 8.01667022705078
-1.07000005245209 7.93764877319336
-1.05999994277954 7.85944652557373
-1.04999995231628 7.78206443786621
-1.03999996185303 7.70549345016479
-1.02999997138977 7.62972784042358
-1.01999998092651 7.55476188659668
-1.00999999046326 7.48058986663818
-1 7.4072060585022
-0.990000009536743 7.33460521697998
-0.980000019073486 7.26278257369995
-0.970000028610229 7.19173336029053
-0.959999978542328 7.12145137786865
-0.949999988079071 7.05193328857422
-0.939999997615814 6.98317384719849
-0.930000007152557 6.91516971588135
-0.920000016689301 6.84791612625122
-0.910000026226044 6.78140926361084
-0.899999976158142 6.71564435958862
-0.889999985694885 6.65062093734741
-0.879999995231628 6.58633327484131
-0.870000004768372 6.52277946472168
-0.860000014305115 6.45995569229126
-0.850000023841858 6.39786100387573
-0.839999973773956 6.33649063110352
-0.829999983310699 6.27584552764893
-0.819999992847443 6.2159218788147
-0.810000002384186 6.15671873092651
-0.800000011920929 6.0982346534729
-0.790000021457672 6.04046821594238
-0.779999971389771 5.9834189414978
-0.769999980926514 5.92708683013916
-0.759999990463257 5.87146997451782
-0.75 5.81656980514526
-0.740000009536743 5.76238679885864
-0.730000019073486 5.70892095565796
-0.720000028610229 5.65617322921753
-0.709999978542328 5.60414505004883
-0.699999988079071 5.55283737182617
-0.689999997615814 5.50225257873535
-0.680000007152557 5.452392578125
-0.670000016689301 5.40325975418091
-0.660000026226044 5.35485649108887
-0.649999976158142 5.30718612670898
-0.639999985694885 5.26025295257568
-0.629999995231628 5.21405935287476
-0.620000004768372 5.16860961914062
-0.610000014305115 5.12390899658203
-0.600000023841858 5.07996034622192
-0.589999973773956 5.0367693901062
-0.579999983310699 4.99434232711792
-0.569999992847443 4.9526834487915
-0.560000002384186 4.91179847717285
-0.550000011920929 4.87169408798218
-0.540000021457672 4.83237600326538
-0.529999971389771 4.79385089874268
-0.519999980926514 4.75612592697144
-0.509999990463257 4.71920680999756
-0.5 4.68310213088989
-0.490000009536743 4.64781856536865
-0.479999989271164 4.61336326599121
-0.469999998807907 4.5797438621521
-0.46000000834465 4.54696798324585
-0.449999988079071 4.51504421234131
-0.439999997615814 4.48397827148438
-0.430000007152557 4.45378017425537
-0.419999986886978 4.4244556427002
-0.409999996423721 4.39601373672485
-0.400000005960464 4.36846160888672
-0.389999985694885 4.34180641174316
-0.379999995231628 4.31605577468872
-0.370000004768372 4.29121589660645
-0.360000014305115 4.26729393005371
-0.349999994039536 4.24429655075073
-0.340000003576279 4.22222948074341
-0.330000013113022 4.20109891891479
-0.319999992847443 4.18090915679932
-0.310000002384186 4.16166591644287
-0.300000011920929 4.14337348937988
-0.28999999165535 4.12603569030762
-0.280000001192093 4.10965490341187
-0.270000010728836 4.09423446655273
-0.259999990463257 4.07977676391602
-0.25 4.06628227233887
-0.239999994635582 4.05375242233276
-0.230000004172325 4.04218673706055
-0.219999998807907 4.03158473968506
-0.209999993443489 4.02194404602051
-0.200000002980232 4.01326417922974
-0.189999997615814 4.00554084777832
-0.180000007152557 3.99877119064331
-0.170000001788139 3.99294948577881
-0.159999996423721 3.98807215690613
-0.150000005960464 3.98413228988647
-0.140000000596046 3.98112344741821
-0.129999995231628 3.97903871536255
-0.119999997317791 3.97786951065063
-0.109999999403954 3.9776074886322
-0.100000001490116 3.97824311256409
-0.0900000035762787 3.97976660728455
-0.0799999982118607 3.98216772079468
-0.0700000002980232 3.98543500900269
-0.0599999986588955 3.98955678939819
-0.0500000007450581 3.99452209472656
-0.0399999991059303 4.00031757354736
-0.0299999993294477 4.00693082809448
-0.0199999995529652 4.01434850692749
-0.00999999977648258 4.02255725860596
0 4.03154325485229
0.00999999977648258 4.04129314422607
0.0199999995529652 4.05179166793823
0.0299999993294477 4.06302499771118
0.0399999991059303 4.07497882843018
0.0500000007450581 4.08763885498047
0.0599999986588955 4.100989818573
0.0700000002980232 4.11501789093018
0.0799999982118607 4.12970781326294
0.0900000035762787 4.1450457572937
0.100000001490116 4.1610164642334
0.109999999403954 4.17760610580444
0.119999997317791 4.19480037689209
0.129999995231628 4.21258544921875
0.140000000596046 4.23094654083252
0.150000005960464 4.24987077713013
0.159999996423721 4.26934480667114
0.170000001788139 4.28935432434082
0.180000007152557 4.30988645553589
0.189999997615814 4.33092880249023
0.200000002980232 4.35246896743774
0.209999993443489 4.37449407577515
0.219999998807907 4.39699172973633
0.230000004172325 4.41995096206665
0.239999994635582 4.443359375
0.25 4.46720695495605
0.259999990463257 4.49148082733154
0.270000010728836 4.51617193222046
0.280000001192093 4.54126930236816
0.28999999165535 4.56676292419434
0.300000011920929 4.59264278411865
0.310000002384186 4.61889839172363
0.319999992847443 4.6455225944519
0.330000013113022 4.67250442504883
0.340000003576279 4.69983577728271
0.349999994039536 4.72750854492188
0.360000014305115 4.75551414489746
0.370000004768372 4.78384351730347
0.379999995231628 4.81249094009399
0.389999985694885 4.84144687652588
0.400000005960464 4.87070608139038
0.409999996423721 4.90025997161865
0.419999986886978 4.93010234832764
0.430000007152557 4.96022653579712
0.439999997615814 4.99062585830688
0.449999988079071 5.02129459381104
0.46000000834465 5.05222606658936
0.469999998807907 5.08341503143311
0.479999989271164 5.11485576629639
0.490000009536743 5.14654350280762
0.5 5.17847156524658
0.509999990463257 5.21063613891602
0.519999980926514 5.24303150177002
0.529999971389771 5.27565288543701
0.540000021457672 5.30849695205688
0.550000011920929 5.34155750274658
0.560000002384186 5.374831199646
0.569999992847443 5.4083137512207
0.579999983310699 5.44200134277344
0.589999973773956 5.47588920593262
0.600000023841858 5.50997447967529
0.610000014305115 5.5442533493042
0.620000004768372 5.57872152328491
0.629999995231628 5.61337566375732
0.639999985694885 5.6482138633728
0.649999976158142 5.68323135375977
0.660000026226044 5.71842479705811
0.670000016689301 5.75379276275635
0.680000007152557 5.78933048248291
0.689999997615814 5.82503604888916
0.699999988079071 5.86090660095215
0.709999978542328 5.89693927764893
0.720000028610229 5.9331316947937
0.730000019073486 5.96948051452637
0.740000009536743 6.00598430633545
0.75 6.04263973236084
0.759999990463257 6.07944536209106
0.769999980926514 6.11639785766602
0.779999971389771 6.15349578857422
0.790000021457672 6.19073629379272
0.800000011920929 6.22811794281006
0.810000002384186 6.26563835144043
0.819999992847443 6.30329561233521
0.829999983310699 6.34108734130859
0.839999973773956 6.37901210784912
0.850000023841858 6.41706800460815
0.860000014305115 6.45525312423706
0.870000004768372 6.49356508255005
0.879999995231628 6.53200340270996
0.889999985694885 6.57056522369385
0.899999976158142 6.60924911499023
0.910000026226044 6.64805459976196
0.920000016689301 6.68697834014893
0.930000007152557 6.72601985931396
0.939999997615814 6.76517677307129
0.949999988079071 6.80444860458374
0.959999978542328 6.84383392333984
0.970000028610229 6.88333082199097
0.980000019073486 6.92293691635132
0.990000009536743 6.96265268325806
1 7.00247573852539
1.00999999046326 7.042405128479
1.01999998092651 7.08243894577026
1.02999997138977 7.12257671356201
1.03999996185303 7.16281700134277
1.04999995231628 7.20315837860107
1.05999994277954 7.24359893798828
1.07000005245209 7.28413963317871
1.08000004291534 7.3247766494751
1.0900000333786 7.36551094055176
1.10000002384186 7.40633964538574
1.11000001430511 7.44726324081421
1.12000000476837 7.48827981948853
1.12999999523163 7.52938890457153
1.13999998569489 7.57058811187744
1.14999997615814 7.61187839508057
1.1599999666214 7.6532564163208
1.16999995708466 7.69472408294678
1.17999994754791 7.73627710342407
1.19000005722046 7.77791738510132
1.20000004768372 7.81964254379272
1.21000003814697 7.86145114898682
1.22000002861023 7.90334415435791
1.23000001907349 7.9453182220459
1.24000000953674 7.98737382888794
1.25 8.02951049804688
1.25999999046326 8.07172679901123
1.26999998092651 8.11402130126953
1.27999997138977 8.15639400482178
1.28999996185303 8.19884490966797
1.29999995231628 8.24137115478516
1.30999994277954 8.28397274017334
1.32000005245209 8.32664966583252
1.33000004291534 8.36940002441406
1.3400000333786 8.41222286224365
1.35000002384186 8.45511817932129
1.36000001430511 8.49808597564697
1.37000000476837 8.54112339019775
1.37999999523163 8.58423233032227
1.38999998569489 8.62740898132324
1.39999997615814 8.67065620422363
1.4099999666214 8.71396923065186
1.41999995708466 8.75735092163086
1.42999994754791 8.80079746246338
1.44000005722046 8.84431171417236
1.45000004768372 8.88788986206055
1.46000003814697 8.93153285980225
1.47000002861023 8.97523880004883
1.48000001907349 9.01900768280029
1.49000000953674 9.06283950805664
1.5 9.10673332214355
1.50999999046326 9.15068912506104
1.51999998092651 9.19470500946045
1.52999997138977 9.23878002166748
1.53999996185303 9.28291511535645
1.54999995231628 9.32710838317871
1.55999994277954 9.37135887145996
1.57000005245209 9.41566848754883
1.58000004291534 9.46003437042236
1.5900000333786 9.50445747375488
1.60000002384186 9.54893493652344
1.61000001430511 9.59346866607666
1.62000000476837 9.63805675506592
1.62999999523163 9.68269729614258
1.63999998569489 9.72739315032959
1.64999997615814 9.77214241027832
1.6599999666214 9.81694316864014
1.66999995708466 9.86179447174072
1.67999994754791 9.90669918060303
1.69000005722046 9.95165538787842
1.70000004768372 9.99665927886963
1.71000003814697 10.0417156219482
1.72000002861023 10.086820602417
1.73000001907349 10.1319732666016
1.74000000953674 10.1771745681763
1.75 10.2224245071411
1.75999999046326 10.2677221298218
1.76999998092651 10.3130655288696
1.77999997138977 10.3584566116333
1.78999996185303 10.4038925170898
1.79999995231628 10.4493751525879
1.80999994277954 10.4949007034302
1.82000005245209 10.540472984314
1.83000004291534 10.5860891342163
1.8400000333786 10.6317472457886
1.85000002384186 10.6774501800537
1.86000001430511 10.7231960296631
1.87000000476837 10.7689847946167
1.87999999523163 10.8148126602173
1.88999998569489 10.8606843948364
1.89999997615814 10.9065971374512
1.9099999666214 10.9525499343872
1.91999995708466 10.9985437393188
1.92999994754791 11.0445775985718
1.94000005722046 11.090651512146
1.95000004768372 11.1367616653442
1.96000003814697 11.1829128265381
1.97000002861023 11.2291021347046
1.98000001907349 11.2753286361694
1.99000000953674 11.3215932846069
};
\addlegendentry{Loss function $f$}
\addplot [line width=1pt, color0]
table {%
-2 18.2714900970459
-1.99000000953674 18.1497402191162
-1.98000001907349 18.0279903411865
-1.97000002861023 17.9062404632568
-1.96000003814697 17.7844905853271
-1.95000004768372 17.6627407073975
-1.94000005722046 17.5409908294678
-1.92999994754791 17.4192409515381
-1.91999995708466 17.2974910736084
-1.9099999666214 17.1757411956787
-1.89999997615814 17.053991317749
-1.88999998569489 16.9322414398193
-1.87999999523163 16.8104915618896
-1.87000000476837 16.68874168396
-1.86000001430511 16.5669918060303
-1.85000002384186 16.4452419281006
-1.8400000333786 16.3234920501709
-1.83000004291534 16.2017421722412
-1.82000005245209 16.0799922943115
-1.80999994277954 15.9582414627075
-1.79999995231628 15.8364925384521
-1.78999996185303 15.7147426605225
-1.77999997138977 15.5929927825928
-1.76999998092651 15.4712429046631
-1.75999999046326 15.3494930267334
-1.75 15.2277431488037
-1.74000000953674 15.105993270874
-1.73000001907349 14.9842433929443
-1.72000002861023 14.8624935150146
-1.71000003814697 14.740743637085
-1.70000004768372 14.6189937591553
-1.69000005722046 14.4972448348999
-1.67999994754791 14.3754930496216
-1.66999995708466 14.2537441253662
-1.6599999666214 14.1319942474365
-1.64999997615814 14.0102443695068
-1.63999998569489 13.8884944915771
-1.62999999523163 13.7667446136475
-1.62000000476837 13.6449947357178
-1.61000001430511 13.5232448577881
-1.60000002384186 13.4014949798584
-1.5900000333786 13.2797451019287
-1.58000004291534 13.1579961776733
-1.57000005245209 13.0362462997437
-1.55999994277954 12.9144945144653
-1.54999995231628 12.7927446365356
-1.53999996185303 12.6709957122803
-1.52999997138977 12.5492458343506
-1.51999998092651 12.4274959564209
-1.50999999046326 12.3057460784912
-1.5 12.1839962005615
-1.49000000953674 12.0622463226318
-1.48000001907349 11.9404964447021
-1.47000002861023 11.8187465667725
-1.46000003814697 11.6969966888428
-1.45000004768372 11.5752477645874
-1.44000005722046 11.4534978866577
-1.42999994754791 11.3317461013794
-1.41999995708466 11.2099962234497
-1.4099999666214 11.0882472991943
-1.39999997615814 10.9664974212646
-1.38999998569489 10.844747543335
-1.37999999523163 10.7229976654053
-1.37000000476837 10.6012477874756
-1.36000001430511 10.4794979095459
-1.35000002384186 10.3577480316162
-1.3400000333786 10.2359981536865
-1.33000004291534 10.1142482757568
-1.32000005245209 9.99249935150146
-1.30999994277954 9.87074756622314
-1.29999995231628 9.74899864196777
-1.28999996185303 9.62724876403809
-1.27999997138977 9.5054988861084
-1.26999998092651 9.38374900817871
-1.25999999046326 9.26199913024902
-1.25 9.14024925231934
-1.24000000953674 9.01849937438965
-1.23000001907349 8.89674949645996
-1.22000002861023 8.77499961853027
-1.21000003814697 8.65324974060059
-1.20000004768372 8.5314998626709
-1.19000005722046 8.40975093841553
-1.17999994754791 8.28799915313721
-1.16999995708466 8.16625022888184
-1.1599999666214 8.04450035095215
-1.14999997615814 7.9227499961853
-1.13999998569489 7.80100059509277
-1.12999999523163 7.67925071716309
-1.12000000476837 7.5575008392334
-1.11000001430511 7.43575096130371
-1.10000002384186 7.31400156021118
-1.0900000333786 7.19225168228149
-1.08000004291534 7.07050180435181
-1.07000005245209 6.94875192642212
-1.05999994277954 6.82700109481812
-1.04999995231628 6.70525121688843
-1.03999996185303 6.58350133895874
-1.02999997138977 6.46175146102905
-1.01999998092651 6.34000205993652
-1.00999999046326 6.21825218200684
-1 6.09650230407715
-0.990000009536743 5.97475242614746
-0.980000019073486 5.85300254821777
-0.970000028610229 5.73125314712524
-0.959999978542328 5.60950231552124
-0.949999988079071 5.48775291442871
-0.939999997615814 5.36600303649902
-0.930000007152557 5.24425315856934
-0.920000016689301 5.12250328063965
-0.910000026226044 5.00075387954712
-0.899999976158142 4.87900304794312
-0.889999985694885 4.75725364685059
-0.879999995231628 4.6355037689209
-0.870000004768372 4.51375389099121
-0.860000014305115 4.39200401306152
-0.850000023841858 4.27025461196899
-0.839999973773956 4.14850425720215
-0.829999983310699 4.02675437927246
-0.819999992847443 3.90500450134277
-0.810000002384186 3.78325462341309
-0.800000011920929 3.6615047454834
-0.790000021457672 3.53975486755371
-0.779999971389771 3.41800498962402
-0.769999980926514 3.29625511169434
-0.759999990463257 3.17450523376465
-0.75 3.05275535583496
-0.740000009536743 2.93100547790527
-0.730000019073486 2.80925559997559
-0.720000028610229 2.6875057220459
-0.709999978542328 2.56575584411621
-0.699999988079071 2.44400596618652
-0.689999997615814 2.32225608825684
-0.680000007152557 2.20050621032715
-0.670000016689301 2.07875633239746
-0.660000026226044 1.95700645446777
-0.649999976158142 1.83525657653809
-0.639999985694885 1.7135066986084
-0.629999995231628 1.59175682067871
-0.620000004768372 1.47000694274902
-0.610000014305115 1.34825706481934
-0.600000023841858 1.22650718688965
-0.589999973773956 1.10475730895996
-0.579999983310699 0.983007431030273
-0.569999992847443 0.861257553100586
-0.560000002384186 0.739507675170898
-0.550000011920929 0.617757797241211
-0.540000021457672 0.496007919311523
-0.529999971389771 0.374258041381836
-0.519999980926514 0.252508163452148
-0.509999990463257 0.130758285522461
-0.5 0.00900840759277344
-0.490000009536743 -0.112741470336914
-0.479999989271164 -0.234491348266602
-0.469999998807907 -0.356241226196289
-0.46000000834465 -0.477991104125977
-0.449999988079071 -0.599740028381348
-0.439999997615814 -0.721489906311035
-0.430000007152557 -0.843239784240723
-0.419999986886978 -0.964991569519043
-0.409999996423721 -1.08674049377441
-0.400000005960464 -1.2084903717041
-0.389999985694885 -1.33024024963379
-0.379999995231628 -1.45199012756348
-0.370000004768372 -1.57374000549316
-0.360000014305115 -1.69548988342285
-0.349999994039536 -1.81723976135254
-0.340000003576279 -1.93898963928223
-0.330000013113022 -2.0607385635376
-0.319999992847443 -2.18249034881592
-0.310000002384186 -2.30424022674561
-0.300000011920929 -2.42599010467529
-0.28999999165535 -2.54773902893066
-0.280000001192093 -2.66948890686035
-0.270000010728836 -2.79123878479004
-0.259999990463257 -2.91298866271973
-0.25 -3.03473854064941
-0.239999994635582 -3.1564884185791
-0.230000004172325 -3.27823829650879
-0.219999998807907 -3.39998817443848
-0.209999993443489 -3.52173805236816
-0.200000002980232 -3.64348697662354
-0.189999997615814 -3.76523685455322
-0.180000007152557 -3.88698768615723
-0.170000001788139 -4.00873756408691
-0.159999996423721 -4.1304874420166
-0.150000005960464 -4.25223731994629
-0.140000000596046 -4.37398719787598
-0.129999995231628 -4.49573707580566
-0.119999997317791 -4.61748695373535
-0.109999999403954 -4.73923683166504
-0.100000001490116 -4.86098670959473
-0.0900000035762787 -4.98273658752441
-0.0799999982118607 -5.1044864654541
-0.0700000002980232 -5.22623634338379
-0.0599999986588955 -5.34798622131348
-0.0500000007450581 -5.46973609924316
-0.0399999991059303 -5.59148597717285
-0.0299999993294477 -5.71323585510254
-0.0199999995529652 -5.83498573303223
-0.00999999977648258 -5.95673561096191
0 -6.0784854888916
0.00999999977648258 -6.20023536682129
0.0199999995529652 -6.32198524475098
0.0299999993294477 -6.44373512268066
0.0399999991059303 -6.56548500061035
0.0500000007450581 -6.68723487854004
0.0599999986588955 -6.80898475646973
0.0700000002980232 -6.93073463439941
0.0799999982118607 -7.0524845123291
0.0900000035762787 -7.17423439025879
0.100000001490116 -7.29598426818848
0.109999999403954 -7.41773414611816
0.119999997317791 -7.53948402404785
0.129999995231628 -7.66123390197754
0.140000000596046 -7.78298377990723
0.150000005960464 -7.90473365783691
0.159999996423721 -8.0264835357666
0.170000001788139 -8.14823341369629
0.180000007152557 -8.26998329162598
0.189999997615814 -8.39173316955566
0.200000002980232 -8.51348304748535
0.209999993443489 -8.63523292541504
0.219999998807907 -8.75698280334473
0.230000004172325 -8.87873268127441
0.239999994635582 -9.0004825592041
0.25 -9.12223243713379
0.259999990463257 -9.24398231506348
0.270000010728836 -9.36573219299316
0.280000001192093 -9.48748207092285
0.28999999165535 -9.60923194885254
0.300000011920929 -9.73098182678223
0.310000002384186 -9.85273170471191
0.319999992847443 -9.9744815826416
0.330000013113022 -10.0962314605713
0.340000003576279 -10.217981338501
0.349999994039536 -10.3397312164307
0.360000014305115 -10.4614810943604
0.370000004768372 -10.58323097229
0.379999995231628 -10.7049808502197
0.389999985694885 -10.8267307281494
0.400000005960464 -10.9484806060791
0.409999996423721 -11.0702304840088
0.419999986886978 -11.1919803619385
0.430000007152557 -11.3137302398682
0.439999997615814 -11.4354801177979
0.449999988079071 -11.5572299957275
0.46000000834465 -11.6789798736572
0.469999998807907 -11.8007297515869
0.479999989271164 -11.9224796295166
0.490000009536743 -12.0442295074463
0.5 -12.165979385376
0.509999990463257 -12.2877292633057
0.519999980926514 -12.4094791412354
0.529999971389771 -12.531229019165
0.540000021457672 -12.6529788970947
0.550000011920929 -12.7747287750244
0.560000002384186 -12.8964786529541
0.569999992847443 -13.0182285308838
0.579999983310699 -13.1399784088135
0.589999973773956 -13.2617263793945
0.600000023841858 -13.3834762573242
0.610000014305115 -13.5052299499512
0.620000004768372 -13.6269760131836
0.629999995231628 -13.7487297058105
0.639999985694885 -13.870475769043
0.649999976158142 -13.9922294616699
0.660000026226044 -14.1139793395996
0.670000016689301 -14.2357273101807
0.680000007152557 -14.3574771881104
0.689999997615814 -14.47922706604
0.699999988079071 -14.6009769439697
0.709999978542328 -14.7227268218994
0.720000028610229 -14.8444766998291
0.730000019073486 -14.9662265777588
0.740000009536743 -15.0879764556885
0.75 -15.2097263336182
0.759999990463257 -15.3314762115479
0.769999980926514 -15.4532260894775
0.779999971389771 -15.5749759674072
0.790000021457672 -15.6967258453369
0.800000011920929 -15.8184757232666
0.810000002384186 -15.9402256011963
0.819999992847443 -16.061975479126
0.829999983310699 -16.1837253570557
0.839999973773956 -16.3054733276367
0.850000023841858 -16.4272232055664
0.860000014305115 -16.5489768981934
0.870000004768372 -16.6707229614258
0.879999995231628 -16.7924766540527
0.889999985694885 -16.9142227172852
0.899999976158142 -17.0359764099121
0.910000026226044 -17.1577262878418
0.920000016689301 -17.2794742584229
0.930000007152557 -17.4012241363525
0.939999997615814 -17.5229740142822
0.949999988079071 -17.6447238922119
0.959999978542328 -17.7664737701416
0.970000028610229 -17.8882236480713
0.980000019073486 -18.009973526001
0.990000009536743 -18.1317234039307
1 -18.2534732818604
1.00999999046326 -18.37522315979
1.01999998092651 -18.4969730377197
1.02999997138977 -18.6187229156494
1.03999996185303 -18.7404727935791
1.04999995231628 -18.8622226715088
1.05999994277954 -18.9839725494385
1.07000005245209 -19.1057243347168
1.08000004291534 -19.2274723052979
1.0900000333786 -19.3492240905762
1.10000002384186 -19.4709701538086
1.11000001430511 -19.5927238464355
1.12000000476837 -19.714469909668
1.12999999523163 -19.8362216949463
1.13999998569489 -19.957971572876
1.14999997615814 -20.0797214508057
1.1599999666214 -20.2014713287354
1.16999995708466 -20.323221206665
1.17999994754791 -20.4449672698975
1.19000005722046 -20.5667209625244
1.20000004768372 -20.6884708404541
1.21000003814697 -20.8102207183838
1.22000002861023 -20.9319705963135
1.23000001907349 -21.0537204742432
1.24000000953674 -21.1754703521729
1.25 -21.2972202301025
1.25999999046326 -21.4189701080322
1.26999998092651 -21.5407199859619
1.27999997138977 -21.6624698638916
1.28999996185303 -21.7842197418213
1.29999995231628 -21.905969619751
1.30999994277954 -22.0277194976807
1.32000005245209 -22.1494731903076
1.33000004291534 -22.27121925354
1.3400000333786 -22.3929691314697
1.35000002384186 -22.5147190093994
1.36000001430511 -22.6364688873291
1.37000000476837 -22.7582187652588
1.37999999523163 -22.8799686431885
1.38999998569489 -23.0017185211182
1.39999997615814 -23.1234683990479
1.4099999666214 -23.2452182769775
1.41999995708466 -23.3669681549072
1.42999994754791 -23.4887142181396
1.44000005722046 -23.6104679107666
1.45000004768372 -23.7322177886963
1.46000003814697 -23.853967666626
1.47000002861023 -23.9757175445557
1.48000001907349 -24.0974674224854
1.49000000953674 -24.219217300415
1.5 -24.3409671783447
1.50999999046326 -24.4627170562744
1.51999998092651 -24.5844669342041
1.52999997138977 -24.7062168121338
1.53999996185303 -24.8279666900635
1.54999995231628 -24.9497165679932
1.55999994277954 -25.0714664459229
1.57000005245209 -25.1932201385498
1.58000004291534 -25.3149662017822
1.5900000333786 -25.4367160797119
1.60000002384186 -25.5584659576416
1.61000001430511 -25.6802158355713
1.62000000476837 -25.801965713501
1.62999999523163 -25.9237155914307
1.63999998569489 -26.0454654693604
1.64999997615814 -26.16721534729
1.6599999666214 -26.2889652252197
1.66999995708466 -26.4107151031494
1.67999994754791 -26.5324611663818
1.69000005722046 -26.6542148590088
1.70000004768372 -26.7759647369385
1.71000003814697 -26.8977146148682
1.72000002861023 -27.0194644927979
1.73000001907349 -27.1412143707275
1.74000000953674 -27.2629642486572
1.75 -27.3847141265869
1.75999999046326 -27.5064640045166
1.76999998092651 -27.6282138824463
1.77999997138977 -27.749963760376
1.78999996185303 -27.8717136383057
1.79999995231628 -27.9934635162354
1.80999994277954 -28.115213394165
1.82000005245209 -28.236967086792
1.83000004291534 -28.3587131500244
1.8400000333786 -28.4804630279541
1.85000002384186 -28.6022129058838
1.86000001430511 -28.7239627838135
1.87000000476837 -28.8457126617432
1.87999999523163 -28.9674625396729
1.88999998569489 -29.0892124176025
1.89999997615814 -29.2109622955322
1.9099999666214 -29.3327121734619
1.91999995708466 -29.4544620513916
1.92999994754791 -29.576208114624
1.94000005722046 -29.697961807251
1.95000004768372 -29.8197116851807
1.96000003814697 -29.9414615631104
1.97000002861023 -30.06321144104
1.98000001907349 -30.1849613189697
1.99000000953674 -30.3067111968994
};
\addlegendentry{Linearization at $\w_t$}
\addplot [line width=1pt, color1]
table {%
-2 3.9776074886322
-1.99000000953674 3.9776074886322
-1.98000001907349 3.9776074886322
-1.97000002861023 3.9776074886322
-1.96000003814697 3.9776074886322
-1.95000004768372 3.9776074886322
-1.94000005722046 3.9776074886322
-1.92999994754791 3.9776074886322
-1.91999995708466 3.9776074886322
-1.9099999666214 3.9776074886322
-1.89999997615814 3.9776074886322
-1.88999998569489 3.9776074886322
-1.87999999523163 3.9776074886322
-1.87000000476837 3.9776074886322
-1.86000001430511 3.9776074886322
-1.85000002384186 3.9776074886322
-1.8400000333786 3.9776074886322
-1.83000004291534 3.9776074886322
-1.82000005245209 3.9776074886322
-1.80999994277954 3.9776074886322
-1.79999995231628 3.9776074886322
-1.78999996185303 3.9776074886322
-1.77999997138977 3.9776074886322
-1.76999998092651 3.9776074886322
-1.75999999046326 3.9776074886322
-1.75 3.9776074886322
-1.74000000953674 3.9776074886322
-1.73000001907349 3.9776074886322
-1.72000002861023 3.9776074886322
-1.71000003814697 3.9776074886322
-1.70000004768372 3.9776074886322
-1.69000005722046 3.9776074886322
-1.67999994754791 3.9776074886322
-1.66999995708466 3.9776074886322
-1.6599999666214 3.9776074886322
-1.64999997615814 3.9776074886322
-1.63999998569489 3.9776074886322
-1.62999999523163 3.9776074886322
-1.62000000476837 3.9776074886322
-1.61000001430511 3.9776074886322
-1.60000002384186 3.9776074886322
-1.5900000333786 3.9776074886322
-1.58000004291534 3.9776074886322
-1.57000005245209 3.9776074886322
-1.55999994277954 3.9776074886322
-1.54999995231628 3.9776074886322
-1.53999996185303 3.9776074886322
-1.52999997138977 3.9776074886322
-1.51999998092651 3.9776074886322
-1.50999999046326 3.9776074886322
-1.5 3.9776074886322
-1.49000000953674 3.9776074886322
-1.48000001907349 3.9776074886322
-1.47000002861023 3.9776074886322
-1.46000003814697 3.9776074886322
-1.45000004768372 3.9776074886322
-1.44000005722046 3.9776074886322
-1.42999994754791 3.9776074886322
-1.41999995708466 3.9776074886322
-1.4099999666214 3.9776074886322
-1.39999997615814 3.9776074886322
-1.38999998569489 3.9776074886322
-1.37999999523163 3.9776074886322
-1.37000000476837 3.9776074886322
-1.36000001430511 3.9776074886322
-1.35000002384186 3.9776074886322
-1.3400000333786 3.9776074886322
-1.33000004291534 3.9776074886322
-1.32000005245209 3.9776074886322
-1.30999994277954 3.9776074886322
-1.29999995231628 3.9776074886322
-1.28999996185303 3.9776074886322
-1.27999997138977 3.9776074886322
-1.26999998092651 3.9776074886322
-1.25999999046326 3.9776074886322
-1.25 3.9776074886322
-1.24000000953674 3.9776074886322
-1.23000001907349 3.9776074886322
-1.22000002861023 3.9776074886322
-1.21000003814697 3.9776074886322
-1.20000004768372 3.9776074886322
-1.19000005722046 3.9776074886322
-1.17999994754791 3.9776074886322
-1.16999995708466 3.9776074886322
-1.1599999666214 3.9776074886322
-1.14999997615814 3.9776074886322
-1.13999998569489 3.9776074886322
-1.12999999523163 3.9776074886322
-1.12000000476837 3.9776074886322
-1.11000001430511 3.9776074886322
-1.10000002384186 3.9776074886322
-1.0900000333786 3.9776074886322
-1.08000004291534 3.9776074886322
-1.07000005245209 3.9776074886322
-1.05999994277954 3.9776074886322
-1.04999995231628 3.9776074886322
-1.03999996185303 3.9776074886322
-1.02999997138977 3.9776074886322
-1.01999998092651 3.9776074886322
-1.00999999046326 3.9776074886322
-1 3.9776074886322
-0.990000009536743 3.9776074886322
-0.980000019073486 3.9776074886322
-0.970000028610229 3.9776074886322
-0.959999978542328 3.9776074886322
-0.949999988079071 3.9776074886322
-0.939999997615814 3.9776074886322
-0.930000007152557 3.9776074886322
-0.920000016689301 3.9776074886322
-0.910000026226044 3.9776074886322
-0.899999976158142 3.9776074886322
-0.889999985694885 3.9776074886322
-0.879999995231628 3.9776074886322
-0.870000004768372 3.9776074886322
-0.860000014305115 3.9776074886322
-0.850000023841858 3.9776074886322
-0.839999973773956 3.9776074886322
-0.829999983310699 3.9776074886322
-0.819999992847443 3.9776074886322
-0.810000002384186 3.9776074886322
-0.800000011920929 3.9776074886322
-0.790000021457672 3.9776074886322
-0.779999971389771 3.9776074886322
-0.769999980926514 3.9776074886322
-0.759999990463257 3.9776074886322
-0.75 3.9776074886322
-0.740000009536743 3.9776074886322
-0.730000019073486 3.9776074886322
-0.720000028610229 3.9776074886322
-0.709999978542328 3.9776074886322
-0.699999988079071 3.9776074886322
-0.689999997615814 3.9776074886322
-0.680000007152557 3.9776074886322
-0.670000016689301 3.9776074886322
-0.660000026226044 3.9776074886322
-0.649999976158142 3.9776074886322
-0.639999985694885 3.9776074886322
-0.629999995231628 3.9776074886322
-0.620000004768372 3.9776074886322
-0.610000014305115 3.9776074886322
-0.600000023841858 3.9776074886322
-0.589999973773956 3.9776074886322
-0.579999983310699 3.9776074886322
-0.569999992847443 3.9776074886322
-0.560000002384186 3.9776074886322
-0.550000011920929 3.9776074886322
-0.540000021457672 3.9776074886322
-0.529999971389771 3.9776074886322
-0.519999980926514 3.9776074886322
-0.509999990463257 3.9776074886322
-0.5 3.9776074886322
-0.490000009536743 3.9776074886322
-0.479999989271164 3.9776074886322
-0.469999998807907 3.9776074886322
-0.46000000834465 3.9776074886322
-0.449999988079071 3.9776074886322
-0.439999997615814 3.9776074886322
-0.430000007152557 3.9776074886322
-0.419999986886978 3.9776074886322
-0.409999996423721 3.9776074886322
-0.400000005960464 3.9776074886322
-0.389999985694885 3.9776074886322
-0.379999995231628 3.9776074886322
-0.370000004768372 3.9776074886322
-0.360000014305115 3.9776074886322
-0.349999994039536 3.9776074886322
-0.340000003576279 3.9776074886322
-0.330000013113022 3.9776074886322
-0.319999992847443 3.9776074886322
-0.310000002384186 3.9776074886322
-0.300000011920929 3.9776074886322
-0.28999999165535 3.9776074886322
-0.280000001192093 3.9776074886322
-0.270000010728836 3.9776074886322
-0.259999990463257 3.9776074886322
-0.25 3.9776074886322
-0.239999994635582 3.9776074886322
-0.230000004172325 3.9776074886322
-0.219999998807907 3.9776074886322
-0.209999993443489 3.9776074886322
-0.200000002980232 3.9776074886322
-0.189999997615814 3.9776074886322
-0.180000007152557 3.9776074886322
-0.170000001788139 3.9776074886322
-0.159999996423721 3.9776074886322
-0.150000005960464 3.9776074886322
-0.140000000596046 3.9776074886322
-0.129999995231628 3.9776074886322
-0.119999997317791 3.9776074886322
-0.109999999403954 3.9776074886322
-0.100000001490116 3.9776074886322
-0.0900000035762787 3.9776074886322
-0.0799999982118607 3.9776074886322
-0.0700000002980232 3.9776074886322
-0.0599999986588955 3.9776074886322
-0.0500000007450581 3.9776074886322
-0.0399999991059303 3.9776074886322
-0.0299999993294477 3.9776074886322
-0.0199999995529652 3.9776074886322
-0.00999999977648258 3.9776074886322
0 3.9776074886322
0.00999999977648258 3.9776074886322
0.0199999995529652 3.9776074886322
0.0299999993294477 3.9776074886322
0.0399999991059303 3.9776074886322
0.0500000007450581 3.9776074886322
0.0599999986588955 3.9776074886322
0.0700000002980232 3.9776074886322
0.0799999982118607 3.9776074886322
0.0900000035762787 3.9776074886322
0.100000001490116 3.9776074886322
0.109999999403954 3.9776074886322
0.119999997317791 3.9776074886322
0.129999995231628 3.9776074886322
0.140000000596046 3.9776074886322
0.150000005960464 3.9776074886322
0.159999996423721 3.9776074886322
0.170000001788139 3.9776074886322
0.180000007152557 3.9776074886322
0.189999997615814 3.9776074886322
0.200000002980232 3.9776074886322
0.209999993443489 3.9776074886322
0.219999998807907 3.9776074886322
0.230000004172325 3.9776074886322
0.239999994635582 3.9776074886322
0.25 3.9776074886322
0.259999990463257 3.9776074886322
0.270000010728836 3.9776074886322
0.280000001192093 3.9776074886322
0.28999999165535 3.9776074886322
0.300000011920929 3.9776074886322
0.310000002384186 3.9776074886322
0.319999992847443 3.9776074886322
0.330000013113022 3.9776074886322
0.340000003576279 3.9776074886322
0.349999994039536 3.9776074886322
0.360000014305115 3.9776074886322
0.370000004768372 3.9776074886322
0.379999995231628 3.9776074886322
0.389999985694885 3.9776074886322
0.400000005960464 3.9776074886322
0.409999996423721 3.9776074886322
0.419999986886978 3.9776074886322
0.430000007152557 3.9776074886322
0.439999997615814 3.9776074886322
0.449999988079071 3.9776074886322
0.46000000834465 3.9776074886322
0.469999998807907 3.9776074886322
0.479999989271164 3.9776074886322
0.490000009536743 3.9776074886322
0.5 3.9776074886322
0.509999990463257 3.9776074886322
0.519999980926514 3.9776074886322
0.529999971389771 3.9776074886322
0.540000021457672 3.9776074886322
0.550000011920929 3.9776074886322
0.560000002384186 3.9776074886322
0.569999992847443 3.9776074886322
0.579999983310699 3.9776074886322
0.589999973773956 3.9776074886322
0.600000023841858 3.9776074886322
0.610000014305115 3.9776074886322
0.620000004768372 3.9776074886322
0.629999995231628 3.9776074886322
0.639999985694885 3.9776074886322
0.649999976158142 3.9776074886322
0.660000026226044 3.9776074886322
0.670000016689301 3.9776074886322
0.680000007152557 3.9776074886322
0.689999997615814 3.9776074886322
0.699999988079071 3.9776074886322
0.709999978542328 3.9776074886322
0.720000028610229 3.9776074886322
0.730000019073486 3.9776074886322
0.740000009536743 3.9776074886322
0.75 3.9776074886322
0.759999990463257 3.9776074886322
0.769999980926514 3.9776074886322
0.779999971389771 3.9776074886322
0.790000021457672 3.9776074886322
0.800000011920929 3.9776074886322
0.810000002384186 3.9776074886322
0.819999992847443 3.9776074886322
0.829999983310699 3.9776074886322
0.839999973773956 3.9776074886322
0.850000023841858 3.9776074886322
0.860000014305115 3.9776074886322
0.870000004768372 3.9776074886322
0.879999995231628 3.9776074886322
0.889999985694885 3.9776074886322
0.899999976158142 3.9776074886322
0.910000026226044 3.9776074886322
0.920000016689301 3.9776074886322
0.930000007152557 3.9776074886322
0.939999997615814 3.9776074886322
0.949999988079071 3.9776074886322
0.959999978542328 3.9776074886322
0.970000028610229 3.9776074886322
0.980000019073486 3.9776074886322
0.990000009536743 3.9776074886322
1 3.9776074886322
1.00999999046326 3.9776074886322
1.01999998092651 3.9776074886322
1.02999997138977 3.9776074886322
1.03999996185303 3.9776074886322
1.04999995231628 3.9776074886322
1.05999994277954 3.9776074886322
1.07000005245209 3.9776074886322
1.08000004291534 3.9776074886322
1.0900000333786 3.9776074886322
1.10000002384186 3.9776074886322
1.11000001430511 3.9776074886322
1.12000000476837 3.9776074886322
1.12999999523163 3.9776074886322
1.13999998569489 3.9776074886322
1.14999997615814 3.9776074886322
1.1599999666214 3.9776074886322
1.16999995708466 3.9776074886322
1.17999994754791 3.9776074886322
1.19000005722046 3.9776074886322
1.20000004768372 3.9776074886322
1.21000003814697 3.9776074886322
1.22000002861023 3.9776074886322
1.23000001907349 3.9776074886322
1.24000000953674 3.9776074886322
1.25 3.9776074886322
1.25999999046326 3.9776074886322
1.26999998092651 3.9776074886322
1.27999997138977 3.9776074886322
1.28999996185303 3.9776074886322
1.29999995231628 3.9776074886322
1.30999994277954 3.9776074886322
1.32000005245209 3.9776074886322
1.33000004291534 3.9776074886322
1.3400000333786 3.9776074886322
1.35000002384186 3.9776074886322
1.36000001430511 3.9776074886322
1.37000000476837 3.9776074886322
1.37999999523163 3.9776074886322
1.38999998569489 3.9776074886322
1.39999997615814 3.9776074886322
1.4099999666214 3.9776074886322
1.41999995708466 3.9776074886322
1.42999994754791 3.9776074886322
1.44000005722046 3.9776074886322
1.45000004768372 3.9776074886322
1.46000003814697 3.9776074886322
1.47000002861023 3.9776074886322
1.48000001907349 3.9776074886322
1.49000000953674 3.9776074886322
1.5 3.9776074886322
1.50999999046326 3.9776074886322
1.51999998092651 3.9776074886322
1.52999997138977 3.9776074886322
1.53999996185303 3.9776074886322
1.54999995231628 3.9776074886322
1.55999994277954 3.9776074886322
1.57000005245209 3.9776074886322
1.58000004291534 3.9776074886322
1.5900000333786 3.9776074886322
1.60000002384186 3.9776074886322
1.61000001430511 3.9776074886322
1.62000000476837 3.9776074886322
1.62999999523163 3.9776074886322
1.63999998569489 3.9776074886322
1.64999997615814 3.9776074886322
1.6599999666214 3.9776074886322
1.66999995708466 3.9776074886322
1.67999994754791 3.9776074886322
1.69000005722046 3.9776074886322
1.70000004768372 3.9776074886322
1.71000003814697 3.9776074886322
1.72000002861023 3.9776074886322
1.73000001907349 3.9776074886322
1.74000000953674 3.9776074886322
1.75 3.9776074886322
1.75999999046326 3.9776074886322
1.76999998092651 3.9776074886322
1.77999997138977 3.9776074886322
1.78999996185303 3.9776074886322
1.79999995231628 3.9776074886322
1.80999994277954 3.9776074886322
1.82000005245209 3.9776074886322
1.83000004291534 3.9776074886322
1.8400000333786 3.9776074886322
1.85000002384186 3.9776074886322
1.86000001430511 3.9776074886322
1.87000000476837 3.9776074886322
1.87999999523163 3.9776074886322
1.88999998569489 3.9776074886322
1.89999997615814 3.9776074886322
1.9099999666214 3.9776074886322
1.91999995708466 3.9776074886322
1.92999994754791 3.9776074886322
1.94000005722046 3.9776074886322
1.95000004768372 3.9776074886322
1.96000003814697 3.9776074886322
1.97000002861023 3.9776074886322
1.98000001907349 3.9776074886322
1.99000000953674 3.9776074886322
};
\addlegendentry{Minimum $\fstar$}
\end{axis}

\end{tikzpicture}

%% file: include/rsi_plot_arxiv.tex
% This file was created by tikzplotlib v0.9.3.
\begin{tikzpicture}

\definecolor{color0}{rgb}{0.12156862745098,0.466666666666667,0.705882352941177}

\begin{axis}[
axis line style={white!80!black},
compat=newest,
height=0.225\textwidth,
legend cell align={left},
legend style={at={(0.5,1.2)}, anchor=south, draw=lightgray!20.0!black},
tick align=outside,
width=0.3\textwidth,
x grid style={white!80!black},
x grid style={white},
xmajorgrids,
xmajorticks=false,
xmajorticks=true,
xmin=-0.7, xmax=0.7,
xtick style={color=white!15!black},
xtick={-0.6,0,0.6},
xticklabels={\(\displaystyle \scriptscriptstyle w_t=-\frac{3}{5}\),\Arrow{1.0cm},\(\displaystyle \scriptscriptstyle w_{t+1}=\frac{3}{5}\)},
y grid style={white!80!black},
ymajorgrids,
ymajorticks=false,
ymajorticks=true,
ymin=0, ymax=0.15,
yminorticks=false,
ytick style={color=white!15!black},
ytick={0},
yticklabels={\phantom{a}}
]
\addplot [line width=1pt, black]
table {%
-0.699999988079071 0.14699998497963
-0.689999997615814 0.1475909948349
-0.680000007152557 0.147967994213104
-0.670000016689301 0.148137003183365
-0.660000026226044 0.148104012012482
-0.649999976158142 0.147874981164932
-0.639999985694885 0.147455990314484
-0.629999995231628 0.146852999925613
-0.620000004768372 0.146072000265121
-0.610000014305115 0.14511901140213
-0.600000023841858 0.143999993801117
-0.589999973773956 0.142720997333527
-0.579999983310699 0.141287997364998
-0.569999992847443 0.139706999063492
-0.560000002384186 0.13798400759697
-0.550000011920929 0.13612499833107
-0.540000021457672 0.134136006236076
-0.529999971389771 0.132022991776466
-0.519999980926514 0.129792004823685
-0.509999990463257 0.127448990941048
-0.5 0.125
-0.490000009536743 0.1224509999156
-0.479999989271164 0.119808003306389
-0.469999998807907 0.117077000439167
-0.46000000834465 0.114264003932476
-0.449999988079071 0.111374996602535
-0.439999997615814 0.108415998518467
-0.430000007152557 0.105392999947071
-0.419999986886978 0.102311998605728
-0.409999996423721 0.0991789996623993
-0.400000005960464 0.0960000082850456
-0.389999985694885 0.0927809923887253
-0.379999995231628 0.089528001844883
-0.370000004768372 0.0862470045685768
-0.360000014305115 0.0829440057277679
-0.349999994039536 0.0796249955892563
-0.340000003576279 0.0762960016727448
-0.330000013113022 0.0729630067944527
-0.319999992847443 0.0696319937705994
-0.310000002384186 0.066309005022049
-0.300000011920929 0.063000001013279
-0.28999999165535 0.0597109943628311
-0.280000001192093 0.0564480014145374
-0.270000010728836 0.0532170012593269
-0.259999990463257 0.0500239990651608
-0.25 0.046875
-0.239999994635582 0.0437759980559349
-0.230000004172325 0.0407330021262169
-0.219999998807907 0.0377519987523556
-0.209999993443489 0.0348389968276024
-0.200000002980232 0.0320000015199184
-0.189999997615814 0.0292409993708134
-0.180000007152557 0.0265680011361837
-0.170000001788139 0.0239870008081198
-0.159999996423721 0.0215039998292923
-0.150000005960464 0.0191249996423721
-0.140000000596046 0.0168559998273849
-0.129999995231628 0.0147029999643564
-0.119999997317791 0.0126719996333122
-0.109999999403954 0.0107690002769232
-0.100000001490116 0.00900000054389238
-0.0900000035762787 0.00737100001424551
-0.0799999982118607 0.00588799966499209
-0.0700000002980232 0.00455700000748038
-0.0599999986588955 0.00338399992324412
-0.0500000007450581 0.00237500015646219
-0.0399999991059303 0.00153599993791431
-0.0299999993294477 0.000873000011779368
-0.0199999995529652 0.000391999987186864
-0.00999999977648258 9.89999971352518e-05
1.11022302462516e-16 1.23259516440783e-32
0.00999999977648258 9.89999971352518e-05
0.0199999995529652 0.000391999987186864
0.0299999993294477 0.000873000011779368
0.0399999991059303 0.00153599993791431
0.0500000007450581 0.00237500015646219
0.0599999986588955 0.00338399992324412
0.0700000002980232 0.00455700000748038
0.0799999982118607 0.00588799966499209
0.0900000035762787 0.00737100001424551
0.100000001490116 0.00900000054389238
0.109999999403954 0.0107690002769232
0.119999997317791 0.0126719996333122
0.129999995231628 0.0147029999643564
0.140000000596046 0.0168559998273849
0.150000005960464 0.0191249996423721
0.159999996423721 0.0215039998292923
0.170000001788139 0.0239870008081198
0.180000007152557 0.0265680011361837
0.189999997615814 0.0292409993708134
0.200000002980232 0.0320000015199184
0.209999993443489 0.0348389968276024
0.219999998807907 0.0377519987523556
0.230000004172325 0.0407330021262169
0.239999994635582 0.0437759980559349
0.25 0.046875
0.259999990463257 0.0500239990651608
0.270000010728836 0.0532170012593269
0.280000001192093 0.0564480014145374
0.28999999165535 0.0597109943628311
0.300000011920929 0.063000001013279
0.310000002384186 0.066309005022049
0.319999992847443 0.0696319937705994
0.330000013113022 0.0729630067944527
0.340000003576279 0.0762960016727448
0.349999994039536 0.0796249955892563
0.360000014305115 0.0829440057277679
0.370000004768372 0.0862470045685768
0.379999995231628 0.089528001844883
0.389999985694885 0.0927809923887253
0.400000005960464 0.0960000082850456
0.409999996423721 0.0991789996623993
0.419999986886978 0.102311998605728
0.430000007152557 0.105392999947071
0.439999997615814 0.108415998518467
0.449999988079071 0.111374996602535
0.46000000834465 0.114264003932476
0.469999998807907 0.117077000439167
0.479999989271164 0.119808003306389
0.490000009536743 0.1224509999156
0.5 0.125
0.509999990463257 0.127448990941048
0.519999980926514 0.129792004823685
0.529999971389771 0.132022991776466
0.540000021457672 0.134136006236076
0.550000011920929 0.13612499833107
0.560000002384186 0.13798400759697
0.569999992847443 0.139706999063492
0.579999983310699 0.141287997364998
0.589999973773956 0.142720997333527
0.600000023841858 0.143999993801117
0.610000014305115 0.14511901140213
0.620000004768372 0.146072000265121
0.629999995231628 0.146852999925613
0.639999985694885 0.147455990314484
0.649999976158142 0.147874981164932
0.660000026226044 0.148104012012482
0.670000016689301 0.148137003183365
0.680000007152557 0.147967994213104
0.689999997615814 0.1475909948349
};
\addlegendentry{$f: w \mapsto w^2 - |w|^3$}
\addplot [line width=0.5pt, color0]
table {%
-0.699999988079071 0.15599998831749
-0.689999997615814 0.154799997806549
-0.680000007152557 0.153599992394447
-0.670000016689301 0.152399986982346
-0.660000026226044 0.151199996471405
-0.649999976158142 0.149999991059303
-0.639999985694885 0.148799985647202
-0.629999995231628 0.147599995136261
-0.620000004768372 0.146399989724159
-0.610000014305115 0.145199999213219
-0.600000023841858 0.143999993801117
-0.589999973773956 0.142799988389015
-0.579999983310699 0.141599982976913
-0.569999992847443 0.140399992465973
-0.560000002384186 0.139199987053871
-0.550000011920929 0.137999996542931
-0.540000021457672 0.136799991130829
-0.529999971389771 0.135599985718727
-0.519999980926514 0.134399995207787
-0.509999990463257 0.133199989795685
-0.5 0.131999984383583
-0.490000009536743 0.130799993872643
-0.479999989271164 0.129599988460541
-0.469999998807907 0.1283999979496
-0.46000000834465 0.127199992537498
-0.449999988079071 0.125999987125397
-0.439999997615814 0.124799989163876
-0.430000007152557 0.123599991202354
-0.419999986886978 0.122399985790253
-0.409999996423721 0.121199987828732
-0.400000005960464 0.11999998986721
-0.389999985694885 0.118799984455109
-0.379999995231628 0.117599993944168
-0.370000004768372 0.116399988532066
-0.360000014305115 0.115199990570545
-0.349999994039536 0.113999992609024
-0.340000003576279 0.112799987196922
-0.330000013113022 0.111599996685982
-0.319999992847443 0.11039999127388
-0.310000002384186 0.109199985861778
-0.300000011920929 0.107999995350838
-0.28999999165535 0.106799989938736
-0.280000001192093 0.105599984526634
-0.270000010728836 0.104399994015694
-0.259999990463257 0.103199988603592
-0.25 0.101999990642071
-0.239999994635582 0.10079999268055
-0.230000004172325 0.0995999872684479
-0.219999998807907 0.0983999893069267
-0.209999993443489 0.0971999913454056
-0.200000002980232 0.0959999859333038
-0.189999997615814 0.0947999879717827
-0.180000007152557 0.0935999900102615
-0.170000001788139 0.0923999920487404
-0.159999996423721 0.0911999866366386
-0.150000005960464 0.0899999886751175
-0.140000000596046 0.0887999832630157
-0.129999995231628 0.0875999927520752
-0.119999997317791 0.0863999873399734
-0.109999999403954 0.0851999893784523
-0.100000001490116 0.0839999914169312
-0.0900000035762787 0.08279999345541
-0.0799999982118607 0.0815999880433083
-0.0700000002980232 0.0803999900817871
-0.0599999986588955 0.079199992120266
-0.0500000007450581 0.0779999867081642
-0.0399999991059303 0.0767999887466431
-0.0299999993294477 0.0755999833345413
-0.0199999995529652 0.0743999853730202
-0.00999999977648258 0.073199987411499
1.11022302462516e-16 0.0719999894499779
0.00999999977648258 0.0707999914884567
0.0199999995529652 0.0695999935269356
0.0299999993294477 0.0683999881148338
0.0399999991059303 0.0671999827027321
0.0500000007450581 0.0659999847412109
0.0599999986588955 0.0647999867796898
0.0700000002980232 0.0635999888181686
0.0799999982118607 0.0623999908566475
0.0900000035762787 0.0611999854445457
0.100000001490116 0.0599999874830246
0.109999999403954 0.0587999895215034
0.119999997317791 0.0575999841094017
0.129999995231628 0.0563999861478806
0.140000000596046 0.0551999881863594
0.150000005960464 0.0539999902248383
0.159999996423721 0.0527999922633171
0.170000001788139 0.0515999868512154
0.180000007152557 0.0503999888896942
0.189999997615814 0.0491999909281731
0.200000002980232 0.0479999855160713
0.209999993443489 0.0467999875545502
0.219999998807907 0.0455999821424484
0.230000004172325 0.0443999841809273
0.239999994635582 0.0431999862194061
0.25 0.041999988257885
0.259999990463257 0.0407999902963638
0.270000010728836 0.0395999923348427
0.280000001192093 0.0383999869227409
0.28999999165535 0.0371999889612198
0.300000011920929 0.035999983549118
0.310000002384186 0.0347999855875969
0.319999992847443 0.0335999876260757
0.330000013113022 0.032399982213974
0.340000003576279 0.0311999842524529
0.349999994039536 0.0299999862909317
0.360000014305115 0.0287999883294106
0.370000004768372 0.0275999829173088
0.379999995231628 0.0263999849557877
0.389999985694885 0.0251999869942665
0.400000005960464 0.0239999890327454
0.409999996423721 0.0227999910712242
0.419999986886978 0.0215999931097031
0.430000007152557 0.0203999951481819
0.439999997615814 0.0191999897360802
0.449999988079071 0.017999991774559
0.46000000834465 0.0167999863624573
0.469999998807907 0.0155999809503555
0.479999989271164 0.014399990439415
0.490000009536743 0.0131999850273132
0.5 0.0119999796152115
0.509999990463257 0.0107999891042709
0.519999980926514 0.00959998369216919
0.529999971389771 0.00839999318122864
0.540000021457672 0.0071999728679657
0.550000011920929 0.00599998235702515
0.560000002384186 0.0047999769449234
0.569999992847443 0.00359998643398285
0.579999983310699 0.0023999810218811
0.589999973773956 0.00119997560977936
0.600000023841858 -1.49011611938477e-08
0.610000014305115 -0.00120002031326294
0.620000004768372 -0.00240001082420349
0.629999995231628 -0.00360001623630524
0.639999985694885 -0.00480000674724579
0.649999976158142 -0.00600001215934753
0.660000026226044 -0.00720001757144928
0.670000016689301 -0.00840000808238983
0.680000007152557 -0.00960001349449158
0.689999997615814 -0.0108000040054321
};
\addlegendentry{Linearizations of $f$}
\addplot [line width=0.5pt, color0, forget plot]
table {%
-0.699999988079071 -0.0120000094175339
-0.689999997615814 -0.0108000040054321
-0.680000007152557 -0.00960001349449158
-0.670000016689301 -0.00840000808238983
-0.660000026226044 -0.00720001757144928
-0.649999976158142 -0.00600001215934753
-0.639999985694885 -0.00480000674724579
-0.629999995231628 -0.00360001623630524
-0.620000004768372 -0.00240001082420349
-0.610000014305115 -0.00120002031326294
-0.600000023841858 -1.49011611938477e-08
-0.589999973773956 0.00119997560977936
-0.579999983310699 0.0023999810218811
-0.569999992847443 0.00359998643398285
-0.560000002384186 0.0047999769449234
-0.550000011920929 0.00599998235702515
-0.540000021457672 0.0071999728679657
-0.529999971389771 0.00839999318122864
-0.519999980926514 0.00959998369216919
-0.509999990463257 0.0107999891042709
-0.5 0.0119999796152115
-0.490000009536743 0.0131999850273132
-0.479999989271164 0.014399990439415
-0.469999998807907 0.0155999809503555
-0.46000000834465 0.0167999863624573
-0.449999988079071 0.017999991774559
-0.439999997615814 0.0191999897360802
-0.430000007152557 0.0203999951481819
-0.419999986886978 0.0215999931097031
-0.409999996423721 0.0227999910712242
-0.400000005960464 0.0239999890327454
-0.389999985694885 0.0251999869942665
-0.379999995231628 0.0263999849557877
-0.370000004768372 0.0275999829173088
-0.360000014305115 0.0287999883294106
-0.349999994039536 0.0299999862909317
-0.340000003576279 0.0311999842524529
-0.330000013113022 0.032399982213974
-0.319999992847443 0.0335999876260757
-0.310000002384186 0.0347999855875969
-0.300000011920929 0.035999983549118
-0.28999999165535 0.0371999889612198
-0.280000001192093 0.0383999869227409
-0.270000010728836 0.0395999923348427
-0.259999990463257 0.0407999902963638
-0.25 0.041999988257885
-0.239999994635582 0.0431999862194061
-0.230000004172325 0.0443999841809273
-0.219999998807907 0.0455999821424484
-0.209999993443489 0.0467999875545502
-0.200000002980232 0.0479999855160713
-0.189999997615814 0.0491999909281731
-0.180000007152557 0.0503999888896942
-0.170000001788139 0.0515999868512154
-0.159999996423721 0.0527999922633171
-0.150000005960464 0.0539999902248383
-0.140000000596046 0.0551999881863594
-0.129999995231628 0.0563999861478806
-0.119999997317791 0.0575999841094017
-0.109999999403954 0.0587999895215034
-0.100000001490116 0.0599999874830246
-0.0900000035762787 0.0611999854445457
-0.0799999982118607 0.0623999908566475
-0.0700000002980232 0.0635999888181686
-0.0599999986588955 0.0647999867796898
-0.0500000007450581 0.0659999847412109
-0.0399999991059303 0.0671999827027321
-0.0299999993294477 0.0683999881148338
-0.0199999995529652 0.0695999935269356
-0.00999999977648258 0.0707999914884567
1.11022302462516e-16 0.0719999894499779
0.00999999977648258 0.073199987411499
0.0199999995529652 0.0743999853730202
0.0299999993294477 0.0755999833345413
0.0399999991059303 0.0767999887466431
0.0500000007450581 0.0779999867081642
0.0599999986588955 0.079199992120266
0.0700000002980232 0.0803999900817871
0.0799999982118607 0.0815999880433083
0.0900000035762787 0.08279999345541
0.100000001490116 0.0839999914169312
0.109999999403954 0.0851999893784523
0.119999997317791 0.0863999873399734
0.129999995231628 0.0875999927520752
0.140000000596046 0.0887999832630157
0.150000005960464 0.0899999886751175
0.159999996423721 0.0911999866366386
0.170000001788139 0.0923999920487404
0.180000007152557 0.0935999900102615
0.189999997615814 0.0947999879717827
0.200000002980232 0.0959999859333038
0.209999993443489 0.0971999913454056
0.219999998807907 0.0983999893069267
0.230000004172325 0.0995999872684479
0.239999994635582 0.10079999268055
0.25 0.101999990642071
0.259999990463257 0.103199988603592
0.270000010728836 0.104399994015694
0.280000001192093 0.105599984526634
0.28999999165535 0.106799989938736
0.300000011920929 0.107999995350838
0.310000002384186 0.109199985861778
0.319999992847443 0.11039999127388
0.330000013113022 0.111599996685982
0.340000003576279 0.112799987196922
0.349999994039536 0.113999992609024
0.360000014305115 0.115199990570545
0.370000004768372 0.116399988532066
0.379999995231628 0.117599993944168
0.389999985694885 0.118799984455109
0.400000005960464 0.11999998986721
0.409999996423721 0.121199987828732
0.419999986886978 0.122399985790253
0.430000007152557 0.123599991202354
0.439999997615814 0.124799989163876
0.449999988079071 0.125999987125397
0.46000000834465 0.127199992537498
0.469999998807907 0.1283999979496
0.479999989271164 0.129599988460541
0.490000009536743 0.130799993872643
0.5 0.131999984383583
0.509999990463257 0.133199989795685
0.519999980926514 0.134399995207787
0.529999971389771 0.135599985718727
0.540000021457672 0.136799991130829
0.550000011920929 0.137999996542931
0.560000002384186 0.139199987053871
0.569999992847443 0.140399992465973
0.579999983310699 0.141599982976913
0.589999973773956 0.142799988389015
0.600000023841858 0.143999993801117
0.610000014305115 0.145199999213219
0.620000004768372 0.146399989724159
0.629999995231628 0.147599995136261
0.639999985694885 0.148799985647202
0.649999976158142 0.149999991059303
0.660000026226044 0.151199996471405
0.670000016689301 0.152399986982346
0.680000007152557 0.153599992394447
0.689999997615814 0.154799997806549
};
\end{axis}

\end{tikzpicture}

%% file: include/cifar_train_wrn_cifar100_arxiv.tex
% This file was created by tikzplotlib v0.9.3.
\begin{tikzpicture}

\definecolor{color0}{rgb}{0.12156862745098,0.466666666666667,0.705882352941177}
\definecolor{color1}{rgb}{1,0.498039215686275,0.0549019607843137}
\definecolor{color2}{rgb}{0.83921568627451,0.152941176470588,0.156862745098039}
\definecolor{color3}{rgb}{0.172549019607843,0.627450980392157,0.172549019607843}

\begin{axis}[
axis line style={white!80!black},
compat=newest,
height=0.25\textwidth,
legend cell align={left},
legend style={at={(1.1,0.5)}, anchor=west, draw=white},
legend style={fill opacity=0.8, draw opacity=1, text opacity=1, draw=white!80!black},
log basis y={10},
tick align=outside,
tick label style={font=\tiny},
title={WRN-CIFAR100},
width=0.3\textwidth,
x grid style={white!80!black},
x grid style={white},
xmajorgrids,
xmajorticks=false,
xmajorticks=true,
xmin=-9.75, xmax=204.75,
xtick style={color=white!15!black},
xtick={0,100,200,300},
y grid style={white!80!black},
ymajorgrids,
ymajorticks=false,
ymajorticks=true,
ymin=0.0002, ymax=2,
yminorticks=false,
ymode=log,
ytick style={color=white!15!black},
ytick={1e-05,0.0001,0.001,0.01,0.1,1,10,100},
yticklabels={\(\displaystyle {10^{-5}}\),\(\displaystyle {10^{-4}}\),\(\displaystyle {10^{-3}}\),\(\displaystyle {10^{-2}}\),\(\displaystyle {10^{-1}}\),\(\displaystyle {10^{0}}\),\(\displaystyle {10^{1}}\),\(\displaystyle {10^{2}}\)}
]
\addplot [thick, color0, dash pattern=on 1pt off 0.5pt]
table {%
0 3.6124636287011
1 3.29443170232985
2 3.00968565935771
3 2.74542877978007
4 2.5040140551673
5 2.28367987915463
6 2.08932591931661
7 1.91876851759593
8 1.76758738915761
9 1.63338755395254
10 1.51500104979621
11 1.40593050851186
12 1.30413041701847
13 1.20943863271078
14 1.11979935100979
15 1.03312425913917
16 0.951174721709358
17 0.872995316062503
18 0.796285252935621
19 0.724312594723171
20 0.656423454528385
21 0.591448939214283
22 0.531419799926546
23 0.476592765394847
24 0.424458830046124
25 0.377198124192556
26 0.334218919801712
27 0.296030107658175
28 0.262256644555728
29 0.232629199877315
30 0.207061987287204
31 0.184255918264919
32 0.163975837691095
33 0.145377378258175
34 0.130026193277041
35 0.115560379823976
36 0.104416662452883
37 0.0943287834722466
38 0.0872394649822182
39 0.0805238493372334
40 0.0741565763865577
41 0.0682403714542919
42 0.0628325766184595
43 0.0572126981070306
44 0.0523448495173454
45 0.0480139241388109
46 0.0444153451237413
47 0.0420225090859996
48 0.0400218819074498
49 0.0381054802244239
50 0.0373998924419615
51 0.035572361634175
52 0.0343758408262332
53 0.0323265325684018
54 0.0299857608197133
55 0.0277423445886374
56 0.0265428561758333
57 0.025208088656995
58 0.0240028415876627
59 0.0237525921037462
60 0.0228887566214138
61 0.0223740414808856
62 0.0208399067231682
63 0.0200528952879376
64 0.0188907706401746
65 0.0183485755786631
66 0.017430099966526
67 0.0171677679163218
68 0.0169793031789859
69 0.0165405803126097
70 0.0162454904793368
71 0.0157160713709394
72 0.0153321636314193
73 0.0148275473900967
74 0.0144288978078299
75 0.0138055251621538
76 0.0134637355874148
77 0.0126960579973873
78 0.0116104688979934
79 0.0110225102450202
80 0.0105515851048711
81 0.00991215494923294
82 0.00966801054937972
83 0.0102773273052441
84 0.0103779282193383
85 0.0102933892545932
86 0.0104308858710031
87 0.0107613647010757
88 0.0104854356655975
89 0.0104141518599457
90 0.010322670892328
91 0.010051980862394
92 0.00938096442014807
93 0.00900540078971121
94 0.0085979254609512
95 0.00851947827281223
96 0.00862647481884394
97 0.00854268549183177
98 0.00823793861258361
99 0.00819628158210052
100 0.0079178192190826
101 0.00733126580297947
102 0.00700432479490836
103 0.0067291846453978
104 0.00649539846031202
105 0.0064160165540874
106 0.0063878235940718
107 0.00661913546052244
108 0.00672584270195829
109 0.0066689632133974
110 0.00639310987843408
111 0.00629648027798368
112 0.00584069279577997
113 0.00563525001568927
114 0.00575939653240972
115 0.00572806338313553
116 0.00562966584574845
117 0.00562852402683761
118 0.00542440174728632
119 0.00486415072290434
120 0.0047315040121807
121 0.00462220320517818
122 0.00466783273080985
123 0.00475365294337273
124 0.00494295045853696
125 0.00493307261645173
126 0.00500052180329338
127 0.00489204930857859
128 0.00487634486230918
129 0.00484828888661331
130 0.00487898104800118
131 0.00491721217476659
132 0.00505128793312444
133 0.00498396157342527
134 0.00484361536143555
135 0.00466957435461382
136 0.00443619333408359
137 0.0042526259220019
138 0.00436686653351204
139 0.00434997429202414
140 0.00442865395538923
141 0.00446873314479987
142 0.00444907213508255
143 0.00423989806736095
144 0.0041131742855534
145 0.00413113091645141
146 0.0041318958037181
147 0.00393277893299858
148 0.00384689807080146
149 0.00390898767868264
150 0.00378269033798741
151 0.00362737147479318
152 0.00362882627819872
153 0.00350404551699654
154 0.0034135300047023
155 0.00322374132782738
156 0.00318412268185574
157 0.00309473445335403
158 0.00310492856826427
159 0.00315834313809251
160 0.00318406842939142
161 0.00324018415080177
162 0.00315114435921527
163 0.00317354757706738
164 0.00293358545303759
165 0.00291196707799617
166 0.00284375749309858
167 0.00295336129688554
168 0.00293612977299425
169 0.00305858936422401
170 0.00302055048079954
171 0.00283696994959066
172 0.00280822389654401
173 0.00258997240111646
174 0.00257222547535888
175 0.00256844610432577
176 0.00260394144951676
177 0.00256765889047956
178 0.00274032599047861
179 0.00268533061858267
180 0.0028065689871336
181 0.00280158020618061
182 0.00284832737917701
183 0.00280616795890447
184 0.0028319785933031
185 0.0026861201376054
186 0.00273402320834911
187 0.00263701118495729
188 0.00262888214623349
189 0.00256826890374637
190 0.00253057514444407
191 0.00239462930985209
192 0.00241068035033014
193 0.00253709985325734
194 0.00257720127903753
195 0.00265426343965862
};
\addlegendentry{Adagrad}
\addplot [thick, color1, dash pattern=on 1pt off 0.5pt]
table {%
0 3.15361481024848
1 2.78268211263021
2 2.45797237035963
3 2.18169207224104
4 1.95543286248101
5 1.76855685623593
6 1.61306074137794
7 1.48013381152259
8 1.36611834275987
9 1.26452662942251
10 1.17180110357073
11 1.08778792852402
12 1.00947120198356
13 0.937444751152462
14 0.870751825510661
15 0.807628693010542
16 0.748303778919644
17 0.695253437436422
18 0.642873524716695
19 0.591396763284471
20 0.546174137628344
21 0.504678718544642
22 0.464124091333813
23 0.428778128647274
24 0.397243161641227
25 0.368142553265889
26 0.340106041762034
27 0.315578431952794
28 0.29268796219084
29 0.273796135641734
30 0.253533182346556
31 0.239304925817913
32 0.225021491797765
33 0.211300814613766
34 0.197991885655721
35 0.189768344269594
36 0.180342975160546
37 0.171159037679831
38 0.164036821828153
39 0.158653238662879
40 0.151576280678113
41 0.144024503084289
42 0.137766307064162
43 0.132975735431247
44 0.128553360373709
45 0.125287012690968
46 0.121553180524773
47 0.119840404642688
48 0.116437858756516
49 0.11014598377347
50 0.10603923625244
51 0.104224796646039
52 0.101393334242635
53 0.0968244378574689
54 0.0966780545957883
55 0.0952368501008882
56 0.092673369254801
57 0.0901831917687257
58 0.0914175936146577
59 0.0890741695095433
60 0.0860776979148388
61 0.0830834714178244
62 0.0817785232366456
63 0.0792432701308197
64 0.0775816869355573
65 0.0759256727951103
66 0.0754498807277944
67 0.0736485250209438
68 0.0709293519674407
69 0.0711136990531286
70 0.0717377871223291
71 0.0701733971030182
72 0.0693406442248821
73 0.0674543732098739
74 0.0664701823813386
75 0.0653043223145273
76 0.0639286821171973
77 0.061680273701085
78 0.0613383060688443
79 0.0612346660300758
80 0.0596627115104596
81 0.0579737672565381
82 0.0576156529827913
83 0.0585392685959074
84 0.056907649835017
85 0.0586320830879609
86 0.0573215469579564
87 0.0562738919753498
88 0.054955252790451
89 0.0542376902721988
90 0.0517005245708095
91 0.0524362545054489
92 0.052085871469958
93 0.0509306849400699
94 0.050235896517088
95 0.049251984284636
96 0.0476899250564807
97 0.0470468563212289
98 0.0480992130686839
99 0.0477555462388529
100 0.0466526918424832
101 0.0479741922442615
102 0.0470653303763436
103 0.045448764067077
104 0.0440276121233735
105 0.0419764696904189
106 0.0409230517876148
107 0.0429649291026592
108 0.0423411133507887
109 0.0417904708674219
110 0.0433873038834996
111 0.0430555403611395
112 0.0419109177854161
113 0.0409890997290115
114 0.0393382287324303
115 0.0392738899757962
116 0.040100625220057
117 0.038691871947017
118 0.0396527305466599
119 0.0407897056823307
120 0.0389077718901965
121 0.0378605087050133
122 0.0381646519507302
123 0.0372529549409284
124 0.0373919772269991
125 0.0368214959457848
126 0.0367421910764112
127 0.036270420749055
128 0.0346936619076133
129 0.0326774484202928
130 0.0345860404528512
131 0.0337955037093162
132 0.0348023431350125
133 0.0352397826506363
134 0.0360343589786357
135 0.034468829219838
136 0.0334701387166646
137 0.0315688780606786
138 0.0326441517320938
139 0.0338037194497552
140 0.0322813357091778
141 0.0333966206442813
142 0.033035873339557
143 0.0321876978965766
144 0.0306196809485223
145 0.0310820977892147
146 0.0297525895509124
147 0.030544775934774
148 0.0294035852928873
149 0.0301065998553484
150 0.0301061401776721
151 0.0303728880624225
152 0.0287777224638065
153 0.0306091117457549
154 0.0307165116407474
155 0.0307372762239642
156 0.0299342830373844
157 0.0293890864587492
158 0.0285230397095945
159 0.0275819341157211
160 0.0280447051885724
161 0.0276315555150807
162 0.0277212363314132
163 0.0266095650133987
164 0.0269738984606829
165 0.0265813271044526
166 0.0268031647538476
167 0.0272432218682104
168 0.0278101772008671
169 0.0259903729587131
170 0.0245896772813135
171 0.0238848677522606
172 0.0254566416014565
173 0.0246779769111342
174 0.0254473606240087
175 0.0259933493888213
176 0.026024230144504
177 0.023868228832831
178 0.0244917685195307
179 0.0242297941087517
180 0.0247324667999148
181 0.0248387405804462
182 0.0253381791342961
183 0.0257835998949409
184 0.0257318532858292
185 0.0242404815631443
186 0.0245156181287766
187 0.0249274342412419
188 0.0236806341540814
189 0.0226810562084119
190 0.0231685769873858
191 0.0241060865665972
192 0.0219210778996431
193 0.0214677512137178
194 0.0226266798384239
195 0.0230209087790963
};
\addlegendentry{Adam}
\addplot [thick, color2, dash pattern=on 1pt off 0.5pt]
table {%
0 2.61893474822998
1 2.17214519522349
2 1.85267416748895
3 1.62533805675083
4 1.45504543567657
5 1.31971340153164
6 1.20568990801493
7 1.1072103993649
8 1.01978379154205
9 0.941159542590248
10 0.868002253089481
11 0.800702216493818
12 0.739510131450229
13 0.681709135773977
14 0.628995237530602
15 0.580815946290758
16 0.534358198725383
17 0.492270904690424
18 0.454545395735635
19 0.418834852605396
20 0.384293883040746
21 0.352289514527851
22 0.320661735976537
23 0.296373214367761
24 0.269349912788603
25 0.246736196565628
26 0.228358983976046
27 0.212396529153718
28 0.196233170312775
29 0.182789002916018
30 0.173635680400531
31 0.162052122797436
32 0.149662334332996
33 0.135198669807116
34 0.126144189987183
35 0.115052657069365
36 0.109738590906991
37 0.104882920119497
38 0.101930310635302
39 0.0988162093554603
40 0.0941524002077845
41 0.0875150048526128
42 0.0827532465333409
43 0.0783576181595855
44 0.0725331721950902
45 0.0682413129950895
46 0.0654347477850649
47 0.0635337802726693
48 0.0624894924614165
49 0.0621432433623738
50 0.0596499977135659
51 0.0578569370977084
52 0.057290890791946
53 0.0532829416881667
54 0.0501909842729569
55 0.0486164873112573
56 0.0453376049041748
57 0.0406195548876127
58 0.0396256225655476
59 0.0375667474484444
60 0.0366276611513562
61 0.0356537318139606
62 0.0351120045981805
63 0.0328602693012026
64 0.0315903920243846
65 0.029279050180912
66 0.0292292292598221
67 0.0294194383838442
68 0.0295607206261158
69 0.029299807572431
70 0.0300750190554063
71 0.0284865997477373
72 0.0275273799933328
73 0.0264253679845731
74 0.0251117796617747
75 0.0232224302455452
76 0.0225164538311296
77 0.020003206118279
78 0.0186042691961924
79 0.0186656031020483
80 0.0190628041986625
81 0.0195379132698642
82 0.0214580340242427
83 0.021879243062507
84 0.0216225012337541
85 0.0214371000378041
86 0.0209195860582631
87 0.019305283897784
88 0.0190879594544901
89 0.0197470123039352
90 0.0200835575069322
91 0.0195789058820738
92 0.0194973029716644
93 0.0189301255169345
94 0.0174908785546323
95 0.0166440343350006
96 0.017210420282301
97 0.0169281626936959
98 0.0174749225132747
99 0.0182815226354285
100 0.0177194523821357
101 0.0164045325509128
102 0.015570954873388
103 0.0142408405446675
104 0.01247448843754
105 0.0122037372047702
106 0.0116479017428971
107 0.0114526693097502
108 0.0113959884482208
109 0.0114669261724667
110 0.0114013218733586
111 0.0121406913844248
112 0.0131646324959894
113 0.0130922430970271
114 0.0140232495285819
115 0.0142292449825671
116 0.0133251838728165
117 0.0119130706364289
118 0.0117922971343084
119 0.00988575131302906
120 0.00850065572521753
121 0.00799847248919722
122 0.00800207913724498
123 0.0082476769087712
124 0.00918343737290551
125 0.0097041591389767
126 0.0105195337466192
127 0.0110763687219181
128 0.0105300395966776
129 0.0097029400633441
130 0.00955363853017489
131 0.00912295669529173
132 0.00843589219383895
133 0.00798732886812339
134 0.00865082372872779
135 0.00810022540187256
136 0.00792889430568035
137 0.00803467528851082
138 0.00872661846141848
139 0.00849982192986541
140 0.00876092565753394
141 0.00847620026457641
142 0.00986767149198386
143 0.0100555492655933
144 0.0103787157779766
145 0.0105338088989589
146 0.0110048471538888
147 0.0104820772084263
148 0.0104449743527836
149 0.0103676500575741
150 0.0102583059651363
151 0.0104346683505219
152 0.0101784900837847
153 0.0106233487403765
154 0.0104585429254547
155 0.0112236466809528
156 0.0112673909332107
157 0.0111137475924173
158 0.0106912400779935
159 0.0103693564991549
160 0.0090801269619974
161 0.00891891415933561
162 0.00804642823980914
163 0.00683128948603239
164 0.00618247610471315
165 0.0058571235207965
166 0.00482782162761523
167 0.00454465155483741
168 0.00412751454009571
169 0.00416321185639128
170 0.00463146035930142
171 0.00551914460924557
172 0.00572825174542765
173 0.00608091037338806
174 0.00646166430210901
175 0.00623871592271659
176 0.00543113251641186
177 0.00513028394125402
178 0.00516604674933685
179 0.00494136175804668
180 0.0046389284597834
181 0.00446672384451001
182 0.00423623092205122
183 0.00398892567081131
184 0.0034406943382112
185 0.00334529360240907
186 0.00330802100504024
187 0.00374091827186032
188 0.0044781620937172
189 0.00496923685502261
190 0.00543934574979564
191 0.00588537253486303
192 0.00575769286259698
193 0.00535336564807118
194 0.00533766042500838
195 0.00509148345527136
};
\addlegendentry{SGD}
\addplot [thick, color3]
table {%
0 3.0819587379286
1 2.60348197238498
2 2.20091505810632
3 1.87351704190572
4 1.63106795177036
5 1.45050985809326
6 1.30739064341651
7 1.18574884873284
8 1.08181233420902
9 0.988881746145884
10 0.903713038522932
11 0.827174712153541
12 0.759657202354007
13 0.695104882651435
14 0.635253953355153
15 0.579293319285711
16 0.527751371583939
17 0.479735191142824
18 0.439322663827472
19 0.402244529397752
20 0.369344402941598
21 0.338873676422967
22 0.308833791247474
23 0.277156045012474
24 0.247410251018206
25 0.217254035263591
26 0.185260453095966
27 0.156266060342789
28 0.129056248692009
29 0.102959631753365
30 0.0808075837465127
31 0.0650933623060915
32 0.0514397616567877
33 0.0420457448273235
34 0.0366087806575828
35 0.0323699631795618
36 0.0282612683033943
37 0.0253729450862938
38 0.0219547547700008
39 0.0191447931201591
40 0.01708642120964
41 0.014803724110855
42 0.0135886276414825
43 0.0125002426147295
44 0.0111029310121636
45 0.0102408284020258
46 0.0094851781104443
47 0.00857323960846704
48 0.00817467080893823
49 0.00757797441094597
50 0.00683279244550918
51 0.00665528986198517
52 0.00597797703016446
53 0.00587613497406141
54 0.00565590967370197
55 0.00540318205920255
56 0.00496885556539624
57 0.00463181207656343
58 0.00387441925682438
59 0.00356150306497907
60 0.00323869827914983
61 0.00295858953262783
62 0.00289388783346654
63 0.00282660529243242
64 0.00281267547606801
65 0.00289643834642445
66 0.00300766029570045
67 0.00303328141742386
68 0.00289156988989355
69 0.00273520134182026
70 0.00270509634731131
71 0.00248064779360779
72 0.00238070829033438
73 0.00232446585191724
74 0.00236058601658171
75 0.00210328586565331
76 0.0020433787726047
77 0.00185609511694302
78 0.00175976736492115
79 0.00166429582171711
80 0.00163500747680923
81 0.00154070161607634
82 0.00155646269056429
83 0.00155009907829575
84 0.00145548547109796
85 0.00146655385336063
86 0.00153826936510599
87 0.0014637670262665
88 0.00142916910171192
89 0.00149051801893388
90 0.00144866511450764
91 0.00138064570956805
92 0.00138614484575066
93 0.00133218269560072
94 0.00125990136250802
95 0.00120972761364395
96 0.00118701089221208
97 0.00119750280590146
98 0.00112204754509446
99 0.00108992630639838
100 0.00106571325937131
101 0.00100124974554974
102 0.000905078285922193
103 0.000966546711249587
104 0.000950391557343925
105 0.00105074542561598
106 0.00106908447688885
107 0.00108429254310024
108 0.00105129982832444
109 0.0010259430779681
110 0.000906614146795149
111 0.000969062317779365
112 0.000957574838134848
113 0.000926882703709674
114 0.000923822536476525
115 0.000940146995128856
116 0.000841530963054133
117 0.000808911560914065
118 0.000826718292236845
119 0.00086314430237075
120 0.000828299022249152
121 0.000822957708569139
122 0.000834695044620376
123 0.000829108462860135
124 0.000767474191448257
125 0.000773718931673632
126 0.0007953557207805
127 0.000816636132394812
128 0.000790351477987878
129 0.000816580967626069
130 0.000809080679139459
131 0.00076697158388677
132 0.000742943267831449
133 0.000733902066558884
134 0.000712711281263812
135 0.000697410700070516
136 0.000652515744117781
137 0.000610105741286631
138 0.000613855567505525
139 0.000582278022762201
140 0.000567601157293571
141 0.000593523332806459
142 0.00060174229727754
143 0.000567672646839637
144 0.000563162555694321
145 0.000559281173282135
146 0.000537288157141496
147 0.000537444559729257
148 0.000561620216365943
149 0.000579061414925705
150 0.000581474545791409
151 0.000605798416137048
152 0.000583506842718537
153 0.000580829628837139
154 0.000586295136344463
155 0.000609067878722174
156 0.00062471086290002
157 0.000626299540199332
158 0.000610428610905429
159 0.000607485285863739
160 0.000585998865762037
161 0.000581659656098426
162 0.000586879013907892
163 0.000595970653956578
164 0.000571838372547143
165 0.000595051210190188
166 0.00054513072543231
167 0.000539477969271279
168 0.000527724342358577
169 0.00050917788612763
170 0.000461888550667014
171 0.000449620467310693
172 0.000431885772301919
173 0.000427661141287535
174 0.000433279855514152
175 0.000448892824380826
176 0.000490789856377329
177 0.000521595380540691
178 0.000532189290766255
179 0.000545256796986861
180 0.000539872074164968
181 0.000514621151855681
182 0.000496118314593074
183 0.000490475203256162
184 0.000474760784566776
185 0.0004731674279104
186 0.000481926345824573
187 0.000463035257127833
188 0.000452007844721496
189 0.000461432724007479
190 0.000470631451085614
191 0.000472080225953266
192 0.000486455654046941
193 0.000502251548238128
194 0.000495750558112082
195 0.000473233527608165
};
\addlegendentry{ALI-G}
\end{axis}

\end{tikzpicture}

%% file: include/cifar_train_dn_cifar100_arxiv.tex
% This file was created by tikzplotlib v0.9.3.
\begin{tikzpicture}

\definecolor{color0}{rgb}{0.12156862745098,0.466666666666667,0.705882352941177}
\definecolor{color1}{rgb}{1,0.498039215686275,0.0549019607843137}
\definecolor{color2}{rgb}{0.83921568627451,0.152941176470588,0.156862745098039}
\definecolor{color3}{rgb}{0.172549019607843,0.627450980392157,0.172549019607843}

\begin{axis}[
axis line style={white!80!black},
compat=newest,
height=0.25\textwidth,
log basis y={10},
tick align=outside,
tick label style={font=\tiny},
title={DN-CIFAR100},
width=0.3\textwidth,
x grid style={white!80!black},
x grid style={white},
xmajorgrids,
xmajorticks=false,
xmajorticks=true,
xmin=-14.75, xmax=309.75,
xtick style={color=white!15!black},
xtick={0,100,200,300},
y grid style={white!80!black},
ymajorgrids,
ymajorticks=false,
ymajorticks=true,
ymin=0.0002, ymax=2,
yminorticks=false,
ymode=log,
ytick style={color=white!15!black},
ytick={1e-05,0.0001,0.001,0.01,0.1,1,10,100},
yticklabels={\(\displaystyle {10^{-5}}\),\(\displaystyle {10^{-4}}\),\(\displaystyle {10^{-3}}\),\(\displaystyle {10^{-2}}\),\(\displaystyle {10^{-1}}\),\(\displaystyle {10^{0}}\),\(\displaystyle {10^{1}}\),\(\displaystyle {10^{2}}\)}
]
\addplot [thick, color0, dash pattern=on 1pt off 0.5pt]
table {%
0 3.20495142756992
1 2.82975459161546
2 2.52224390293545
3 2.25991858874003
4 2.04028279593574
5 1.86052348355611
6 1.71607525004069
7 1.5939639307234
8 1.49058747432285
9 1.4003330343967
10 1.31892278758579
11 1.24550649275886
12 1.1781387607638
13 1.11478437800725
14 1.05394331014633
15 0.998351684229109
16 0.946396010909611
17 0.895795058784485
18 0.847116310162015
19 0.803084743821886
20 0.759160784284804
21 0.715992741010454
22 0.677679919204712
23 0.642141188401117
24 0.605206383268569
25 0.573140057758756
26 0.540158045164745
27 0.509424481074016
28 0.479815851050483
29 0.454197748332554
30 0.427978533185323
31 0.406469528244866
32 0.385629864264594
33 0.363952791129218
34 0.343652614557478
35 0.325595246109433
36 0.308204986532
37 0.290841400237613
38 0.275030143076579
39 0.258517418043349
40 0.243214791557524
41 0.229517936601639
42 0.216045378956265
43 0.205297446300718
44 0.197003434342278
45 0.190110500069724
46 0.181921762096617
47 0.173392412457996
48 0.165434715207418
49 0.158139074938033
50 0.147393525887595
51 0.139789592042499
52 0.13552468887753
53 0.130151159381867
54 0.12398283888446
55 0.120998522204823
56 0.115118840087785
57 0.110275229949421
58 0.106501061839263
59 0.102471748034689
60 0.0982403389231365
61 0.095240593284501
62 0.0918935813061397
63 0.0902218039687475
64 0.0893860755605168
65 0.0869910714215703
66 0.0841573670215077
67 0.0805569285440446
68 0.0772555873950324
69 0.0736697392047777
70 0.0734114310693742
71 0.0735416698296866
72 0.0733547692394258
73 0.0730251924909487
74 0.073646764948633
75 0.0694979032799933
76 0.0669859076844322
77 0.0660636018292109
78 0.0638201644391484
79 0.0609497480416298
80 0.060528689271609
81 0.0582437827097045
82 0.0555906596594387
83 0.0542670487236976
84 0.0518591854450438
85 0.0503040931773186
86 0.0501794214710925
87 0.0481927852512731
88 0.0456181511420674
89 0.0448353679446379
90 0.0437523578349749
91 0.0427615608117316
92 0.0431720483039485
93 0.0423609505571922
94 0.0413470600261291
95 0.040544002952377
96 0.039206831103762
97 0.0404142086748945
98 0.040740822800133
99 0.0406603222354253
100 0.0419930755595366
101 0.0419116500929992
102 0.0386846504898866
103 0.0397958739395936
104 0.0394458801705308
105 0.0377125453732411
106 0.0368880015657346
107 0.036449256371061
108 0.0340937911712461
109 0.0332472013751004
110 0.0319612386915419
111 0.0315743478939931
112 0.031098461714387
113 0.0309299293750524
114 0.030629341595504
115 0.0302779484432935
116 0.0299785694458749
117 0.0292485564200083
118 0.0280542954134941
119 0.026763999420272
120 0.0274404818659359
121 0.0268109847658873
122 0.0267397559697761
123 0.0266707710232338
124 0.0290723473590613
125 0.0299104822070731
126 0.0307603073142635
127 0.0302005293189817
128 0.029166939433482
129 0.0278955689726273
130 0.0264090953184499
131 0.0253874113970333
132 0.0265918748030398
133 0.0274775063238542
134 0.0269685874754191
135 0.0262810219819679
136 0.0277511426795192
137 0.0274483077551922
138 0.0264347589764993
139 0.0248976012804111
140 0.0245187619478173
141 0.0222185172696908
142 0.0215759795533949
143 0.0216831301546097
144 0.0221713737236129
145 0.0228580491938194
146 0.0231564123573569
147 0.0240618595653111
148 0.0255108587400782
149 0.0291254830393527
150 0.0295510420924426
151 0.0303941003108688
152 0.0294480462255744
153 0.0283292123033603
154 0.0239613571176596
155 0.0242711765459843
156 0.0247768663908707
157 0.0243656446324786
158 0.0242569486644202
159 0.0248997574547265
160 0.0239712108788888
161 0.0234388524107801
162 0.0229266000722515
163 0.0221183381425673
164 0.0228115443913473
165 0.0229769578679734
166 0.0214916232283248
167 0.0207003010353777
168 0.020159034543501
169 0.0189224537282189
170 0.017389949678017
171 0.0174466567817993
172 0.0174367129925556
173 0.0183424121286141
174 0.018121312081847
175 0.0192148550476962
176 0.0199579980278678
177 0.0206015365080701
178 0.0197932610089581
179 0.0195674367750685
180 0.0195670686289006
181 0.0180442272922728
182 0.0177543831209342
183 0.0172979363324245
184 0.016461345615387
185 0.0150907106168734
186 0.015141248079439
187 0.014493936821984
188 0.0149277523052693
189 0.0161170729987489
190 0.0163571860950854
191 0.0168602420943976
192 0.0183757104757428
193 0.0186314866539174
194 0.0187701590454579
195 0.0183886857426167
196 0.017898001658486
197 0.0166974526429508
198 0.0183371310345663
199 0.0186327557867104
200 0.0184474388750229
201 0.0189956511205601
202 0.0189717343006366
203 0.0166962053824962
204 0.0149614701901708
205 0.0153950073637234
206 0.014746074640486
207 0.016218183776339
208 0.0172442108272513
209 0.01788462650713
210 0.0178631380001373
211 0.0177109868193667
212 0.0166532422138916
213 0.0160450825528635
214 0.0174834958036741
215 0.0171568648744789
216 0.0165442035847406
217 0.0161033428284361
218 0.0158656546089881
219 0.0150088717353841
220 0.014970645588504
221 0.0151092434280118
222 0.0168883932003379
223 0.0176006439688802
224 0.0176056269803644
225 0.0176638815901346
226 0.019014914968941
227 0.0168471018036538
228 0.016543999908997
229 0.016013377288017
230 0.0166085982604822
231 0.0160443537748191
232 0.0166260210255782
233 0.0163638123858969
234 0.01569756441921
235 0.0158650275237362
236 0.0158055299740368
237 0.0150026574402386
238 0.0146591355655425
239 0.0150830725180275
240 0.0133125692545209
241 0.0123777872371839
242 0.0124547367335691
243 0.0130644669312736
244 0.0126764646700522
245 0.013130638117873
246 0.0148115123380059
247 0.0144649516474538
248 0.0131171997829278
249 0.0135132544031739
250 0.0138479686927133
251 0.013479385992686
252 0.0136428473245767
253 0.014103461450802
254 0.0132682666904894
255 0.0127883711126944
256 0.0117832038391465
257 0.0124042268897593
258 0.0121450643122362
259 0.0127165775083834
260 0.0129954200072421
261 0.0137204495204489
262 0.0137757622208364
263 0.0146856123673089
264 0.014409064635353
265 0.0139807989663383
266 0.012665945112192
267 0.0114334272279342
268 0.0109211245218913
269 0.0115713171713551
270 0.0112713945474228
271 0.0126773045860066
272 0.0133429293060138
273 0.0125491303817763
274 0.0119306291320258
275 0.011887952762776
276 0.0113703461724851
277 0.010747525077247
278 0.0117915383267569
279 0.0116645364268125
280 0.0136107387809456
281 0.0139658991270761
282 0.0141202080121637
283 0.0130397423859935
284 0.0131076703998943
285 0.0112420539945695
286 0.00992202336228558
287 0.00991287211077085
288 0.0105738279881743
289 0.0107438056736191
290 0.0106901274891198
291 0.0117305987216533
292 0.011656643535263
293 0.0120615752478772
294 0.0115346043458581
295 0.0127696459131771
};
\addplot [thick, color1, dash pattern=on 1pt off 0.5pt]
table {%
0 2.58747019265069
1 2.19175618540446
2 1.92372253569709
3 1.72575039982266
4 1.57199522418128
5 1.44629406398137
6 1.33981776084688
7 1.24720541897456
8 1.16464204537074
9 1.08990621548759
10 1.02139442031437
11 0.957309182171292
12 0.898426813358731
13 0.842538275265164
14 0.791021717084249
15 0.74183221850925
16 0.69771825914171
17 0.655991907857259
18 0.617410877329509
19 0.580768569043477
20 0.546661220622592
21 0.513359548780653
22 0.480021249296401
23 0.45325879629559
24 0.42605455701828
25 0.403367646257613
26 0.381305906030867
27 0.36189077627182
28 0.340440347834693
29 0.327387307855818
30 0.308015750965542
31 0.294141890846888
32 0.279959553425047
33 0.267411107476552
34 0.252963778299756
35 0.244933574132919
36 0.233488770120409
37 0.224807937925127
38 0.218564057572683
39 0.209803205198712
40 0.205220460996628
41 0.197957130584717
42 0.192530152274238
43 0.186006226387024
44 0.180840641823875
45 0.174415026578903
46 0.16931077895297
47 0.162880199532244
48 0.157211071346336
49 0.155209113996559
50 0.151047345376809
51 0.149893401842117
52 0.146708359533416
53 0.144482695361243
54 0.13891348665052
55 0.138072409622934
56 0.13188654843092
57 0.133088518442313
58 0.134257200115787
59 0.130055024287436
60 0.125063365439574
61 0.125057481769191
62 0.12236140105274
63 0.11796424530824
64 0.116071662504938
65 0.115714727807575
66 0.115138735416995
67 0.113292864307563
68 0.109841974344783
69 0.110099013806979
70 0.105105713424418
71 0.10182273399512
72 0.0998713013977475
73 0.100541377554999
74 0.0987149921194712
75 0.0999959744037522
76 0.0984592489588261
77 0.098533949536615
78 0.0971080033462578
79 0.0966907613407242
80 0.0928280106067658
81 0.091744863481919
82 0.0883370601520274
83 0.0861722115936545
84 0.0842360921066338
85 0.0838525054589908
86 0.0843040389223894
87 0.0855631669933266
88 0.0830536247571309
89 0.0825781762557559
90 0.0822685453067886
91 0.0801448230620225
92 0.0808764913871554
93 0.0780751695824994
94 0.0791235165369511
95 0.079214051822424
96 0.0770137750976616
97 0.0741725391860804
98 0.076524702770975
99 0.0747317916619779
100 0.0735533988281092
101 0.075685200554265
102 0.0739826783933243
103 0.0725677933173048
104 0.0727431416060859
105 0.0690983038741682
106 0.0670987987790837
107 0.0661927108416624
108 0.0669579634685318
109 0.0646701660035717
110 0.0676552712900109
111 0.069181960792078
112 0.0667812706781428
113 0.067939272267388
114 0.067438599382076
115 0.0633876371213132
116 0.0634235846137007
117 0.0635893883295192
118 0.0603840511904822
119 0.0613917072984908
120 0.063421082273788
121 0.0617253783433968
122 0.0611233880819215
123 0.060250386628906
124 0.0580135071605444
125 0.0582460132789612
126 0.0556440931810273
127 0.0578665530604787
128 0.0597070106311639
129 0.0580160119485855
130 0.0563459358654751
131 0.0559724694537785
132 0.0538195563075278
133 0.0504908213302162
134 0.0521466044669681
135 0.0539735118892789
136 0.0542570647884409
137 0.0538968678532375
138 0.0524147198359503
139 0.0551808922399085
140 0.051819904817972
141 0.05333051946544
142 0.0542073050424458
143 0.0555634720297986
144 0.0505454901959168
145 0.0518856510077583
146 0.0499413749953112
147 0.0501251003245513
148 0.0501542087399298
149 0.0509277510162195
150 0.0495143163660169
151 0.0489362745013502
152 0.0464003662686216
153 0.0478989775494735
154 0.0473927124481069
155 0.0466222173617284
156 0.045825041228748
157 0.0459231781320274
158 0.0441149223713577
159 0.0437030625751118
160 0.0478696740517849
161 0.0484117735316688
162 0.0485271772120066
163 0.0461353369243609
164 0.0474479932344623
165 0.0450015785671936
166 0.0459311150117715
167 0.0446282652008865
168 0.0452534079747068
169 0.046194883109861
170 0.0459276436229547
171 0.0431124298768573
172 0.0437020324502058
173 0.0443406234570344
174 0.0437120065050324
175 0.0428784425739448
176 0.0420719767176443
177 0.0426520459295644
178 0.0437670847235124
179 0.0406873928307163
180 0.0415261456062397
181 0.0429144148976605
182 0.04333566933165
183 0.040734058689144
184 0.0413891584934129
185 0.0392069484470619
186 0.0387019493415786
187 0.0371775505083634
188 0.0383115968475077
189 0.0371096294916007
190 0.0378452007579803
191 0.0382115899612672
192 0.0393802208201587
193 0.0395303290561835
194 0.0412108254976736
195 0.0407952696866129
196 0.0405442746799191
197 0.0387337495667736
198 0.038192021064063
199 0.0381259295944042
200 0.0366151469686628
201 0.0385439855792291
202 0.037509546494037
203 0.0370628752287064
204 0.0357528058575425
205 0.0358560593574577
206 0.036294325299114
207 0.036002569279737
208 0.034976712889274
209 0.0367245395946503
210 0.0359288764671319
211 0.0326053604934448
212 0.0340602712974615
213 0.0339395161062479
214 0.034364805694388
215 0.0351744768109255
216 0.0350486942908002
217 0.0343456232793464
218 0.0352773785855373
219 0.0334322725139393
220 0.0349227031004264
221 0.0349061662084692
222 0.0348045536134972
223 0.0346925235115075
224 0.0356704496561737
225 0.0343280033066124
226 0.0333310294184172
227 0.0330973494298839
228 0.0326395330324438
229 0.0334905180074275
230 0.0325851100017296
231 0.0329981056037379
232 0.0330142402933869
233 0.0331203594460587
234 0.0306414671537611
235 0.0319639473637276
236 0.0324650968180928
237 0.0310431376166311
238 0.0313638365349339
239 0.0307155488049156
240 0.0320743853504956
241 0.0314237273923556
242 0.0320186533277813
243 0.0326618570195056
244 0.0323267273614722
245 0.029196316670734
246 0.0288188680154334
247 0.0283709993645383
248 0.0274032080357108
249 0.0273021106832557
250 0.0280488446002868
251 0.0292842269098593
252 0.0298038716985617
253 0.0293257497961488
254 0.0293273993441297
255 0.0303118111521006
256 0.0289593707131263
257 0.0292868844892002
258 0.0308837083562464
259 0.0317928869151655
260 0.0287297592301418
261 0.0302714962570204
262 0.0292377204696006
263 0.0290331464797838
264 0.0273911053944048
265 0.0298437808659921
266 0.0297118328914957
267 0.0300197162572957
268 0.0269859094249043
269 0.0283622303010689
270 0.0274537939221246
271 0.0253669028557423
272 0.0279808011144317
273 0.0305230665730023
274 0.0284535752531804
275 0.0277986333210767
276 0.0295435158582198
277 0.0278249976123539
278 0.0266099463558032
279 0.0290123966839744
280 0.0296045206138492
281 0.0279293972916239
282 0.0269243410107493
283 0.0265965870683806
284 0.0249414508826865
285 0.025262112993449
286 0.026793443328821
287 0.0261239749888165
288 0.0254242749205149
289 0.0274304909149309
290 0.0255419196773404
291 0.0243039120399952
292 0.0259192410343886
293 0.0248654918662873
294 0.0234983336599834
295 0.0253363765962257
};
\addplot [thick, color2, dash pattern=on 1pt off 0.5pt]
table {%
0 3.27014140881856
1 2.95815982888963
2 2.7046275241767
3 2.48752659600152
4 2.30769778787401
5 2.15085589425829
6 2.01347754457262
7 1.89238075492859
8 1.78433535568661
9 1.68539345937093
10 1.59879963872698
11 1.52101427327898
12 1.44819536678314
13 1.38222910961575
14 1.32115283191681
15 1.26425136999342
16 1.20982648732503
17 1.16100290753682
18 1.11369117674086
19 1.06991365565406
20 1.02613226474762
21 0.98275895670149
22 0.943200228843688
23 0.905139825452169
24 0.87022865975274
25 0.83615649866316
26 0.804989616940816
27 0.774328027816348
28 0.746018253493839
29 0.716425983562469
30 0.688962628442976
31 0.66260196097268
32 0.63608125450982
33 0.610811067428588
34 0.585695673789978
35 0.564105193898943
36 0.540720348387824
37 0.518441932491726
38 0.498620512089199
39 0.480605963083903
40 0.4605305255572
41 0.446020164256626
42 0.430469386075338
43 0.410767578540378
44 0.393667391569349
45 0.379059946081373
46 0.362097824791802
47 0.348286941562229
48 0.337759260779487
49 0.325043544114431
50 0.310901972450681
51 0.300374220256806
52 0.290421842051612
53 0.277945991522471
54 0.267339649327596
55 0.260691002748277
56 0.252782167362637
57 0.242214553657108
58 0.238296343862746
59 0.230409953123729
60 0.2224813641421
61 0.211559753456116
62 0.203483692830404
63 0.194307046517266
64 0.187964724466536
65 0.18018353178872
66 0.174286347779168
67 0.168994807798598
68 0.162361253185272
69 0.160336182526483
70 0.155237916689979
71 0.151846209193336
72 0.146958575231764
73 0.142677036531237
74 0.138244026241302
75 0.134494595209758
76 0.130821498817868
77 0.130997412696415
78 0.126070416971843
79 0.12072316702101
80 0.11807661576589
81 0.115591661332448
82 0.110838931325277
83 0.110433408976661
84 0.1102605459955
85 0.104008450423347
86 0.104759005858103
87 0.104299467124939
88 0.104695010821025
89 0.104662583433787
90 0.107678936511146
91 0.10665017601649
92 0.103527595274184
93 0.102685505570306
94 0.099970618885888
95 0.0966683053504097
96 0.0936966832500034
97 0.0921972660043505
98 0.0903130322392782
99 0.0861331204626296
100 0.0867785038524204
101 0.0842760755962796
102 0.0810833413336012
103 0.0750475262027317
104 0.0706268506262038
105 0.0687882516394721
106 0.0678757579188877
107 0.065200174530877
108 0.0647279494900174
109 0.0640672737163968
110 0.0620863560824924
111 0.0607720489586724
112 0.058764734210968
113 0.0569703733189901
114 0.055480239338345
115 0.0544379227595859
116 0.053579998404185
117 0.0576664411735535
118 0.0578876983812121
119 0.0596290084181892
120 0.0591228797933791
121 0.0583821002896627
122 0.0576702074220446
123 0.0586170448748271
124 0.0563860136286418
125 0.0538920996284485
126 0.0517769415749444
127 0.0572798319901361
128 0.0550236674669054
129 0.0544806661605836
130 0.0540839530118307
131 0.0519504032177395
132 0.0433888073772854
133 0.0418058615960015
134 0.0423248770798577
135 0.0412525992054409
136 0.0393279885673523
137 0.0357716841697693
138 0.0340114024670919
139 0.0329535656017728
140 0.0358654951900906
141 0.0378488823276096
142 0.0391175266435412
143 0.040727995177375
144 0.041069583683014
145 0.0394442507129246
146 0.0396703963512845
147 0.0386806879764134
148 0.0430426261732314
149 0.0446231507449681
150 0.0460099968020122
151 0.0472640487183465
152 0.048748489596049
153 0.0445612790701125
154 0.0418265358691745
155 0.0371281176461114
156 0.0332599820921157
157 0.0326643365097046
158 0.0293109322505527
159 0.0284919550683764
160 0.0289906764602661
161 0.0283998647393121
162 0.0263462337938945
163 0.0263369643168979
164 0.0260259454769559
165 0.027827236404419
166 0.0295972755580479
167 0.0310610105111865
168 0.0333972862646316
169 0.0344733779864842
170 0.0319360813776652
171 0.0292319449509515
172 0.0270111873520745
173 0.024328054239485
174 0.0216941511789958
175 0.0238446029811436
176 0.0250625566969978
177 0.0271958515273201
178 0.0291356087599861
179 0.0299113030582005
180 0.029235575764974
181 0.0306690987650554
182 0.0285028001700508
183 0.0261307067871094
184 0.0249134797117446
185 0.0309885149510702
186 0.029234559050666
187 0.029456076596578
188 0.0310558899137709
189 0.0336450502861871
190 0.0262446405453152
191 0.0276729579925537
192 0.0271645209121704
193 0.0258512106831869
194 0.0257635210800171
195 0.0268097908549839
196 0.0250108772595724
197 0.0252309739748637
198 0.0247417051845127
199 0.0250668075201247
200 0.0251028171157837
201 0.0245333587392172
202 0.0264007305357192
203 0.0273791999265883
204 0.0246804429647658
205 0.0229953479194641
206 0.0228303107473585
207 0.0217414536688063
208 0.0244104317749871
209 0.0294640302658081
210 0.0297402254358928
211 0.0308234218724569
212 0.0319459477403429
213 0.0314705878490872
214 0.027654535829756
215 0.028695556494395
216 0.0289612311447991
217 0.0267109489292569
218 0.0237640783712599
219 0.0243768679512872
220 0.0236373310852051
221 0.0251651902728611
222 0.0259074784999424
223 0.0262404068311056
224 0.0239850301975674
225 0.0240353177875943
226 0.0214030746735467
227 0.0246978354008993
228 0.0279496059629652
229 0.0283597615771824
230 0.02691430873447
231 0.0264825235070123
232 0.0248686349360148
233 0.0225441246329414
234 0.02474218846639
235 0.0262521938451131
236 0.0266264439858331
237 0.0261575600941976
238 0.0244051756710477
239 0.0270018002361722
240 0.0276268304655287
241 0.0278578859710694
242 0.0290816298378839
243 0.029702415216234
244 0.0242027431699965
245 0.0214723203659058
246 0.0211559680811564
247 0.0175945187038846
248 0.0190987338892619
249 0.0193085637664795
250 0.0205942180718316
251 0.0205085201772054
252 0.0222401981820001
253 0.0234584255303277
254 0.0250019447496203
255 0.0240512384711372
256 0.0228543426895142
257 0.0199882711071438
258 0.0151578304714627
259 0.0147857718467712
260 0.0167695992300246
261 0.019132860499488
262 0.0202522350078159
263 0.0245458997790019
264 0.0230972890769111
265 0.0229977068286472
266 0.0221781064520942
267 0.0252136942884657
268 0.0227919137869941
269 0.0220470404222276
270 0.0193262132475111
271 0.0186114952723185
272 0.014677147263421
273 0.0160205232344733
274 0.0170110351371765
275 0.0180165206379361
276 0.0172765441237556
277 0.0195464516088698
278 0.0184243659930759
279 0.017953513001336
280 0.0193633331086901
281 0.0195162248569065
282 0.016654811062283
283 0.0157037362798055
284 0.0161137557898627
285 0.0170208802456326
286 0.0165231926769681
287 0.0196695059161716
288 0.0189238013331095
289 0.0184179745949639
290 0.0154147295506795
291 0.0168075426355998
292 0.0152423609330919
293 0.0155292009438409
294 0.014625729637146
295 0.0138615214453803
};
\addplot [thick, color3]
table {%
0 2.56264761564467
1 2.12210809150696
2 1.82166233183967
3 1.61325939327664
4 1.4577927752516
5 1.33307255460315
6 1.22828712311003
7 1.13728997082816
8 1.05908511335373
9 0.987293829835256
10 0.922014704399109
11 0.863653346527947
12 0.812228580036163
13 0.763799715124766
14 0.71840461760415
15 0.676950602664947
16 0.637602588926951
17 0.601255397978888
18 0.567438466936747
19 0.536805181719462
20 0.506924127786424
21 0.479291179385715
22 0.450306289329529
23 0.422887215016683
24 0.398080323515998
25 0.378822292989095
26 0.355413716244168
27 0.339363224228753
28 0.323571964437697
29 0.308772135052151
30 0.290652982548608
31 0.279433796746996
32 0.267445843552484
33 0.253449955194262
34 0.240600680743324
35 0.229818641270532
36 0.218380716703203
37 0.208807618586222
38 0.198824021203783
39 0.188358503915999
40 0.179007349940406
41 0.170008493321737
42 0.158201046013302
43 0.149467754900191
44 0.142182119613224
45 0.134851859020657
46 0.128342855525547
47 0.122119604144626
48 0.115106354662577
49 0.106617380048964
50 0.0997550553766886
51 0.0922690463426378
52 0.0842577605925666
53 0.078777655184004
54 0.0757171174282497
55 0.0700424825074937
56 0.0652665632841322
57 0.0608326998477512
58 0.0570966112687853
59 0.0517481837145488
60 0.0472902161280314
61 0.0447348616409302
62 0.0421414480230544
63 0.0387066540124682
64 0.0360052498541938
65 0.0342972319094341
66 0.0306024730322097
67 0.0276854173151653
68 0.0259902180904813
69 0.0241495179896885
70 0.0220454087130229
71 0.0207372073682149
72 0.0199827868737115
73 0.0191946022160848
74 0.0182633204227024
75 0.0180674160342747
76 0.017924373541938
77 0.0173383505588108
78 0.0161721945359972
79 0.0152518556128608
80 0.0140668336529202
81 0.0126290440283882
82 0.0114575913323297
83 0.0107724483193292
84 0.0101209925291274
85 0.00960053078545468
86 0.00937462320115834
87 0.0087464027616713
88 0.00824930290222169
89 0.00777252164204917
90 0.00753260503556994
91 0.00705949464162192
92 0.00710169275071887
93 0.00697459770202638
94 0.00706399298773873
95 0.00672966212378609
96 0.00659958718193903
97 0.0063834013663398
98 0.00619690727657743
99 0.00592262456681994
100 0.00584269341786702
101 0.00579834057278104
102 0.00547233481513129
103 0.00534673459370931
104 0.00517212399376763
105 0.00486651154412164
106 0.00452562798817952
107 0.00446323907640245
108 0.00448665123833551
109 0.00431717849731446
110 0.00416752917819554
111 0.00411505884382461
112 0.00399855269961888
113 0.00371080687204998
114 0.00370673399607342
115 0.00373854411231148
116 0.00383974989149306
117 0.00378971564822727
118 0.00369051453908285
119 0.00362616488138836
120 0.00353591553158231
121 0.00336874570210776
122 0.0032682266743978
123 0.00328876945919461
124 0.00330579106648764
125 0.00317242407904731
126 0.00307692700703939
127 0.0031730756462945
128 0.00313280473921035
129 0.00315342671288385
130 0.00321550184037951
131 0.0031658900917901
132 0.00303623986562093
133 0.00301046554989285
134 0.00274711917029487
135 0.00264544559478761
136 0.00264312889522977
137 0.00257450578901504
138 0.00254823071373835
139 0.00251141871558297
140 0.00265825815836589
141 0.00267569776323107
142 0.00273703552246094
143 0.00272092050764296
144 0.00272995654635959
145 0.00253211034986708
146 0.00249565189785428
147 0.00243490907880995
148 0.00251645645565457
149 0.00259087679545086
150 0.0027641959169176
151 0.00273172073364259
152 0.00277905140770807
153 0.00260934676276313
154 0.00249521074930827
155 0.00231779350704617
156 0.00233864549848769
157 0.00222307144165039
158 0.00230284529791938
159 0.00230297003428141
160 0.00241947290208605
161 0.00236441516876221
162 0.00232974778069391
163 0.00224541963789198
164 0.00219074467976888
165 0.00204820562150743
166 0.00202883039686416
167 0.00204282133314346
168 0.0020526645618015
169 0.00208359678480361
170 0.00204398790571426
171 0.00197742963578966
172 0.00189523470984566
173 0.00184240725623237
174 0.00190959149254694
175 0.00194366098192004
176 0.00197429371727839
177 0.00195718702952068
178 0.00179855914221871
179 0.00168094514634875
180 0.0016315150027805
181 0.00168857369740806
182 0.00184038640764028
183 0.00194421499888106
184 0.00199900683508983
185 0.00204085505591503
186 0.00204627871195479
187 0.00202117957221139
188 0.00204715760548911
189 0.00198371890597875
190 0.00188704143948026
191 0.00185072007921008
192 0.00178386570400663
193 0.00171033344692655
194 0.00182166033426922
195 0.00180420873006186
196 0.00175784078809951
197 0.00171058157602948
198 0.00178054773966474
199 0.00168687958611384
200 0.00172867081112333
201 0.00167975445641414
202 0.00171017920176189
203 0.00165293888939752
204 0.00168010115305583
205 0.00168023969862197
206 0.0017172483698527
207 0.00160188376956516
208 0.00156726496802436
209 0.00152628678639731
210 0.0014856440056695
211 0.00149477863735624
212 0.00157288303375246
213 0.00161858125898575
214 0.00166567721472848
215 0.00171097743140328
216 0.00163963835822212
217 0.00155853736877442
218 0.00147163358476429
219 0.00134975171407065
220 0.00130853997972277
221 0.00139202206081815
222 0.00147989713033041
223 0.00145042928483751
224 0.0014778298356798
225 0.00143988797081842
226 0.00130332015567356
227 0.00131719031863743
228 0.00130128921508789
229 0.00129414184994168
230 0.0012641439098782
231 0.00129226162804498
232 0.00121398110707601
233 0.00126518361409505
234 0.00128909968058269
235 0.00132909063127306
236 0.00137212117089166
237 0.00142548030853272
238 0.00153438066270617
239 0.00152064958784316
240 0.00151526814354791
241 0.00146725146399605
242 0.00140111897362604
243 0.00135502402835424
244 0.00138867152320016
245 0.00142591662089032
246 0.00144023616366918
247 0.00143256007724339
248 0.00137736950768365
249 0.00131128246731229
250 0.00129047642178006
251 0.00127916414472792
252 0.001331789126926
253 0.00130069610595704
254 0.00123291160159642
255 0.00132231449551053
256 0.0013383750237359
257 0.00135839518229167
258 0.00144130509694418
259 0.00146648242102729
260 0.00139221159193251
261 0.00135295478820801
262 0.00124174991183811
263 0.00110874148898655
264 0.0011729372575548
265 0.00116410438961454
266 0.00114354674445259
267 0.00114626376681858
268 0.00123968160417346
269 0.00115230251312257
270 0.00113313118404813
271 0.00115739956325956
272 0.00118262379540339
273 0.00114803002675376
274 0.00111937564425999
275 0.00101124042510988
276 0.00101893461439346
277 0.00103413723839654
278 0.00102441458808052
279 0.00111183889177111
280 0.00120508129967584
281 0.00114876857333714
282 0.00108418855455187
283 0.00112771371205648
284 0.00111315217759875
285 0.00105690475463868
286 0.00114347115834555
287 0.00113832547505697
288 0.00105397919972738
289 0.00104261001586915
290 0.00111179862128365
291 0.00106175683339438
292 0.00110878234863283
293 0.00115647335900204
294 0.0011829575220744
295 0.00114722167544897
};
\end{axis}

\end{tikzpicture}

%% file: include/imagenet_obj_arxiv.tex
% This file was created by tikzplotlib v0.9.3.
\begin{tikzpicture}

\definecolor{color0}{rgb}{0.83921568627451,0.152941176470588,0.156862745098039}
\definecolor{color1}{rgb}{0.172549019607843,0.627450980392157,0.172549019607843}

\begin{axis}[
axis line style={white!80!black},
compat=newest,
height=0.25\textwidth,
log basis y={10},
tick align=outside,
tick label style={font=\tiny},
title={Training Objective},
width=0.3\textwidth,
x grid style={white!80!black},
x grid style={white},
xmajorgrids,
xmajorticks=false,
xmajorticks=true,
xmin=0, xmax=90,
xtick style={color=white!15!black},
xtick={0,30, 60, 90, 120},
y grid style={white!80!black},
ymajorgrids,
ymajorticks=false,
ymajorticks=true,
ymin=0.0001, ymax=10,
yminorticks=false,
ymode=log,
ytick style={color=white!15!black},
ytick={1e-05,0.0001,0.001,0.01,0.1,1,10,100},
yticklabels={\(\displaystyle {10^{-5}}\),\(\displaystyle {10^{-4}}\),\(\displaystyle {10^{-3}}\),\(\displaystyle {10^{-2}}\),\(\displaystyle {10^{-1}}\),\(\displaystyle {10^{0}}\),\(\displaystyle {10^{1}}\),\(\displaystyle {10^{2}}\)}
]
\addplot [thick, color0, dash pattern=on 1pt off 0.5pt]
table {%
0 3.56218794807964
1 1.80976105429876
2 1.32796717169974
3 1.0579292802679
4 0.865861204041686
5 0.711024470750385
6 0.581730241624138
7 0.471714242099012
8 0.376202609569677
9 0.299388408312991
10 0.239807596809485
11 0.190727300658545
12 0.152895371309284
13 0.124358817692513
14 0.103014586819191
15 0.0889487066284911
16 0.0733268459090934
17 0.0615252531765347
18 0.0530590608018887
19 0.044558595244939
20 0.0386503625018383
21 0.0312040283539232
22 0.0264668257815269
23 0.0237397754151231
24 0.0204431977657062
25 0.0192206294520743
26 0.0162187577886811
27 0.0148282137715615
28 0.0118264060024434
29 0.0108911619011016
30 0.00384652371790734
31 0.00165183377048425
32 0.00125496535782551
33 0.00103831465558433
34 0.000861668467942745
35 0.000785460334036239
36 0.000691212697278362
37 0.000646683675010597
38 0.000584489823272833
39 0.000566526530588372
40 0.000542326422170098
41 0.000500160002989038
42 0.000477288733342208
43 0.000439722852210073
44 0.000449410144988034
45 0.00042373690217724
46 0.000397738560030273
47 0.000379801383707171
48 0.000366489275480967
49 0.000372239546407915
50 0.000352934064156761
51 0.000343703712444161
52 0.000326003991174887
53 0.000324919822501381
54 0.000308549529584391
55 0.000294243823099667
56 0.000294622587705621
57 0.000294449784295962
58 0.00028217943104919
59 0.000272445417717736
60 0.00026665435644428
61 0.000266332508581392
62 0.000264795395346029
63 0.000263262952383993
64 0.000264626205578198
65 0.000279708853129435
66 0.000261881760397025
67 0.000256717321059738
68 0.000258364437336687
69 0.000270318245798519
70 0.000263257388761695
71 0.000254555071400841
72 0.000265615481631742
73 0.000279911240146635
74 0.000252117951652648
75 0.000261954403007904
76 0.000252838397550059
77 0.000252938735409646
78 0.000256496015784561
79 0.000254870663978308
80 0.000254464851998618
81 0.000248754616507794
82 0.00025911783052865
83 0.000254661117869002
84 0.00025001417666557
85 0.000257607362499746
86 0.000253657124990042
87 0.000246471774304586
88 0.000250572195802781
89 0.000248662269524756
};
\addplot [thick, color1]
table {%
0 4.55035919139139
1 2.10688564945709
2 1.4060710139985
3 1.03853890755776
4 0.771308380206587
5 0.566036432206252
6 0.409990099090191
7 0.297964179192613
8 0.218534901016302
9 0.161987864386746
10 0.121431849594919
11 0.0924459955225649
12 0.0715915050997548
13 0.0562206667497634
14 0.04491502182217
15 0.0364611406202497
16 0.0300282106365918
17 0.0250817345371199
18 0.0213656306210843
19 0.018288414581797
20 0.0157662215856065
21 0.0138071533894157
22 0.0121849807075638
23 0.0107939749217358
24 0.00968128659118792
25 0.00861762258143769
26 0.00777202159792511
27 0.00704163543641877
28 0.00643207237392386
29 0.00583286092209404
30 0.00532352139413057
31 0.00491766869215366
32 0.0045154257481354
33 0.00418936503213818
34 0.00383313903448445
35 0.00355041539202782
36 0.00330336680303832
37 0.00306541639269628
38 0.00286773223732603
39 0.00266925166526478
40 0.00248259720096395
41 0.00231952362366454
42 0.0021872443305994
43 0.00202871149930055
44 0.00189438405860282
45 0.00179005073296576
46 0.00169388784672327
47 0.00159258662581667
48 0.00149649234848095
49 0.00142079225991213
50 0.00133732219907028
51 0.00126261072721079
52 0.00118997663199074
53 0.00113385943234026
54 0.00107408322549552
55 0.00102796591587916
56 0.000969598714768412
57 0.000923700451534594
58 0.000877663994986136
59 0.000832270908344357
60 0.000789785998457367
61 0.000754927499049815
62 0.000721828637265459
63 0.000685803049283153
64 0.00066845497566692
65 0.000630701977755612
66 0.000616196906853159
67 0.000581715429747221
68 0.000557683266238051
69 0.000526290000939725
70 0.000505382839478454
71 0.000483159807930327
72 0.000469054391343421
73 0.000452835505934124
74 0.000432540517484864
75 0.000415775518528054
76 0.000396495218221306
77 0.00038458203876617
78 0.000369102068748304
79 0.000353897062414864
80 0.000343279686504409
81 0.000332423930921641
82 0.000323288447156822
83 0.000306720146772059
84 0.000300744174751984
85 0.00028500904424542
86 0.000280157732590869
87 0.000274703084633062
88 0.000264575554226565
89 0.000248674672987894
90 0.00024716545112175
91 0.000243937377863323
92 0.000239139135384678
93 0.000221452452213698
};
\end{axis}

\end{tikzpicture}

%% file: include/imagenet_acc_arxiv.tex
% This file was created by tikzplotlib v0.9.3.
\begin{tikzpicture}

\definecolor{color0}{rgb}{0.83921568627451,0.152941176470588,0.156862745098039}
\definecolor{color1}{rgb}{0.172549019607843,0.627450980392157,0.172549019607843}

\begin{axis}[
axis line style={white!80!black},
compat=newest,
height=0.25\textwidth,
legend cell align={left},
legend style={fill opacity=0.8, draw opacity=1, text opacity=1, at={(0.97,0.03)}, anchor=south east, draw=white!80!black},
tick align=outside,
tick label style={font=\tiny},
title={Training Top-5 Accuracy (\%)},
width=0.3\textwidth,
x grid style={white!80!black},
x grid style={white},
xmajorgrids,
xmajorticks=false,
xmajorticks=true,
xmin=0, xmax=90,
xtick style={color=white!15!black},
xtick={0,30, 60, 90, 120},
y grid style={white!80!black},
ymajorgrids,
ymajorticks=false,
ymajorticks=true,
ymin=40, ymax=101,
yminorticks=false,
ytick style={color=white!15!black}
]
\addplot [thick, color0, dash pattern=on 1pt off 0.5pt]
table {%
0 39.0376277445578
1 66.658192315128
2 74.7781371479131
3 79.3510379590882
4 82.5761920000544
5 85.2320483183362
6 87.4169689544474
7 89.3552128630301
8 91.0596134070163
9 92.4662474432572
10 93.6352205955741
11 94.7183401749745
12 95.5361015491516
13 96.2133457225612
14 96.7579514051679
15 97.131012389633
16 97.5733572884852
17 97.9276555715767
18 98.180505309353
19 98.4273444843634
20 98.6243122371317
21 98.8822790750723
22 99.0327868048949
23 99.1187216020867
24 99.2370647010563
25 99.2826312617101
26 99.3872475360383
27 99.4360630489267
28 99.552131881237
29 99.5912817605427
30 99.8677676284113
31 99.959388092264
32 99.9737647076028
33 99.9822119844119
34 99.9840801321678
35 99.9873290845758
36 99.9895221278042
37 99.9922837373197
38 99.9931771995001
39 99.9943955565221
40 99.9942331091019
41 99.9941518852862
42 99.9948828996256
43 99.9948828994148
44 99.9961824806728
45 99.9961824806727
46 99.996832271197
47 99.995532690149
48 99.9963449283039
49 99.9954514663337
50 99.9963449283036
51 99.9969134950124
52 99.9970759424326
53 99.9972383902739
54 99.9968322711969
55 99.9970759426433
56 99.9972383902741
57 99.9972383902742
58 99.9969134948018
59 99.9976445093514
60 99.9976445093515
61 99.9974820617208
62 99.9969947188276
63 99.9977257331671
64 99.9976445093515
65 99.9978069569823
66 99.9973196140895
67 99.9976445093515
68 99.9980506284289
69 99.9975632853257
70 99.9974820617207
71 99.9979694046136
72 99.9977257331671
73 99.9970759426432
74 99.9977257331669
75 99.9978069569826
76 99.9976445093515
77 99.9978069569825
78 99.9974008379052
79 99.9974008379049
80 99.9971571664585
81 99.9978881807979
82 99.9976445091407
83 99.997725733167
84 99.9978069569824
85 99.9978881807979
86 99.9978881805869
87 99.9976445093517
88 99.9973196140896
89 99.9978069569822
};
\addlegendentry{SGD}
\addplot [thick, color1]
table {%
0 25.7351161411653
1 61.9282858687005
2 73.6203728815731
3 79.9756490994398
4 84.7884850618184
5 88.6567692731624
6 91.7946889378478
7 94.1505857053299
8 95.8775664694722
9 97.1102190913476
10 97.9972643805565
11 98.6070927866903
12 99.0479756578431
13 99.3417621968703
14 99.5393797421209
15 99.671124769615
16 99.764775828421
17 99.8285365252029
18 99.8666304934572
19 99.8966020821352
20 99.917557826527
21 99.9309597560797
22 99.9406253901211
23 99.9481792041917
24 99.9525652897934
25 99.9554893483526
26 99.9593880902931
27 99.9625558190964
28 99.9644239680546
29 99.9666170110722
30 99.9662108912266
31 99.9688912779054
32 99.9702720815664
33 99.9711655435366
34 99.9727900198459
35 99.9745769437864
36 99.9750642878816
37 99.9747393914174
38 99.9764450931783
39 99.9771761058816
40 99.9781507916671
41 99.9787193583754
42 99.9788005833932
43 99.9798564917922
44 99.9803438351186
45 99.9806687303806
46 99.9806687311491
47 99.9808311775779
48 99.9816434169344
49 99.9821307586253
50 99.9822932062561
51 99.9820495348099
52 99.982374431274
53 99.9831866682267
54 99.9830242205955
55 99.9835115634883
56 99.9839176837679
57 99.9840801301965
58 99.9839989075833
59 99.9840801301967
60 99.9841613552144
61 99.9845674730895
62 99.9848111457382
63 99.9848923695536
64 99.9852172636135
65 99.9853797112445
66 99.9852172648155
67 99.9856233826906
68 99.9859482779524
69 99.9859482779526
70 99.9862731732144
71 99.986191949399
72 99.9865980684764
73 99.9861107260173
74 99.9864356220475
75 99.9868417399229
76 99.9868417399229
77 99.9870041875538
78 99.9867605161074
79 99.986841741125
80 99.9869229637383
81 99.9870854125713
82 99.987247859
83 99.9873290840178
84 99.987735201893
85 99.9875727542619
86 99.9878976499576
87 99.9879788733395
88 99.9883037686013
89 99.9879788733395
90 99.9882225447858
91 99.9882225452195
92 99.9881413226063
93 99.9885474400478
};
\addlegendentry{ALI-G}
\end{axis}

\end{tikzpicture}

%% file: appendix.tex
\section{Local Interpretation of the Polyak Step-Size}

In this section, we provide two results that shed light on a geometrical interpretation of the Polyak step-size.
First, proposition \ref{prop:nr_step} provides a proximal interpretation for the standard Polyak step-size.
Second, proposition \ref{prop:prox_step} gives a similar result when using a maximal learning-rate, which corresponds to the update used by ALI-G.

\begin{restatable}{proposition}{propnrstep} \label{prop:nr_step}
Suppose that the problem is unconstrained: $\Omega = \mathbb{R}^p$.
Let $\w_{t+1} = \w_{t} - \frac{f(\w_t) - \fstar}{\|\nabla f(\w_t) \|^2 } \nabla f(\w_t)$.
Then $\w_{t+1}$ verifies:
\begin{equation}
\w_{t+1} = \argmin_{\w \in \mathbb{R}^p} \|\w - \w_t\| \text{ subject to: } f(\w_t) + \nabla f(\w_t)^\top (\w - \w_t) = \fstar,
\end{equation}
where we remind that $\fstar$ is the minimum of $f$, and $\w \mapsto f(\w_t) + \nabla f(\w_t)^\top (\w - \w_t)$ is the linearization of $f$ at $\w_t$.
In other words, $\w_{t+1}$ is the closest point to $\w_{t}$ that lies on the hyper-plane $f(\w_t) + \nabla f(\w_t)^\top (\w - \w_t) = \fstar$.
\end{restatable}

\begin{proof}
First we show that $\w_{t+1}$ satisfies the linear equality constraint:
\begin{leftlinebox}[nobreak=false]
\begin{equation}
\begin{split}
&\kern-1em f(\w_t) + \nabla f(\w_t)^\top (\w_{t+1} - \w_t) \\
&=f(\w_t) + \nabla f(\w_t)^\top \left(-\dfrac{f(\w_t) - \fstar}{\|\nabla f(\w_t) \|^2} \nabla f(\w_t) \right), \\
&=f(\w_t) - f(\w_t) + \fstar, \\
&= \fstar.
\end{split}
\end{equation}
\end{leftlinebox}
Now let us show that it has a minimal distance to $\w_t$.
\begin{leftlinebox}[nobreak=false]
We take $\vwhat \in \mathbb{R}^p$ a solution of the linear equality constraint, and we will show that $\|\w_{t+1} - \w_t\| \leq \|\vwhat - \w_t\|$.
By definition, we have that $\vwhat$ satisfies:
\begin{equation}
f(\w_t) + \nabla f(\w_t)^\top (\vwhat - \w_t) = \fstar.
\end{equation}
Now we can write:
\begin{equation}
\begin{split}
\|\w_{t+1} - \w_t\|
    &= \| \dfrac{f(\w_t) - \fstar}{\|\nabla f(\w_t) \|^2} \nabla f(\w_t) \|, \\
    &= \dfrac{f(\w_t) - \fstar}{\|\nabla f(\w_t) \|}, \\
    &= \dfrac{|\nabla f(\w_t)^\top (\vwhat - \w_t)|}{\|\nabla f(\w_t) \|}, \\
    &\leq \dfrac{||\nabla f(\w_t)\| \|\vwhat - \w_t\|}{\|\nabla f(\w_t) \|},  \quad \text{(Cauchy-Schwarz)} \\
    &= \|\vwhat - \w_t\|.
\end{split}
\end{equation}
\end{leftlinebox}
\end{proof}

\begin{restatable}{proposition}{thproxstep}[Proximal Interpretation] \label{prop:prox_step}
Suppose that $\Omega = \mathbb{R}^p$ and let $\delta = 0$.
We consider the update performed by SGD: $\vw_{t+1}^{\text{SGD}} = \vw_t - \eta_t \nabla \ell_{z_t}(\vw_t)$; and the update performed by ALI-G: $\vw_{t+1}^{\text{ALI-G}} = \vw_t - \gamma_t \nabla \ell_{z_t}(\vw_t)$, where $\gamma_t = \min\left\{\frac{\ell_{z_t}(\vw_t)}{\| \nabla \ell_{z_t}(\vw_t)\|^2 + \delta}, \eta \right\}$.
Then we have:
\begin{align}
\vw_{t+1}^{\text{SGD}} &= \argmin_{\vw \in \mathbb{R}^p} \Big\{
        \frac{1}{2 \eta_t} \| \vw - \vw_t \|^2 + \ell_{z_t}(\vw_t) + \nabla \ell_{z_t}(\vw_t)^\top (\vw - \vw_t) \Big\}, \\
\vw_{t+1}^{\text{ALI-G}} &= \argmin_{\vw \in \mathbb{R}^p} \Big\{
    \frac{1}{2 \eta} \| \vw - \vw_t \|^2 +
    \max \left\{\ell_{z_t}(\vw_t) + \nabla \ell_{z_t}(\vw_t)^\top (\vw - \vw_t), 0 \right\}  \Big\}. \label{eq:prox_pb}
\end{align}
\end{restatable}

\begin{proof} \
In order to make the notation simpler, we use $\bm{d}_t \triangleq \nabla \ell_{z_t} (\w_t)$ and $l_t \triangleq \ell_{z_t} (\w_t)$. \\
First, let us consider $\bm{d}_t = \bm{0}$.
\begin{leftlinebox}[nobreak=false]
Then we choose $\gamma_t=0$ and it is clear that $\w_{t+1} = \w_t - \eta \gamma_t \bm{d}_t = \w_t$ is the optimal solution of problem (\ref{eq:prox_pb}).
\end{leftlinebox}
We now assume $\bm{d}_t \neq \bm{0}$.

\begin{leftlinebox}[nobreak=false]
We can successively re-write the proximal problem (\ref{eq:prox_pb}) as :
\begin{align*}
    \min_{\w \in \mathbb{R}^p}
        &\left\{ \dfrac{1}{2 \eta} \| \w - \w_t \|^2 + \max \left\{\ell_{z_t} (\w_t) + \nabla \ell_{z_t} (\w_t)^\top (\w - \w_t), 0 \right\}  \right\}, \\
    \min_{\w \in \mathbb{R}^p}
        &\left\{ \dfrac{1}{2 \eta} \| \w - \w_t \|^2 + \max \left\{l_t + \bm{d}_t^\top (\w - \w_t), 0 \right\}  \right\}, \\
    \min_{\w \in \mathbb{R}^p, \upsilon}
        &\left\{ \dfrac{1}{2 \eta} \| \w - \w_t \|^2 + \upsilon \right\} \: \text{subject to: } \: \upsilon \geq 0, \: \upsilon \geq l_t + \bm{d}_t^\top (\w - \w_t) \\
    \min_{\w \in \mathbb{R}^p, \upsilon} \sup_{\mu, \nu \geq 0}
         &\left\{ \dfrac{1}{2 \eta} \| \w - \w_t \|^2 + \upsilon - \mu \upsilon - \nu (\upsilon - l_t - \bm{d}_t^\top (\w - \w_t)) \right\} \\
    \sup_{\mu, \nu \geq 0} \min_{\w \in \mathbb{R}^p, \upsilon}
         &\left\{ \dfrac{1}{2 \eta} \| \w - \w_t \|^2 + \upsilon - \mu \upsilon - \nu (\upsilon - l_t - \bm{d}_t^\top (\w - \w_t)) \right\}, \stepcounter{equation}\tag{\theequation} \\
\end{align*}
where the last equation uses strong duality.
The inner problem is now smooth in $\w$ and $\upsilon$. We write its KKT conditions:
\begin{equation}
    \dfrac{\partial \cdot}{\partial \upsilon} = 0: \quad 1 - \mu - \nu = 0
\end{equation}
\begin{equation}
    \dfrac{\partial \cdot}{\partial \w} = 0: \quad \frac{1}{\eta}(\w - \w_t) + \nu \bm{d}_t = \bm{0}
\end{equation}
We plug in these results and obtain:
\begin{align*}
    \sup_{\mu, \nu \geq 0}
        &\left\{ \dfrac{1}{2 \eta} \| \eta \nu \bm{d}_t \|^2 + \nu (l_t + \bm{d}_t^\top (-\eta \nu \bm{d}_t)) \right\} \\
        \text{st: } \quad &\mu + \nu = 1 \\
        \sup_{\nu \in [0, 1]}
        &\left\{ \dfrac{\eta}{2} \nu^2 \| \bm{d}_t \|^2 + \nu l_t - \eta \nu^2 \| \bm{d}_t^\top\|^2 \right\} \\
        \sup_{\nu \in [0, 1]}
        &\left\{ -\dfrac{\eta}{2} \nu^2 \| \bm{d}_t \|^2 + \nu l_t \right\} \stepcounter{equation}\tag{\theequation}
\end{align*}
This is a one-dimensional quadratic problem in $\nu$.
It can be solved in closed-form by finding the global maximum of the quadratic objective, and projecting the solution on $[0, 1]$.
We have:
\begin{equation}
    \dfrac{\partial \cdot}{\partial \nu} = 0: - \eta \nu \| \bm{d}_t \|^2 + l_t = 0
\end{equation}
Since $\bm{d}_t \neq \bm{0}$ and $\eta \neq 0$, this gives the optimal solution:
\begin{equation}
\nu = \min \left\{ \max \left\{\dfrac{l_t}{\eta \|\bm{d}_t\|^2}, 0 \right\}, 1 \right\} = \min \left\{\dfrac{l_t}{\eta \|\bm{d}_t\|^2}, 1 \right\},
\end{equation}
since $l_t, \eta, \|\bm{d}_t\|^2 \geq 0$. \\
Plugging this back in the KKT conditions, we obtain that the solution $\w_{t+1}$ of the primal problem can be written as:
\begin{equation}
\begin{split}
\w_{t+1}
    &= \w_t - \eta \nu \bm{d}_t, \\
    &= \w_t - \eta \min \left\{\dfrac{l_t}{\eta \|\bm{d}_t\|^2}, 1 \right\} \bm{d}_t, \\
    &= \w_t - \eta \min \left\{\dfrac{\ell_{z_t}(\w_t)}{\eta \|\nabla \ell_{z_t}(\w_t)\|^2}, 1 \right\} \nabla \ell_{z_t}(\w_t), \\
    &= \w_t - \min \left\{\dfrac{\ell_{z_t}(\w_t)}{\|\nabla \ell_{z_t}(\w_t)\|^2}, \eta \right\} \nabla \ell_{z_t}(\w_t). \\
\end{split}
\end{equation}
\end{leftlinebox}
\end{proof}

\section{Summary of Convergence Results}
\label{app:sec:convergence}

\paragraph{Problem Formulation.}
We remind the problem setting as follows.
The learning task can be expressed as the problem $(\mathcal{P})$ of finding a feasible vector of parameters $\wstar \in \Omega$ that minimizes $f$:
\begin{equation} \tag{$\mathcal{P}$} \label{eq:main_problem}
    \wstar \in \argmin\limits_{\vw \in \Omega} f(\vw).
\end{equation}
Also note that $\fstar$ refers to the minimum value of $f$ over $\Omega$: $\fstar \triangleq \min_{\vw \in \Omega} f(\vw)$.

In the remainder of this section, we give an overview of convergence results of ALI-G in various stochastic settings.
First, we summarize convergence results in the convex setting in section \ref{subsec:summary_cvx_results}.
Notably, these results show convergence for any maximal learning-rate $\eta$, including $\eta = \infty$, which is equivalent to not using any clipping to a maximal value.
Second, we give results for a class of non-convex problems.
These results show that a maximal learning-rate is necessary and sufficient for convergence of the Polyak step-size.
Indeed we show that the Polyak step-size can oscillate indefinitely without a maximal learning-rate, and that using a maximal learning-rate provably leads to (exponentially fast) convergence.

\subsection{Convex Setting}
\label{subsec:summary_cvx_results}

For simplicity purposes, we assume that we are in the perfect interpolation setting: $\forall z, \: \ell_z(\wstar) = 0$.
Detailed results with an interpolation tolerance $\varepsilon > 0$ are given in section \ref{app:sec:detailed_cvx_results}.
Since we are in the perfect interpolation setting, note that we can safely set the small constant for numerical stability to zero: $\delta = 0$.
The summary of the results is presented in table \ref{tab:cvx_results}.

\begin{table}[ht]
\centering
\footnotesize
\begin{tabular}{llcc}
    \toprule
    Assumption on Loss Functions &Distance Considered & \multicolumn{2}{c}{Convergence Rate} \\
    \cmidrule(lr){1-1} \cmidrule(lr){2-2} \cmidrule(lr){3-4}
    && Small $\eta$ & Large $\eta$ (potentially $\infty$) \\
    \cmidrule(lr){3-3} \cmidrule(lr){4-4}
    Convex and $C$-Lipschitz & $\E \left[f\left(\tfrac{1}{T+ 1} \sum\limits_{t=0}^T \w_t \right) \right] - \fstar$ & $\tfrac{\| \w_{0} - \wstar \|^2}{\eta (T + 1)} + \sqrt{ \tfrac{C^2 \| \w_{0} - \wstar \|^2}{T + 1}}$  & $\sqrt{ \tfrac{C^2 \| \w_{0} - \wstar \|^2}{T + 1}} $ \\
    Convex and $\beta$-Smooth & $\E \left[f\left(\tfrac{1}{T+ 1} \sum\limits_{t=0}^T \w_t \right) \right] - \fstar$ & $\tfrac{\| \w_{0} - \wstar \|^2}{\eta (T+1)}$  & $\tfrac{2 \beta \|\w_{0} - \wstar\|^2}{T + 1}$ \\
    $\alpha$-Strongly Convex and $\beta$-Smooth & $\E[f(\w_{T+1})] - \fstar$ & $\tfrac{\beta}{2} \exp \left(\tfrac{-\alpha \eta T }{2} \right)  \| \w_{0} - \wstar \|^2$ & $\tfrac{\beta}{2} \exp\left(- \tfrac{\alpha t}{4 \beta} \right)  \| \w_{0} - \wstar \|^2$ \\
    \bottomrule
    \end{tabular}
\caption{\em
    Summary of convergence rates for convex problems in the perfect interpolation setting.
    We remind that $\eta$ denotes the hyper-parameter used by ALI-G to clip its learning-rate to a maximal value.
    Our convergence results yield different results when $\eta$ has a small value (middle column), and when $\eta$ has a large, possibly even infinite, value (right column).
    The formal statements of these results are available in section \ref{app:sec:detailed_cvx_results}, along with their proofs.
    }
\label{tab:cvx_results}
\end{table}

The overall convergence speed is similar to that of \emph{non-stochastic} Polyak step-size, which is itself the same as the optimal rate of \emph{non-stochastic} gradient descent: $\mathcal{O}(1/\sqrt{T})$ for convex Lipschitz functions, $\mathcal{O}(1/T)$ for convex and smooth functions, and $\mathcal{O}(\exp(-kT))$ (for some constant $k$) for smooth and strongly convex functions \citep{Hazan2019}.

\subsection{Non-Convex Setting}

We also assume that we are in the perfect interpolation setting and thus we set the constant for numerical stability $\delta$ to zero.
We further assume that the problem is unconstrained.
The summary of the results is presented in table \ref{tab:noncvx_results}.

\begin{table}[H]
\centering
\footnotesize
\begin{tabular}{cc}
    \toprule
    \multicolumn{2}{c}{Convergence Result} \\
    \midrule
     $0 < \eta \leq \tfrac{2 \alpha}{\beta^2}$  & $\eta=\infty$ \\
     \cmidrule(lr){1-1} \cmidrule(lr){2-2}
     $ f(\w_{T+1}) - \fstar \leq \tfrac{\beta}{2} \exp \left( - \kappa T \right) \| \w_{0} - \wstar \|^2$ &  Can Fail to Converge (Proved) \\
    \bottomrule
    \end{tabular}
\caption{\em
    Summary of convergence results for $\alpha$-RSI and $\beta$-smooth loss functions in the perfect interpolation setting.
    We remind that $\eta$ denotes the hyper-parameter used by ALI-G to clip its learning-rate to a maximal value.
    The constant $\kappa$ depends on $\alpha$, $\beta$ and $\eta$.
    These results show that using a maximal learning-rate is necessary and sufficient for convergence.
    The formal statements of these results are available in section \ref{app:sec:detailed_noncvx_results}, along with their proofs.
    }
\label{tab:noncvx_results}
\end{table}

\section{Detailed Convex Results}
\label{app:sec:detailed_cvx_results}

\subsection{Lipschitz Convex Functions}

\begin{restatable}{theorem}{thaligcvxlargeeta}\label{th:alig_cvx_large_eta}
We assume that $\Omega$ is a convex set, and that for every $z \in \Z$, $\ell_z$ is convex and $C$-Lipschitz.
Let $\wstar$ be a solution of (\ref{eq:main_problem}) such that $\forall z \in \Z, \: \ell_z(\wstar) \leq \varepsilon$.
We further assume that $\eta > \frac{\varepsilon}{\delta}$.
Then if we apply ALI-G with a maximal learning-rate of $\eta$ to $f$, we have:
\begin{equation}
\begin{split}
\E \left[f\left(\frac{1}{T+ 1} \sum\limits_{t=0}^T \w_t \right) \right] - \fstar
    &\leq \dfrac{\| \w_{0} - \wstar \|^2}{(\eta - \frac{\varepsilon}{\delta}) (T + 1)} + \dfrac{\varepsilon^2}{\delta (\eta - \frac{\varepsilon}{\delta})} \\
    &\quad + \sqrt{ \dfrac{(C^2 + \delta) \| \w_{0} - \wstar \|^2}{T + 1}} + \varepsilon \sqrt{\dfrac{C^2}{\delta} + 1}.
\end{split}
\end{equation}
\end{restatable}

\begin{proof} \
\begin{leftlinebox}[nobreak=false]
We consider the update at time $t$, which we condition on the draw of $z_t \in \Z$:
\begin{align*}
&\kern-1em \| \w_{t+1} - \wstar \|^2 \\
&= \| \Pi_\Omega(\w_t - \gamma_t \nabla \ell_{z_t}(\w_t)) - \wstar \|^2 \\
&\leq \| \w_t - \gamma_t \nabla \ell_{z_t}(\w_t) - \wstar \|^2 \eqcomment{$\Pi_\Omega$ projection} \\
&= \| \w_{t} - \wstar \|^2 - 2 \gamma_t \nabla \ell_{z_t}(\w_t)^\top(\w_{t} - \wstar) + \gamma_t^2 \| \nabla \ell_{z_t}(\w_t) \|^2 \\
&\leq \| \w_{t} - \wstar \|^2 - 2 \gamma_t \nabla \ell_{z_t}(\w_t)^\top(\w_{t} - \wstar) + \gamma_t \frac{\ell_{z_t}(\w_t)}{\| \nabla \ell_{z_t}(\w_t) \|^2 + \delta} \| \nabla \ell_{z_t}(\w_t) \|^2 \\
&\quad \eqcomment{because $\gamma_t \leq \frac{\ell_{z_t}(\w_t)}{\| \nabla \ell_{z_t}(\w_t) \|^2 + \delta}$} \\
&\leq \| \w_{t} - \wstar \|^2 - 2 \gamma_t \nabla \ell_{z_t}(\w_t)^\top(\w_{t} - \wstar) + \gamma_t \frac{\ell_{z_t}(\w_t)}{\| \nabla \ell_{z_t}(\w_t) \|^2} \| \nabla \ell_{z_t}(\w_t) \|^2 \\
&\quad \eqcomment{because $\ell_{z_t}(\w_t) \geq 0$ and $\delta \geq 0$} \\
&\leq \| \w_{t} - \wstar \|^2 - 2 \gamma_t (\ell_{z_t}(\w_t) - \ell_{z_t}(\wstar)) + \gamma_t \ell_{z_t}(\w_t) \eqcomment{convexity of $\ell_{z_t}$} \\
&= \| \w_{t} - \wstar \|^2 - 2\gamma_t (\ell_{z_t}(\w_t) - \ell_{z_t}(\wstar)) + \gamma_t (\ell_{z_t}(\w_t) - \ell_{z_t}(\wstar)) + \gamma_t \ell_{z_t}(\wstar) \\
&= \| \w_{t} - \wstar \|^2 - \gamma_t (\ell_{z_t}(\w_t) - \ell_{z_t}(\wstar)) + \gamma_t \ell_{z_t}(\wstar) \stepcounter{equation}\tag{\theequation}\label{eq:alig_cvx_basic_iterate_bound}
\end{align*}
\end{leftlinebox}
\begin{leftlinebox}[nobreak=false]
We now consider different cases, according to the value that $\gamma_t$ takes: $\gamma_t = \frac{\ell_{z_t}(\w_t)}{\| \nabla \ell_{z_t}(\w_t) \|^2 + \delta}$ or $\gamma_t = \eta$.
\begin{leftlinebox}[nobreak=false]
First, suppose that $\gamma_t = \frac{\ell_{z_t}(\w_t)}{\| \nabla \ell_{z_t}(\w_t) \|^2 + \delta}$.
Then we have:
\begin{align*}
&\| \w_{t+1} - \wstar \|^2 \\
&\leq \| \w_{t} - \wstar \|^2 - \gamma_t \Big( \ell_{z_t}(\w_t) - 2 \ell_{z_t}(\wstar) \Big)\\
&= \| \w_{t} - \wstar \|^2 - \dfrac{1}{\| \nabla \ell_{z_t}(\w_t) \|^2 + \delta} \Big( \ell_{z_t}(\w_t)^2 - 2 \ell_{z_t}(\w_t) \ell_{z_t}(\wstar) \Big) \\
&= \| \w_{t} - \wstar \|^2 - \dfrac{1}{\| \nabla \ell_{z_t}(\w_t) \|^2 + \delta} \Big( (\ell_{z_t}(\w_t) - \ell_{z_t}(\wstar))^2 - \ell_{z_t}(\wstar)^2 \Big) \\
&= \| \w_{t} - \wstar \|^2 - \dfrac{(\ell_{z_t}(\w_t) - \ell_{z_t}(\wstar))^2}{\| \nabla \ell_{z_t}(\w_t) \|^2 + \delta} + \dfrac{\ell_{z_t}(\wstar)^2}{\| \nabla \ell_{z_t}(\w_t) \|^2 + \delta} \\
&\leq \| \w_{t} - \wstar \|^2 - \dfrac{(\ell_{z_t}(\w_t) - \ell_{z_t}(\wstar))^2}{C^2 + \delta} + \dfrac{\ell_{z_t}(\wstar)^2}{\delta} \\
&\quad \eqcomment{because we have $0 \leq \| \nabla \ell_{z_t}(\w_t)\|^2 \leq C^2$} \\
&\leq \| \w_{t} - \wstar \|^2 - \dfrac{(\ell_{z_t}(\w_t) - \ell_{z_t}(\wstar))^2}{C^2 + \delta} + \dfrac{\varepsilon^2}{\delta} \eqcomment{definition of $\varepsilon$} \stepcounter{equation}\tag{\theequation} \label{eq:iterate_bound_alig_cvx_1}
\end{align*}

\end{leftlinebox}
\begin{leftlinebox}[nobreak=false]
Now suppose $\gamma_t = \eta$ and $\ell_{z_t}(\w_t) - \ell_{z_t}(\wstar) \leq 0$.
We can use $\gamma_t \leq \frac{\ell_{z_t}(\w_t)}{\| \nabla \ell_{z_t}(\w_t) \|^2 + \delta}$ to write:
\begin{equation}
\begin{split}
\| \w_{t+1} - \wstar \|^2
    &\leq \| \w_{t} - \wstar \|^2 - \gamma_t (\ell_{z_t}(\w_t) - \ell_{z_t}(\wstar)) + \gamma_t \ell_{z_t}(\wstar), \\
    &\leq \| \w_{t} - \wstar \|^2 - \frac{\ell_{z_t}(\w_t)}{\| \nabla \ell_{z_t}(\w_t) \|^2 + \delta} (\ell_{z_t}(\w_t) - \ell_{z_t}(\wstar))  + \frac{\ell_{z_t}(\w_t)}{\| \nabla \ell_{z_t}(\w_t) \|^2 + \delta} \ell_{z_t}(\wstar),
\end{split}
\end{equation}
where the last inequality has used $\gamma_t \leq \frac{\ell_{z_t}(\w_t)}{\| \nabla \ell_{z_t}(\w_t) \|^2 + \delta}$, $\ell_{z_t}(\w_t) - \ell_{z_t}(\wstar) \leq 0$ and $\ell_{z_t}(\wstar) \geq 0$.
Therefore we are exactly in the same situation as the first case (where we used $\gamma_t = \frac{\ell_{z_t}(\w_t)}{\| \nabla \ell_{z_t}(\w_t) \|^2 + \delta}$), and thus we have again:
\begin{equation} \label{eq:iterate_bound_alig_cvx_2}
    \| \w_{t+1} - \wstar \|^2 \leq \| \w_{t} - \wstar \|^2 - \dfrac{(\ell_{z_t}(\w_t) - \ell_{z_t}(\wstar))^2}{C^2 + \delta} + \dfrac{\varepsilon^2}{\delta}.
\end{equation}
\end{leftlinebox}
\begin{leftlinebox}[nobreak=false]
Now suppose that $\gamma_t = \eta$ and $\ell_{z_t}(\w_t) - \ell_{z_t}(\wstar) \geq 0$.
The inequality (\ref{eq:alig_cvx_basic_iterate_bound}) gives:
\begin{equation} \label{eq:iterate_bound_alig_cvx_3}
\begin{split}
&\kern-1em \| \w_{t+1} - \wstar \|^2 \\
    &\leq \| \w_{t} - \wstar \|^2 - \gamma_t (\ell_{z_t}(\w_t) - \ell_{z_t}(\wstar)) + \gamma_t \ell_{z_t}(\wstar), \\
    &= \| \w_{t} - \wstar \|^2 - \eta (\ell_{z_t}(\w_t) - \ell_{z_t}(\wstar)) + \gamma_t \ell_{z_t}(\wstar), \eqcomment{$\gamma_t = \eta$} \\
    &\leq  \| \w_{t} - \wstar \|^2 - \eta (\ell_{z_t}(\w_t) - \ell_{z_t}(\wstar)) + \gamma_t \varepsilon, \eqcomment{definition of $\varepsilon$, $\gamma_t \geq 0$} \\
    &\leq  \| \w_{t} - \wstar \|^2 - \eta (\ell_{z_t}(\w_t) - \ell_{z_t}(\wstar)) + \varepsilon \dfrac{\ell_{z_t}(\w_t)}{\| \nabla \ell_{z_t}(\w_t) \|^2 + \delta}, \\
    &\quad \eqcomment{because $\gamma_t \leq \frac{\ell_{z_t}(\w_t)}{\| \nabla \ell_{z_t}(\w_t) \|^2 + \delta}$, $\varepsilon \geq 0$} \\
    &\leq  \| \w_{t} - \wstar \|^2 - \eta (\ell_{z_t}(\w_t) - \ell_{z_t}(\wstar)) + \varepsilon \dfrac{\ell_{z_t}(\w_t)}{\delta}, \\
    &\quad \eqcomment{because $\| \nabla \ell_{z_t}(\w_t) \|^2 \geq 0$} \\
    &= \| \w_{t} - \wstar \|^2 - \eta (\ell_{z_t}(\w_t) - \ell_{z_t}(\wstar)) + \varepsilon \dfrac{\ell_{z_t}(\w_t) - \ell_{z_t}(\wstar) + \ell_{z_t}(\wstar)}{\delta}, \\
    &\leq \| \w_{t} - \wstar \|^2 - \eta (\ell_{z_t}(\w_t) - \ell_{z_t}(\wstar)) + \varepsilon \dfrac{\ell_{z_t}(\w_t) - \ell_{z_t}(\wstar) + \varepsilon}{\delta}, \\
    &\quad \eqcomment{because $\ell_{z_t}(\wstar) \leq \varepsilon$} \\
    &= \| \w_{t} - \wstar \|^2 - \left(\eta - \dfrac{\varepsilon}{\delta}\right) (\ell_{z_t}(\w_t) - \ell_{z_t}(\wstar)) + \dfrac{\varepsilon^2}{\delta}.
\end{split}
\end{equation}
\end{leftlinebox}
\end{leftlinebox}
\begin{leftlinebox}[nobreak=false]
We now introduce $\mathcal{I}_T$ and $\mathcal{J}_T$ as follows:
\begin{equation}
\begin{split}
    \mathcal{I}_T &\triangleq \left\{ t \in \{0, ..., T\} : \gamma_t = \eta \ \text{and} \ \ell_{z_t}(\w_t) - \ell_{z_t}(\wstar) \geq 0 \right\} \\
    \mathcal{J}_T &\triangleq \{0, ..., T\} \ \backslash \ \mathcal{I}_T
\end{split}
\end{equation}
Then, by combining inequalities (\ref{eq:iterate_bound_alig_cvx_1}), (\ref{eq:iterate_bound_alig_cvx_2}) and (\ref{eq:iterate_bound_alig_cvx_3}), and using a telescopic sum, we obtain:
\begin{equation}
\begin{split}
\| \w_{T+1} - \wstar \|^2
    &\leq \| \w_{0} - \wstar \|^2 + \sum\limits_{t \in \mathcal{J}_T} \left( -\dfrac{(\ell_{z_t}(\w_t) - \ell_{z_t}(\wstar))^2}{C^2 + \delta} + \dfrac{\varepsilon^2}{\delta}\right) \\
    &\qquad + \sum\limits_{t \in \mathcal{I}_T} \left( -\left(\eta - \dfrac{\varepsilon}{\delta}\right) (\ell_{z_t}(\w_t) - \ell_{z_t}(\wstar)) + \dfrac{\varepsilon^2}{\delta} \right)
\end{split}
\end{equation}
Using $\| \w_{T+1} - \wstar \|^2 \geq 0$, we obtain:
\begin{equation} \label{eq:alig_cvx_two_terms_bounded}
\begin{split}
\dfrac{1}{C^2 + \delta} \sum\limits_{t \in \mathcal{J}_T} (\ell_{z_t}(\w_t) - \ell_{z_t}(\wstar))^2 + \left(\eta - \dfrac{\varepsilon}{\delta}\right) \sum\limits_{t \in \mathcal{I}_T} (\ell_{z_t}(\w_t) - \ell_{z_t}(\wstar)) \\
    \leq \| \w_{0} - \wstar \|^2 + (T + 1) \dfrac{\varepsilon^2}{\delta}
\end{split}
\end{equation}
\end{leftlinebox}
\begin{leftlinebox}[nobreak=false]
In particular, the inequality (\ref{eq:alig_cvx_two_terms_bounded}) gives that:
\begin{equation} \label{eq:alig_cvx_sum_normal}
\left(\eta - \dfrac{\varepsilon}{\delta}\right) \sum\limits_{t \in \mathcal{I}_T} (\ell_{z_t}(\w_t) - \ell_{z_t}(\wstar))
    \leq \| \w_{0} - \wstar \|^2 + (T + 1) \dfrac{\varepsilon^2}{\delta}.
\end{equation}
Furthermore, for every $t \in \mathcal{I}_T$, we have $(\ell_{z_t}(\w_t) - \ell_{z_t}(\wstar)) \geq 0$, which yields $\left(\eta - \frac{\varepsilon}{\delta}\right) \sum\limits_{t \in \mathcal{I}_T} (\ell_{z_t}(\w_t) - \ell_{z_t}(\wstar)) \geq 0$ since $\eta > \frac{\epsilon}{\delta}$.
Thus the inequality (\ref{eq:alig_cvx_two_terms_bounded}) also gives:
\begin{equation}
\dfrac{1}{C^2 + \delta} \sum\limits_{t \in \mathcal{J}_T} (\ell_{z_t}(\w_t) - \ell_{z_t}(\wstar))^2
    \leq \| \w_{0} - \wstar \|^2 + (T + 1) \dfrac{\varepsilon^2}{\delta}.
\end{equation}
Using the Cauchy-Schwarz inequality, we can further write:
\begin{equation}
\left( \sum\limits_{t \in \mathcal{J}_T} \ell_{z_t}(\w_t) - \ell_{z_t}(\wstar) \right)^2
    \leq |\mathcal{J}_T| \sum\limits_{t \in \mathcal{J}_T} (\ell_{z_t}(\w_t) - \ell_{z_t}(\wstar))^2.
\end{equation}
Therefore we have:
\begin{equation} \label{eq:alig_cvx_sum_squares}
\begin{split}
\sum\limits_{t \in \mathcal{J}_T} \ell_{z_t}(\w_t) - \ell_{z_t}(\wstar)
    &\leq \sqrt{|\mathcal{J}_T| \sum\limits_{t \in \mathcal{J}_T} (\ell_{z_t}(\w_t) - \ell_{z_t}(\wstar))^2}, \\
    &\leq \sqrt{|\mathcal{J}_T| (C^2 + \delta) \left( \| \w_{0} - \wstar \|^2 + (T + 1) \dfrac{\varepsilon^2}{\delta} \right)}. \\
\end{split}
\end{equation}
\end{leftlinebox}
\begin{leftlinebox}[nobreak=false]
We can now put together inequalities (\ref{eq:alig_cvx_sum_normal}) and (\ref{eq:alig_cvx_sum_squares}) by writing:
\begin{equation}
\begin{split}
&\kern-1em \sum\limits_{t=0}^T \ell_{z_t}(\w_t) - \ell_{z_t}(\wstar) \\
    &= \sum\limits_{t \in \mathcal{I}_T} \ell_{z_t}(\w_t) - \ell_{z_t}(\wstar) + \sum\limits_{t \in \mathcal{J}_T} \ell_{z_t}(\w_t) - \ell_{z_t}(\wstar) \\
    &\leq \dfrac{1}{\eta - \frac{\varepsilon}{\delta}} \left( \| \w_{0} - \wstar \|^2 + (T + 1) \dfrac{\varepsilon^2}{\delta} \right) \\
    &\quad + \sqrt{|\mathcal{J}_T| (C^2 + \delta) \left( \| \w_{0} - \wstar \|^2 + (T + 1) \dfrac{\varepsilon^2}{\delta} \right)} \\
    &\leq \dfrac{1}{\eta - \frac{\varepsilon}{\delta}} \left( \| \w_{0} - \wstar \|^2 + (T + 1) \dfrac{\varepsilon^2}{\delta} \right) \\
    &\quad + \sqrt{(T + 1) (C^2 + \delta) \left( \| \w_{0} - \wstar \|^2 + (T + 1) \dfrac{\varepsilon^2}{\delta} \right)}
\end{split}
\end{equation}
Dividing by $T+1$ and taking the expectation (over $z_1, ..., z_T$), we obtain:
\begin{equation}
\begin{split}
&\kern-1em \E \left[f\left(\dfrac{1}{T+1} \sum\limits_{t=0}^T \w_t \right) \right] - \fstar \\
    &\leq \dfrac{1}{T+1} \sum\limits_{t=0}^T \E [f(\w_t)] - \fstar, \eqcomment{$f$ is convex} \\
    &\leq \dfrac{\| \w_{0} - \wstar \|^2}{(\eta - \frac{\varepsilon}{\delta}) (T + 1)} + \dfrac{\varepsilon^2}{\delta (\eta - \frac{\varepsilon}{\delta})} + \sqrt{(C^2 + \delta) \left( \dfrac{\| \w_{0} - \wstar \|^2}{T + 1} + \dfrac{\varepsilon^2}{\delta} \right)}, \\
    &\leq \dfrac{\| \w_{0} - \wstar \|^2}{(\eta - \frac{\varepsilon}{\delta}) (T + 1)} + \dfrac{\varepsilon^2}{\delta (\eta - \frac{\varepsilon}{\delta})} + \sqrt{ \dfrac{(C^2 + \delta) \| \w_{0} - \wstar \|^2}{T + 1}} + \varepsilon \sqrt{\dfrac{C^2}{\delta} + 1}.
\end{split}
\end{equation}
\end{leftlinebox}
\end{proof}

When $\eta$ is small, the convergence error of Theorem \ref{th:alig_cvx_large_eta} is large.
This is corrected in the following result which is informative in the regime where $\eta$ is small:

\begin{restatable}{theorem}{thaligcvxsmalleta}\label{th:alig_cvx_small_eta}
We assume that $\Omega$ is a convex set, and that for every $z \in \Z$, $\ell_z$ is convex and $C$-Lipschitz.
Let $\wstar$ be a solution of (\ref{eq:main_problem}) such that $\forall z \in \Z, \: \ell_z(\wstar) \leq \varepsilon$.
Then if we apply ALI-G with a maximal learning-rate of $\eta$ to $f$, we have:
\begin{equation}
\E \left[f\left(\frac{1}{T+ 1} \sum\limits_{t=0}^T \w_t \right) \right] - \fstar
    \leq \dfrac{\| \w_{0} - \wstar \|^2}{\eta (T + 1)} + 2 \varepsilon + \sqrt{ \dfrac{(C^2 + \delta) \| \w_{0} - \wstar \|^2}{T + 1}} + 2 \eta \varepsilon \sqrt{C^2 + \delta}.
\end{equation}
\end{restatable}

\begin{proof} \
\begin{leftlinebox}[nobreak=false]
We consider the update at time $t$, which we condition on the draw of $z_t \in \Z$.
We re-use the inequality (\ref{eq:alig_cvx_basic_iterate_bound}) from the proof of Theorem \ref{th:alig_cvx_large_eta}:
\begin{equation} \label{eq:alig_cvx_small_eta_iterate_bound}
\| \w_{t+1} - \wstar \|^2 \\
    \leq \| \w_{t} - \wstar \|^2 - \gamma_t (\ell_{z_t}(\w_t) - \ell_{z_t}(\wstar)) + \gamma_t \ell_{z_t}(\wstar)
\end{equation}
\end{leftlinebox}
\begin{leftlinebox}[nobreak=false]
We consider again different cases, according to the value of $\gamma_t$ and the sign of $\ell_{z_t}(\w_t) - \ell_{z_t}(\wstar)$.
\begin{leftlinebox}[nobreak=false]
Suppose that $\ell_{z_t}(\w_t) - \ell_{z_t}(\wstar) < 0$.
Then the inequality (\ref{eq:alig_cvx_small_eta_iterate_bound}) gives:
\begin{equation} \label{eq:iterate_bound_alig_cvx_small_eta_1}
\begin{split}
&\kern-1em \| \w_{t+1} - \wstar \|^2 \\
    &\leq \| \w_{t} - \wstar \|^2 - \gamma_t (\ell_{z_t}(\w_t) - \ell_{z_t}(\wstar)) + \gamma_t \ell_{z_t}(\wstar), \\
    &= \| \w_{t} - \wstar \|^2 - \gamma_t \ell_{z_t}(\w_t) + 2 \gamma_t \ell_{z_t}(\wstar), \\
    &\leq \| \w_{t} - \wstar \|^2 + 2 \gamma_t \ell_{z_t}(\wstar), \eqcomment{$\gamma_t, \ell_{z_t}(\w_t) \geq 0$}\\
    &\leq \| \w_{t} - \wstar \|^2 + 2 \eta \varepsilon, \eqcomment{$\gamma_t \leq \eta$, definition of $\varepsilon$} \\
\end{split}
\end{equation}
\end{leftlinebox}
\begin{leftlinebox}[nobreak=false]
Now suppose $\ell_{z_t}(\w_t) - \ell_{z_t}(\wstar) \geq 0$ and $\gamma_t = \eta$.
Then the inequality (\ref{eq:alig_cvx_small_eta_iterate_bound}) gives:
\begin{equation} \label{eq:iterate_bound_alig_cvx_small_eta_2}
\begin{split}
\| \w_{t+1} - \wstar \|^2
    &\leq \| \w_{t} - \wstar \|^2 - \gamma_t (\ell_{z_t}(\w_t) - \ell_{z_t}(\wstar)) + \gamma_t \ell_{z_t}(\wstar), \\
    &= \| \w_{t} - \wstar \|^2 - \eta (\ell_{z_t}(\w_t) - \ell_{z_t}(\wstar)) + \eta \ell_{z_t}(\wstar), \\
    &\quad \eqcomment{because $\gamma_t = \eta$} \\
    &\leq \| \w_{t} - \wstar \|^2 - \eta (\ell_{z_t}(\w_t) - \ell_{z_t}(\wstar)) + \eta \varepsilon, \\
    &\quad \eqcomment{definition of $\varepsilon$, $\eta \geq 0$} \\
\end{split}
\end{equation}
\end{leftlinebox}
\begin{leftlinebox}[nobreak=false]
Finally, suppose that $\ell_{z_t}(\w_t) - \ell_{z_t}(\wstar) \geq 0$ and $\gamma_t = \frac{\ell_{z_t}(\w_t)}{\| \nabla \ell_{z_t}(\w_t) \|^2 + \delta}$.
Then the inequality (\ref{eq:alig_cvx_small_eta_iterate_bound}) gives:
\begin{equation} \label{eq:iterate_bound_alig_cvx_small_eta_3}
\begin{split}
&\kern-1em \| \w_{t+1} - \wstar \|^2 \\
    &\leq \| \w_{t} - \wstar \|^2 - \gamma_t (\ell_{z_t}(\w_t) - \ell_{z_t}(\wstar)) + \gamma_t \ell_{z_t}(\wstar), \\
    &\leq \| \w_{t} - \wstar \|^2 - \gamma_t (\ell_{z_t}(\w_t) - \ell_{z_t}(\wstar)) + \eta \ell_{z_t}(\wstar), \\
    &\quad \eqcomment{because $\gamma_t \leq \eta$, $\ell_{z_t}(\wstar) \geq 0$} \\
    &\leq \| \w_{t} - \wstar \|^2 - \gamma_t (\ell_{z_t}(\w_t) - \ell_{z_t}(\wstar)) + \eta \varepsilon, \eqcomment{definition of $\varepsilon$, $\eta \geq 0$} \\
    &= \| \w_{t} - \wstar \|^2 - \dfrac{\ell_{z_t}(\w_t)}{\| \nabla \ell_{z_t}(\w_t) \|^2 + \delta} (\ell_{z_t}(\w_t) - \ell_{z_t}(\wstar)) + \eta \varepsilon, \\
    &\quad \eqcomment{because $\gamma_t = \frac{\ell_{z_t}(\w_t)}{\| \nabla \ell_{z_t}(\w_t) \|^2 + \delta}$} \\
    &\leq \| \w_{t} - \wstar \|^2 - \dfrac{(\ell_{z_t}(\w_t) - \ell_{z_t}(\wstar))^2}{\| \nabla \ell_{z_t}(\w_t) \|^2 + \delta} + \eta \varepsilon, \\
    &\quad \eqcomment{because $\ell_{z_t}(\w_t) \geq \ell_{z_t}(\w_t) - \ell_{z_t}(\wstar) \geq 0$} \\
    &\leq \| \w_{t} - \wstar \|^2 - \dfrac{(\ell_{z_t}(\w_t) - \ell_{z_t}(\wstar))^2}{C^2 + \delta} + \eta \varepsilon, \eqcomment{$\| \nabla \ell_{z_t}(\w_t) \|^2 \leq C^2$} \\
\end{split}
\end{equation}
\end{leftlinebox}
\end{leftlinebox}
\begin{leftlinebox}[nobreak=false]
We now introduce $\mathcal{I}_T$ and $\mathcal{J}_T$ as follows:
\begin{equation}
\begin{split}
    \mathcal{I}_T &\triangleq \left\{ t \in \{0, ..., T\} : \ell_{z_t}(\w_t) - \ell_{z_t}(\wstar) < 0 \right\} \\
    \mathcal{J}_T &\triangleq \left\{ t \in \{0, ..., T\} : \gamma_t = \frac{\ell_{z_t}(\w_t)}{\| \nabla \ell_{z_t}(\w_t) \|^2 + \delta}\ \text{and} \ \ell_{z_t}(\w_t) - \ell_{z_t}(\wstar) \geq 0 \right\} \\
    \mathcal{K}_T &\triangleq \{0, ..., T\} \ \backslash \ \mathcal{I}_T \cup \mathcal{J}_T
\end{split}
\end{equation}
Then, by combining inequalities (\ref{eq:iterate_bound_alig_cvx_small_eta_1}), (\ref{eq:iterate_bound_alig_cvx_small_eta_2}) and (\ref{eq:iterate_bound_alig_cvx_small_eta_3}), and using a telescopic sum, we obtain:
\begin{equation}
\begin{split}
\| \w_{T+1} - \wstar \|^2
    &\leq \| \w_{0} - \wstar \|^2 + |\mathcal{I}_T| 2\eta \varepsilon + \sum\limits_{t \in \mathcal{J}_T} \left( -\dfrac{(\ell_{z_t}(\w_t) - \ell_{z_t}(\wstar))^2}{C^2 + \delta} + \eta \varepsilon\right) \\
    &\qquad + \sum\limits_{t \in \mathcal{K}_T} \left( -\eta (\ell_{z_t}(\w_t) - \ell_{z_t}(\wstar)) + \eta \varepsilon \right)
\end{split}
\end{equation}
Using $\| \w_{T+1} - \wstar \|^2 \geq 0$, we obtain:
\begin{equation}
\begin{split}
&\kern-1em \dfrac{1}{C^2 + \delta} \sum\limits_{t \in \mathcal{J}_T} (\ell_{z_t}(\w_t) - \ell_{z_t}(\wstar))^2 + \eta \sum\limits_{t \in \mathcal{K}_T} (\ell_{z_t}(\w_t) - \ell_{z_t}(\wstar)) \\
    &\leq \| \w_{0} - \wstar \|^2 + (T + 1) 2 \eta \varepsilon
\end{split}
\end{equation}
\end{leftlinebox}
\begin{leftlinebox}[nobreak=false]
Since $\forall t \in \mathcal{K}_t, \ell_{z_t}(\w_t) - \ell_{z_t}(\wstar) \geq 0$, both LHS terms are non-negative and thus each of them is smaller or equal to the RHS:
\begin{equation} \label{eq:alig_cvx_small_eta_sum_normal}
\eta \sum\limits_{t \in \mathcal{K}_T} (\ell_{z_t}(\w_t) - \ell_{z_t}(\wstar))
    \leq \| \w_{0} - \wstar \|^2 + 2 (T + 1) \eta \varepsilon,
\end{equation}
and:
\begin{equation}
\dfrac{1}{C^2 + \delta} \sum\limits_{t \in \mathcal{J}_T} (\ell_{z_t}(\w_t) - \ell_{z_t}(\wstar))^2
    \leq \| \w_{0} - \wstar \|^2 + 2 (T + 1) \eta \varepsilon.
\end{equation}
Using the Cauchy-Schwarz inequality, we can further write:
\begin{equation}
\left( \sum\limits_{t \in \mathcal{J}_T} \ell_{z_t}(\w_t) - \ell_{z_t}(\wstar) \right)^2
    \leq |\mathcal{J}_T| \sum\limits_{t \in \mathcal{J}_T} (\ell_{z_t}(\w_t) - \ell_{z_t}(\wstar))^2.
\end{equation}
Therefore we have:
\begin{equation} \label{eq:alig_cvx_small_eta_sum_squares}
\begin{split}
\sum\limits_{t \in \mathcal{J}_T} \ell_{z_t}(\w_t) - \ell_{z_t}(\wstar)
    &\leq \sqrt{|\mathcal{J}_T| \sum\limits_{t \in \mathcal{J}_T} (\ell_{z_t}(\w_t) - \ell_{z_t}(\wstar))^2}, \\
    &\leq \sqrt{|\mathcal{J}_T| (C^2 + \delta) \left( \| \w_{0} - \wstar \|^2 + 2 (T + 1) \eta \varepsilon \right)}. \\
\end{split}
\end{equation}
\end{leftlinebox}
\begin{leftlinebox}[nobreak=false]
We can now put together inequalities (\ref{eq:alig_cvx_small_eta_sum_normal}) and (\ref{eq:alig_cvx_small_eta_sum_squares}) by writing:
\begin{equation}
\begin{split}
\kern-1em \sum\limits_{t=0}^T \ell_{z_t}(\w_t) - \ell_{z_t}(\wstar)
    &\leq \sum\limits_{t \in \mathcal{K}_T} \ell_{z_t}(\w_t) - \ell_{z_t}(\wstar) + \sum\limits_{t \in \mathcal{J}_T} \ell_{z_t}(\w_t) - \ell_{z_t}(\wstar), \eqcomment{only negative contributions in $\mathcal{I}_t$}\\
    &\leq \dfrac{1}{\eta} \left( \| \w_{0} - \wstar \|^2 + 2 (T + 1) \eta \varepsilon \right) + \sqrt{|\mathcal{J}_T| (C^2 + \delta) \left( \| \w_{0} - \wstar \|^2 + 2 (T + 1) \eta \varepsilon \right)}, \\
    &\leq \dfrac{1}{\eta} \left( \| \w_{0} - \wstar \|^2 + 2 (T + 1) \eta \varepsilon \right) + \sqrt{(T + 1) (C^2 + \delta) \left( \| \w_{0} - \wstar \|^2 + 2 (T + 1) \eta \varepsilon \right)}.
\end{split}
\end{equation}
\end{leftlinebox}
\begin{leftlinebox}[nobreak=false]
Dividing by $T+1$ and taking the expectation, we obtain:
\begin{equation}
\begin{split}
&\kern-1em \E \left[f\left(\dfrac{1}{T+1} \sum\limits_{t=0}^T \w_t \right) \right]- \fstar \\
    &\leq \dfrac{1}{T+1} \sum\limits_{t=0}^T \E[f(\w_t)] - \fstar, \eqcomment{$f$ is convex} \\
    &\leq \dfrac{\| \w_{0} - \wstar \|^2}{\eta (T + 1)} + 2 \varepsilon + \sqrt{(C^2 + \delta) \left( \dfrac{\| \w_{0} - \wstar \|^2}{T + 1} + 2 \eta \varepsilon \right)}, \\
    &\leq \dfrac{\| \w_{0} - \wstar \|^2}{\eta (T + 1)} + 2 \varepsilon + \sqrt{ \dfrac{(C^2 + \delta) \| \w_{0} - \wstar \|^2}{T + 1}} + 2 \eta \varepsilon \sqrt{C^2 + \delta}.
\end{split}
\end{equation}
\end{leftlinebox}
\end{proof}

\subsection{Smooth Convex Functions}
We now tackle the convex and $\beta$-smooth case.
Our proof techniques naturally produce the separation $\eta \geq \frac{1}{2 \beta}$ and $\eta \leq \frac{1}{2 \beta}$.

\begin{lemma} \label{lemma:smooth_bound}
    Let $z \in \Z$.
    Assume that $\ell_{z}$ is $\beta$-smooth and non-negative on $\mathbb{R}^p$.
    Then we have:
    \begin{equation} \label{eq:smooth_inequality}
        \forall \: \w \in \mathbb{R}^p, \ \ell_{z}(\w) \geq \frac{1}{2 \beta} \| \nabla \ell_{z}(\w) \|^2
    \end{equation}
    Note that we do not assume that $\ell_z$ is convex.
\end{lemma}

\begin{proof} \
\begin{leftlinebox}[nobreak=false]
Let $\vw \in \mathbb{R}^p$. By Lemma 3.4 of \citep{Bubeck2015}, we have:
\begin{equation}
\forall \: \vu \in \mathbb{R}^p, \: | \ell_z(\vu) - \ell_z(\vw) - \nabla \ell_z(\vw)^\top (\vu - \vw)| \leq \dfrac{\beta}{2} \| \vu - \vw \|^2.
\end{equation}
Therefore we can write:
\begin{equation}
\forall \: \vu \in \mathbb{R}^p, \: \ell_z(\vu) \leq \ell_z(\vw) + \nabla \ell_z(\vw)^\top (\vu - \vw) + \dfrac{\beta}{2} \| \vu - \vw \|^2.
\end{equation}
And since $\forall \: \vu \in \mathbb{R}^p, \: \ell_z(\vu) \geq 0$, we have:
\begin{equation}
\forall \: \vu \in \mathbb{R}^p, \: 0 \leq \ell_z(\vw) + \nabla \ell_z(\vw)^\top (\vu - \vw) + \dfrac{\beta}{2} \| \vu - \vw \|^2.
\end{equation}
We now choose $\vu = \vw - \dfrac{1}{\beta} \nabla \ell_z(\vw)$, which yields:
\begin{equation}
0 \leq \ell_z(\vw) - \dfrac{1}{\beta} \| \nabla \ell_z(\vw)\|^2 + \dfrac{1}{2 \beta} \| \nabla \ell_z(\vw) \|^2,
\end{equation}
which gives the desired result.
\end{leftlinebox}
\end{proof}

\begin{lemma} \label{lemma:smooth_gamma_bound}
    Let $z \in \Z$.
    Assume that $\ell_{z}$ is $\beta$-smooth and non-negative on $\mathbb{R}^p$.
    Then we have:
    \begin{equation}
        \forall \: \w \in \mathbb{R}^p, \ \dfrac{\ell_{z}(\w)}{\| \nabla \ell_{z}(\w) \|^2 + \delta} \geq \dfrac{1}{2 \beta} -  \dfrac{\delta}{ 4 \beta^2 \ell_{z}(\w)}
    \end{equation}
\end{lemma}

\begin{proof} \
\begin{leftlinebox}[nobreak=false]
Let $\w \in \mathbb{R}^p$.
We apply Lemma \ref{lemma:smooth_bound} and we write successively:

\begin{equation}
\begin{split}
\dfrac{\ell_{z}(\w)}{\| \nabla \ell_{z}(\w) \|^2 + \delta}
    &\geq \dfrac{\ell_{z}(\w)}{ 2 \beta \ell_{z}(\w) + \delta}, \eqcomment{Lemma \ref{lemma:smooth_bound}} \\
    &= \dfrac{\ell_{z}(\w) + \frac{\delta}{2 \beta} - \frac{\delta}{2 \beta}}{ 2 \beta (\ell_{z}(\w) + \frac{\delta}{2 \beta})}, \\
    &= \dfrac{1}{2 \beta} - \dfrac{\frac{\delta}{2 \beta}}{ 2 \beta (\ell_{z}(\w) + \frac{\delta}{2 \beta})}, \\
    &\geq \dfrac{1}{2 \beta} - \dfrac{\delta}{4 \beta^2 \ell_{z}(\w)}. \eqcomment{$\delta \geq 0$} \\
\end{split}
\end{equation}
\end{leftlinebox}
\end{proof}

\begin{restatable}{theorem}{thaligcvxsmoothlargeeta}\label{th:alig_cvx_smooth_large_eta}
We assume that $\Omega$ is a convex set, and that for every $z \in \Z$, $\ell_z$ is convex and $\beta$-smooth.
Let $\wstar$ be a solution of (\ref{eq:main_problem}) such that $\forall z \in \Z, \: \ell_z(\wstar) \leq \varepsilon$, and suppose that $\delta > 2 \beta \varepsilon$.
Further assume that $\eta \geq \frac{1}{2 \beta}$.
Then if we apply ALI-G with a maximal learning-rate of $\eta$ to $f$, we have:
\begin{equation}
\E \left[f\left(\frac{1}{T+ 1} \sum\limits_{t=0}^T \w_t \right) \right] - \fstar
    \leq \dfrac{\delta}{\beta(1 - \frac{2 \beta \varepsilon}{\delta})} + \dfrac{2 \beta}{1 - \frac{2 \beta \varepsilon}{\delta}} \dfrac{\|\w_{0} - \wstar\|^2}{T + 1}.
\end{equation}
\end{restatable}

\begin{proof} \
\begin{leftlinebox}[nobreak=false]
We re-use the inequality (\ref{eq:alig_cvx_basic_iterate_bound}) from the proof of Theorem \ref{th:alig_cvx_large_eta}:
\begin{equation} \label{eq:smooth_iterate_largeeta_bound}
\| \w_{t+1} - \wstar \|^2 \leq \| \w_{t} - \wstar \|^2 - \gamma_t (\ell_{z_t}(\w_t) - \ell_{z_t}(\wstar)) + \gamma_t \ell_{z_t}(\wstar)
\end{equation}
\end{leftlinebox}
As previously, we lower bound $\gamma_t (\ell_{z_t}(\w_t) - \ell_{z_t}(\wstar))$ and upper bound $\gamma_t \ell_{z_t}(\wstar)$ individually.

\begin{leftlinebox}[nobreak=false]
We begin with $\gamma_t (\ell_{z_t}(\w_t) - \ell_{z_t}(\wstar))$.
We remark that either $\gamma_t = \frac{\ell_{z_t}(\w_t)}{\| \nabla \ell_{z_t}(\w_t) \|^2 + \delta}$ or $\gamma_t = \eta$.
    \begin{leftlinebox}[nobreak=false]
    Suppose $\ell_{z_t}(\w_t) - \ell_{z_t}(\wstar) \geq 0$ and $\gamma_t = \frac{\ell_{z_t}(\w_t)}{\| \nabla \ell_{z_t}(\w_t) \|^2 + \delta}$.
    Then we can write:
    \begin{equation}
    \begin{split}
    &\kern-1em \gamma_t (\ell_{z_t}(\w_t) - \ell_{z_t}(\wstar)) \\
    & = \dfrac{\ell_{z_t}(\w_t)}{\| \nabla \ell_{z_t}(\w_t) \|^2 + \delta} (\ell_{z_t}(\w_t) - \ell_{z_t}(\wstar)), \eqcomment{definition of $\gamma_t$} \\
    & \geq \left( \dfrac{1}{2 \beta} - \dfrac{\delta}{4 \beta^2 \ell_{z}(\w_t)} \right) (\ell_{z_t}(\w_t) - \ell_{z_t}(\wstar)) \\
    &\quad \eqcomment{using Lemma \ref{lemma:smooth_gamma_bound}, $\ell_{z_t}(\w_t) - \ell_{z_t}(\wstar) \geq 0$} \\
    & = \dfrac{1}{2 \beta} (\ell_{z_t}(\w_t) - \ell_{z_t}(\wstar)) - \dfrac{\delta}{4 \beta^2} \dfrac{\ell_{z_t}(\w_t) - \ell_{z_t}(\wstar)}{\ell_{z_t}(\w_t)}\\
    & \geq \dfrac{1}{2 \beta} (\ell_{z_t}(\w_t) - \ell_{z_t}(\wstar)) - \dfrac{\delta}{4 \beta^2} \eqcomment{$\ell_{z_t}(\wstar) \geq 0$, $\ell_{z_t}(\w_t) \geq 0$}\\
    \end{split}
    \end{equation}
    \end{leftlinebox}

    \begin{leftlinebox}[nobreak=false]
    Now suppose $\ell_{z_t}(\w_t) - \ell_{z_t}(\wstar) \geq 0$ and $\gamma_t = \eta$.
    Then we have:
    \begin{align*}
    \gamma_t (\ell_{z_t}(\w_t) - \ell_{z_t}(\wstar))
        &= \eta (\ell_{z_t}(\w_t) - \ell_{z_t}(\wstar)) \\
        &\geq \eta (\ell_{z_t}(\w_t) - \ell_{z_t}(\wstar)) - \dfrac{\delta}{4 \beta^2} \\
        &\geq \frac{1}{2 \beta} (\ell_{z_t}(\w_t) - \ell_{z_t}(\wstar)) - \dfrac{\delta}{4 \beta^2} \\
        &\quad \eqcomment{because $\eta \geq \frac{1}{2 \beta}$, $\ell_{z_t}(\w_t) - \ell_{z_t}(\wstar) \geq 0$}.
    \stepcounter{equation}\tag{\theequation}
    \end{align*}
    \end{leftlinebox}

    \begin{leftlinebox}[nobreak=false]
    Now suppose $\ell_{z_t}(\w_t) - \ell_{z_t}(\wstar) \leq 0$.
    We have:
    \begin{equation} \label{eq:smooth_gamma_t_bound}
    \begin{split}
    \gamma_t
    &\leq \dfrac{\ell_{z_t}(\w_t)}{\| \nabla \ell_{z_t}(\w_t) \|^2 + \delta} \\
    &\leq \dfrac{\ell_{z_t}(\wstar)}{\| \nabla \ell_{z_t}(\w_t) \|^2 + \delta} \eqcomment{$\ell_{z_t}(\w_t) - \ell_{z_t}(\wstar) \leq 0$}\\
    &\leq \dfrac{\varepsilon}{\| \nabla \ell_{z_t}(\w_t) \|^2 + \delta} \eqcomment{definition of $\varepsilon$} \\
    &\leq \dfrac{\varepsilon}{\delta} \eqcomment{$\| \nabla \ell_{z_t}(\w_t) \| \geq 0$} \\
    &\leq \dfrac{1}{2 \beta} \eqcomment{$\delta \geq 2 \beta \varepsilon$} \\
    \end{split}
    \end{equation}
    We now write:
    \begin{equation}
    \begin{split}
    \gamma_t \left( \ell_{z_t}(\w_t) - \ell_{z_t}(\wstar) \right)
        &\geq \frac{1}{2 \beta} (\ell_{z_t}(\w_t) - \ell_{z_t}(\wstar)) \eqcomment{$\ell_{z_t}(\w_t) - \ell_{z_t}(\wstar) \leq 0$, $\gamma_t \leq \frac{1}{2 \beta}$} \\
        &\geq \frac{1}{2 \beta} (\ell_{z_t}(\w_t) - \ell_{z_t}(\wstar)) - \dfrac{\delta}{4 \beta^2} \\
    \end{split}
    \end{equation}
    \end{leftlinebox}

In conclusion, in all cases, it holds true that:
\begin{equation} \label{eq:smooth_largeeta_bound_first_term}
\gamma_t (\ell_{z_t}(\w_t) - \ell_{z_t}(\wstar)) \geq \frac{1}{2 \beta} (\ell_{z_t}(\w_t) - \ell_{z_t}(\wstar)) - \dfrac{\delta}{4 \beta^2}
\end{equation}

\begin{leftlinebox}[nobreak=false]
We now upper bound $\gamma_t \ell_{z_t}(\wstar)$:
\begin{equation} \label{eq:smooth_largeeta_bound_second_term}
\begin{split}
\gamma_t \ell_{z_t}(\wstar)
    &\leq \dfrac{\ell_{z_t}(\w_t) \ell_{z_t}(\wstar)}{\| \nabla \ell_{z_t}(\w_t) \|^2 + \delta}, \eqcomment{definition of $\gamma_t$ and $\ell_{z_t}(\wstar) \geq 0$} \\
    &\leq \dfrac{\ell_{z_t}(\w_t) \ell_{z_t}(\wstar)}{\delta}, \eqcomment{$\| \nabla \ell_{z_t}(\w_t) \| \geq 0$} \\
    &\leq \dfrac{(\ell_{z_t}(\w_t) - \ell_{z_t}(\wstar) + \varepsilon)\varepsilon}{\delta}, \eqcomment{definition of $\varepsilon$ twice} \\
    &= \dfrac{\varepsilon}{\delta} (\ell_{z_t}(\w_t) - \ell_{z_t}(\wstar)) + \dfrac{\varepsilon^2}{\delta}.
\end{split}
\end{equation}
\end{leftlinebox}

\end{leftlinebox}
\begin{leftlinebox}[nobreak=false]
We now put together inequalities (\ref{eq:smooth_iterate_largeeta_bound}), (\ref{eq:smooth_largeeta_bound_first_term}) and (\ref{eq:smooth_largeeta_bound_second_term}):
\begin{equation}
\begin{split}
&\kern-1em \| \w_{t+1} - \wstar \|^2 \\
    &\leq  \| \w_{t} - \wstar \|^2 - \frac{1}{2 \beta} (\ell_{z_t}(\w_t) - \ell_{z_t}(\wstar)) + \dfrac{\delta}{4 \beta^2} +  \dfrac{\varepsilon}{\delta} (\ell_{z_t}(\w_t) - \ell_{z_t}(\wstar)) + \dfrac{\varepsilon^2}{\delta}, \\
    &=  \| \w_{t} - \wstar \|^2 - \left(\frac{1}{2 \beta} - \dfrac{\varepsilon}{\delta} \right) (\ell_{z_t}(\w_t) - \ell_{z_t}(\wstar)) + \dfrac{\delta}{4 \beta^2} + \dfrac{\varepsilon^2}{\delta}.
\end{split}
\end{equation}
Therefore we have:
    \begin{equation} \label{eq:smooth_telescopic_sum}
    \left( \dfrac{1}{2 \beta} - \dfrac{\varepsilon}{\delta} \right) \left( \ell_{z_t}(\w_t) - \ell_{z_t}(\wstar) \right) - \left(\dfrac{\delta}{4 \beta^2} + \dfrac{\varepsilon^2}{\delta} \right) \leq  \|\w_{t} - \wstar\|^2 - \|\w_{t+1} - \wstar\|^2.
    \end{equation}
    \end{leftlinebox}

    \begin{leftlinebox}[nobreak=false]
    By summing over $t$ and taking the expectation over every $z_t$, we obtain:
    \begin{equation}
    \begin{split}
    &\kern-1em  \sum\limits_{t=0}^{T} \left( \dfrac{\delta - 2 \beta \varepsilon}{2 \beta \delta} \left( \E[f(\w_t)] - f(\wstar) \right) - \dfrac{\delta^2 + 4 \beta^2 \varepsilon^2}{4 \beta^2 \delta}  \right) \\
        &\leq \|\w_{0} - \wstar\|^2 - \mathbb{E}\left[\|\w_{T+1} - \wstar\|^2\right], \\
        &\leq \|\w_{0} - \wstar\|^2.
    \end{split}
    \end{equation}
    By assumption, we have that $\delta - 2 \beta \varepsilon > 0$.
    Dividing by $T+1$ and using the convexity of $f$, we finally obtain:
    \begin{equation}
    \begin{split}
    \E \left[f\left(\frac{1}{T+ 1} \sum\limits_{t=0}^T \w_t \right)\right] - \fstar
        &\leq \frac{1}{T+ 1} \sum\limits_{t=0}^T \E[f(\w_t)] - \fstar \eqcomment{convexity of $f$}, \\
        &= \dfrac{2 \beta \delta}{\delta - 2 \beta \varepsilon} \dfrac{\delta^2 + 4 \beta^2 \varepsilon^2}{4 \beta^2 \delta} + \dfrac{2 \beta \delta}{\delta - 2 \beta \varepsilon} \dfrac{\|\w_{0} - \wstar\|^2}{T + 1}, \\
        &= \dfrac{\delta^2 + 4 \beta^2 \varepsilon^2}{2 \beta(\delta - 2 \beta \varepsilon)} + \dfrac{2 \beta \delta}{\delta - 2 \beta \varepsilon} \dfrac{\|\w_{0} - \wstar\|^2}{T + 1}, \\
        &\leq \dfrac{\delta^2}{\beta(\delta - 2 \beta \varepsilon)} + \dfrac{2 \beta \delta}{\delta - 2 \beta \varepsilon} \dfrac{\|\w_{0} - \wstar\|^2}{T + 1}, \eqcomment{$\delta - 2 \beta \varepsilon \geq 0$} \\
        &= \dfrac{\delta}{\beta(1 - \frac{2 \beta \varepsilon}{\delta})} + \dfrac{2 \beta}{1 - \frac{2 \beta \varepsilon}{\delta}} \dfrac{\|\w_{0} - \wstar\|^2}{T + 1}. \\
    \end{split}
    \end{equation}
    \end{leftlinebox}
\end{proof}

\begin{restatable}{theorem}{thaligcvxsmoothsmalleta}\label{th:alig_cvx_smooth_small_eta}
We assume that $\Omega$ is a convex set, and that for every $z \in \Z$, $\ell_z$ is convex and $\beta$-smooth.
Let $\wstar$ be a solution of (\ref{eq:main_problem}) such that $\forall z \in \Z, \: \ell_z(\wstar) \leq \varepsilon$, and suppose that $\delta > 2 \beta \varepsilon$.
Further assume that $\eta \leq \frac{1}{2 \beta}$.
Then if we apply ALI-G with a maximal learning-rate of $\eta$ to $f$, we have:
\begin{equation}
\E \left[f\left(\dfrac{1}{T+1} \sum\limits_{t=0}^T \w_t \right)\right]  - \fstar \leq \dfrac{\| \w_{0} - \wstar \|^2}{\eta (T+1)} + \dfrac{\delta}{2 \beta} + \varepsilon.
\end{equation}
\end{restatable}

\begin{proof} \
\begin{leftlinebox}[nobreak=false]
Similarly to the beginning of previous proofs, we have that:
\begin{equation} \label{eq:smooth_iterate_smalleta_bound}
\| \w_{t+1} - \wstar \|^2 \leq \| \w_{t} - \wstar \|^2 - \gamma_t (\ell_{z_t}(\w_t) - \ell_{z_t}(\wstar)) + \gamma_t \ell_{z_t}(\wstar)
\end{equation}
\end{leftlinebox}
As previously, we lower bound $\gamma_t (\ell_{z_t}(\w_t) - \ell_{z_t}(\wstar))$ and upper bound $\gamma_t \ell_{z_t}(\wstar)$ individually.

\begin{leftlinebox}[nobreak=false]
We begin with $\gamma_t (\ell_{z_t}(\w_t) - \ell_{z_t}(\wstar))$.
We remark that either $\gamma_t = \frac{\ell_{z_t}(\w_t)}{\| \nabla \ell_{z_t}(\w_t) \|^2 + \delta}$ or $\gamma_t = \eta$.
    \begin{leftlinebox}[nobreak=false]
    Suppose $\gamma_t = \frac{\ell_{z_t}(\w_t)}{\| \nabla \ell_{z_t}(\w_t) \|^2 + \delta}$ and $\ell_{z_t}(\w_t) - \ell_{z_t}(\wstar) \geq 0$.
    First we write:
    \begin{equation}
    \begin{split}
    \gamma_t
        &= \dfrac{\ell_{z_t}(\w_t)}{\| \nabla \ell_{z_t}(\w_t) \|^2 + \delta} \\
        &= \dfrac{\ell_{z_t}(\w_t) + \frac{\delta}{2 \beta}}{\| \nabla \ell_{z_t}(\w_t) \|^2 + \delta} - \dfrac{\frac{\delta}{2 \beta}}{\| \nabla \ell_{z_t}(\w_t) \|^2 + \delta}\\
        &\geq \dfrac{\frac{\| \nabla \ell_{z_t}(\w_t) \|^2}{2 \beta} + \frac{\delta}{2 \beta}}{\| \nabla \ell_{z_t}(\w_t) \|^2 + \delta} - \dfrac{\delta}{2 \beta} \dfrac{1}{\| \nabla \ell_{z_t}(\w_t) \|^2 + \delta} \eqcomment{Lemma \ref{lemma:smooth_bound}}\\
        &= \dfrac{1}{2 \beta} - \dfrac{\delta}{2 \beta} \dfrac{1}{\| \nabla \ell_{z_t}(\w_t) \|^2 + \delta} \\
        &\geq \eta - \dfrac{\delta}{2 \beta} \dfrac{1}{\| \nabla \ell_{z_t}(\w_t) \|^2 + \delta} \eqcomment{$\eta \leq \frac{1}{2 \beta}$}\\
    \end{split}
    \end{equation}
    Since $\ell_{z_t}(\w_t) - \ell_{z_t}(\wstar) \geq 0$, this yields:
    \begin{equation}
    \begin{split}
    \gamma_t (\ell_{z_t}(\w_t) - \ell_{z_t}(\wstar))
        &\geq \left(\eta - \dfrac{\delta}{2 \beta} \dfrac{1}{\| \nabla \ell_{z_t}(\w_t) \|^2 + \delta} \right) (\ell_{z_t}(\w_t) - \ell_{z_t}(\wstar)) \\
        &= \eta (\ell_{z_t}(\w_t) - \ell_{z_t}(\wstar)) - \dfrac{\delta}{2 \beta} \dfrac{\ell_{z_t}(\w_t) - \ell_{z_t}(\wstar)}{\| \nabla \ell_{z_t}(\w_t) \|^2 + \delta}  \\
        &\geq \eta (\ell_{z_t}(\w_t) - \ell_{z_t}(\wstar)) - \dfrac{\delta}{2 \beta} \dfrac{\ell_{z_t}(\w_t)}{\|\nabla  \ell_{z_t}(\w_t) \|^2 + \delta} \\
        &\quad \eqcomment{because $\ell_{z_t}(\wstar) \geq 0$} \\
    \end{split}
    \end{equation}
    We now notice that since $\gamma_t = \frac{\ell_{z_t}(\w_t)}{\| \nabla \ell_{z_t}(\w_t) \|^2 + \delta}$, and $\gamma_t \leq \eta$, then necessarily $\frac{\ell_{z_t}(\w_t)}{\| \nabla \ell_{z_t}(\w_t) \|^2 + \delta} \leq \eta$.
    This gives:
    \begin{equation}
    \gamma_t (\ell_{z_t}(\w_t) - \ell_{z_t}(\wstar))
        \geq \eta (\ell_{z_t}(\w_t) - \ell_{z_t}(\wstar)) - \dfrac{\eta \delta}{2 \beta}
    \end{equation}
    \end{leftlinebox}

    \begin{leftlinebox}[nobreak=false]
    Now suppose $\gamma_t = \eta$ and $\ell_{z_t}(\w_t) - \ell_{z_t}(\wstar) \geq 0$.
    Then we have:
    \begin{align*}
    \gamma_t (\ell_{z_t}(\w_t) - \ell_{z_t}(\wstar))
        &= \eta (\ell_{z_t}(\w_t) - \ell_{z_t}(\wstar)) \\
        &\geq \eta (\ell_{z_t}(\w_t) - \ell_{z_t}(\wstar)) - \dfrac{\eta \delta}{2 \beta}.
    \stepcounter{equation}\tag{\theequation}
    \end{align*}
    \end{leftlinebox}
    \begin{leftlinebox}[nobreak=false]
    Now suppose $\ell_{z_t}(\w_t) - \ell_{z_t}(\wstar) \leq 0$.
    Since $\gamma_t \leq \eta$ by definition, we have that:
    \begin{equation}
    \begin{split}
    \gamma_t \left( \ell_{z_t}(\w_t) - \ell_{z_t}(\wstar) \right)
        &\geq \eta (\ell_{z_t}(\w_t) - \ell_{z_t}(\wstar)) \eqcomment{$\ell_{z_t}(\w_t) - \ell_{z_t}(\wstar) \leq 0$} \\
        &\geq \eta (\ell_{z_t}(\w_t) - \ell_{z_t}(\wstar)) - \dfrac{\eta \delta}{2 \beta}.
    \end{split}
    \end{equation}

    \end{leftlinebox}

In conclusion, in all cases, it holds true that:
\begin{equation} \label{eq:smooth_smalleta_bound_first_term}
\gamma_t (\ell_{z_t}(\w_t) - \ell_{z_t}(\wstar)) \geq \eta (\ell_{z_t}(\w_t) - \ell_{z_t}(\wstar)) - \dfrac{\eta \delta}{2 \beta}
\end{equation}

We upper bound $\gamma_t \ell_{z_t}(\wstar)$ as follows:
\begin{equation} \label{eq:smooth_smalleta_bound_second_term}
\begin{split}
\gamma_t \ell_{z_t}(\wstar)
    &\leq \eta \ell_{z_t}(\wstar) \eqcomment{$\ell_{z_t}(\wstar) \geq 0$} \\
    &\leq \eta \varepsilon \eqcomment{definition of $\varepsilon$}
\end{split}
\end{equation}
We combine inequalities (\ref{eq:smooth_iterate_smalleta_bound}), (\ref{eq:smooth_smalleta_bound_first_term}) and (\ref{eq:smooth_smalleta_bound_second_term}) and obtain:
\begin{equation}
    \| \w_{t+1} - \wstar \|^2 \leq  \| \w_{t} - \wstar \|^2 - \eta (\ell_{z_t}(\w_t) - \ell_{z_t}(\wstar)) + \dfrac{\eta \delta}{2 \beta} + \eta \varepsilon.
\end{equation}
By taking the expectation and using a telescopic sum, we obtain:
\begin{equation}
    0 \leq \| \w_{T+1} - \wstar \|^2 \leq  \| \w_{0} - \wstar \|^2 - \sum\limits_{t=0}^T \left( \eta (\E[f(\w_t)] - \fstar) + \dfrac{\eta \delta}{2 \beta} + \eta \varepsilon \right).
\end{equation}
Re-arranging and using the convexity of $f$, we finally obtain:
\begin{equation}
    \E \left[f\left(\dfrac{1}{T+1} \sum\limits_{t=0}^T \w_t \right)\right] - \fstar \leq \dfrac{\| \w_{0} - \wstar \|^2}{\eta (T+1)} + \dfrac{\delta}{2 \beta} + \varepsilon.
\end{equation}
\end{leftlinebox}
\end{proof}

\subsection{Smooth and Strongly Convex Functions}

Finally, we consider the $\alpha$-strongly convex and $\beta$-smooth case.
Again, our proof yields a natural separation between $\eta \geq \frac{1}{2 \beta}$ and $\eta \leq \frac{1}{2 \beta}$.

\begin{lemma} \label{lemma:strglycvx_gamma_bound}
    Let $z \in \Z$.
    Assume that $\ell_{z}$ is $\alpha$-strongly convex, non-negative on $\mathbb{R}^p$, and such that $ \inf \ell_{z} \leq \varepsilon$.
    In addition, suppose that $\delta \geq 2 \alpha \varepsilon$.
    Then we have:
    \begin{equation}
        \forall \: \w \in \mathbb{R}^p, \ \dfrac{\ell_{z}(\w)}{\| \nabla \ell_{z}(\w) \|^2 + \delta} \leq \dfrac{1}{2 \alpha}.
    \end{equation}
\end{lemma}

\begin{proof} \
\begin{leftlinebox}[nobreak=false]
Let \(\w \in \mathbb{R}^p\) and suppose that $\ell_{z}$ reaches its minimum at $\wbar \in \mathbb{R}^p$ (this minimum exists because of strong convexity).
By definition of strong convexity, we have that:
\begin{equation}
    \forall \ \hat{\w} \in \mathbb{R}^p, \ \ell_{z}(\hat{\w}) \geq \ell_{z}(\w) + \nabla \ell_{z}(\w)^\top (\hat{\w} - \w) + \dfrac{\alpha}{2} \| \hat{\w} - \w \|^2
\end{equation}
We minimize the right hand-side over $\hat{\w}$, which gives:
\begin{equation}
\begin{split}
\forall \hat{\w} \in \mathbb{R}^p, \ \ell_{z}(\hat{\w})
    &\geq \ell_{z}(\w) + \nabla \ell_{z}(\w)^\top (\hat{\w} - \w) + \dfrac{\alpha}{2} \| \hat{\w} - \w \|^2 \\
    &\geq \ell_{z}(\w)  - \dfrac{1}{2 \alpha} \| \nabla \ell_{z}(\w) \|^2
\end{split}
\end{equation}
Thus by choosing $\hat{\w} = \wbar$ and re-ordering, we obtain the following result (a.k.a. the Polyak-Lojasiewicz inequality):
\begin{equation}
    \ell_{z}(\w) - \ell_{z}(\wbar) \leq \dfrac{1}{2 \alpha} \| \nabla \ell_{z}(\w) \|^2
\end{equation}
\end{leftlinebox}
Therefore we can write:
\begin{leftlinebox}[nobreak=false]
\begin{equation}
    \dfrac{\ell_{z}(\w)}{\| \nabla \ell_{z}(\w) \|^2 + \delta} \leq \dfrac{\ell_{z}(\w) - \ell_{z}(\wbar) + \varepsilon}{\| \nabla \ell_{z}(\w) \|^2 + \delta} \leq \dfrac{\frac{1}{2 \alpha} \| \nabla \ell_{z}(\w)\|^2 + \varepsilon}{\| \nabla \ell_{z}(\w) \|^2 + \delta}.
\end{equation}
We introduce the function \(\psi: x \in \mathbb{R}^+ \mapsto \dfrac{\frac{1}{2 \alpha} x + \varepsilon}{x + \delta} \), and we compute its derivative:
\begin{equation}
\begin{split}
\psi'(x)
    &= \dfrac{\frac{1}{2 \alpha} (x + \delta) - \frac{1}{2 \alpha} x - \varepsilon}{(x + \delta)^2}, \\
    &= \dfrac{\frac{\delta}{2 \alpha} - \varepsilon}{(x + \delta)^2} \geq 0. \eqcomment{$\delta \geq 2 \alpha \varepsilon$}
\end{split}
\end{equation}
Therefore $\psi$ is monotonically increasing.
As a result, we have:
\begin{equation}
    \forall \ x \in \mathbb{R}^+, \ \psi(x) \leq \lim\limits_{x \to \infty} \psi(x) = \dfrac{1}{2\alpha}.
\end{equation}
Therefore we have that:
\begin{equation}
    \dfrac{\frac{1}{2 \alpha} \| \nabla \ell_{z}(\w)\|^2 + \varepsilon}{\| \nabla \ell_{z}(\w) \|^2 + \delta} = \psi \left(\| \nabla \ell_{z}(\w) \|^2 \right) \leq \dfrac{1}{2 \alpha},
\end{equation}
which concludes the proof.
\end{leftlinebox}
\end{proof}

\begin{lemma} \label{lemma:parallelogram_inequality}
For any $a, b \in \mathbb{R}^p$, we have that:
\begin{equation}
    \|a \|^2 + \|b \|^2 \geq \dfrac{1}{2} \| a - b\|^2
\end{equation}
\end{lemma}

\begin{proof}
This is a simple application of the parallelogram law, but we give the proof here for completeness.
\begin{align*}
\|a \|^2 + \|b \|^2 - \dfrac{1}{2} \| a - b\|^2
    &= \|a \|^2 + \|b \|^2 - \dfrac{1}{2} \| a\|^2 -\dfrac{1}{2} \| b\|^2 + a^\top b \\
    &= \dfrac{1}{2} \| a\|^2 + \dfrac{1}{2} \| b\|^2 + a^\top b \\
    &= \dfrac{1}{2} \| a + b \|^2 \\
    &\geq 0 \\
\end{align*}
\end{proof}

\begin{lemma} \label{lemma:strglycvx_fun_bound}
    Let $z \in \Z$.
    Assume that $\ell_{z}$ is $\alpha$-strongly convex and achieves its (possibly constrained) minimum at $\wstar \in \Omega$.
    Then we have:
    \begin{equation}
        \forall \: \w \in \Omega, \ \ell_{z}(\w)  - \ell_{z}(\wstar) \geq \dfrac{\alpha}{2} \| \w - \wstar \|^2
    \end{equation}
\end{lemma}

\begin{proof}
By definition of strong-convexity \cite{Bubeck2015}, we have:
\begin{equation}
\forall \: \w \in \Omega, \: \ell_z(\w) - \ell_z(\wstar) - \nabla \ell_z(\wstar)^\top (\w - \wstar) \geq \dfrac{\alpha}{2} \| \w - \wstar \|^2.
\end{equation}
In addition, since $\wstar$ minimizes $\ell_z$, then necessarily:
\begin{equation}
    \forall \: \w \in \Omega, \: \nabla \ell_z(\wstar)^\top (\w - \wstar) \geq 0.
\end{equation}
Combining the two equations gives the desired result.
\end{proof}

\begin{restatable}{theorem}{thaligstronglycvxlargeeta}\label{th:alig_cvx_strongly_large_eta}
We assume that $\Omega$ is a convex set, and that for every $z \in \Z$, $\ell_z$ is $\alpha$-strongly convex and $\beta$-smooth.
Let $\wstar$ be a solution of (\ref{eq:main_problem}) such that $\forall z \in \Z, \: \ell_z(\wstar) \leq \varepsilon$, and suppose that $\delta > 2 \beta \varepsilon$.
Further assume that $\eta \geq \frac{1}{2 \beta}$.
Then if we apply ALI-G with a maximal learning-rate of $\eta$ to $f$, we have:
\begin{equation}
\mathbb{E}[f(\w_{T+1})] - \fstar
    \leq \beta \exp\left(- \dfrac{\alpha t}{4 \beta} \right)  \| \w_{0} - \wstar \|^2 + \dfrac{\delta}{\alpha} + 2 \dfrac{\beta}{\alpha} \varepsilon + 2 \dfrac{\beta^2}{\alpha^2} \varepsilon.
\end{equation}
\end{restatable}

\begin{proof} \
\begin{leftlinebox}[nobreak=false]
We condition the update on $z_t$ drawn at random.
The beginning of the proof is identical to that of Theorem \ref{th:alig_cvx_smooth_large_eta} (and in particular requires $\delta > 2 \beta \varepsilon$).
In addition, we remark that $\delta > 2 \beta \varepsilon \geq 2 \alpha \varepsilon$, because it always holds true that $\beta \geq \alpha$.
Combining inequalities (\ref{eq:alig_cvx_basic_iterate_bound}) and (\ref{eq:smooth_largeeta_bound_first_term}), we obtain:
\begin{align*}
\kern-1em \| \w_{t+1} - \wstar \|^2
    &\leq  \| \w_{t} - \wstar \|^2 - \dfrac{1}{2 \beta} (\ell_{z_t}(\w_t) - \ell_{z_t}(\wstar)) + \dfrac{\delta}{4 \beta^2} + \gamma_t \ell_{z_t}(\wstar), \\
    &\leq  \| \w_{t} - \wstar \|^2 - \dfrac{1}{2 \beta} (\ell_{z_t}(\w_t) - \ell_{z_t}(\wstar)) + \dfrac{\delta}{4 \beta^2} + \gamma_t \varepsilon, \eqcomment{definition of $\varepsilon$} \\
    &\leq  \| \w_{t} - \wstar \|^2 - \dfrac{1}{2 \beta} (\ell_{z_t}(\w_t) - \ell_{z_t}(\wstar)) + \dfrac{\delta}{4 \beta^2} + \dfrac{\varepsilon}{2 \alpha}. \eqcomment{Lemma \ref{lemma:strglycvx_gamma_bound}}
    \label{eq:strgly_cvx_iterate_bound} \stepcounter{equation}\tag{\theequation}
\end{align*}
\end{leftlinebox}

Taking the expectation over $z_t | z_{t-1}$, we obtain:
\begin{align*}
\kern-1em \mathbb{E}_{z_t | z_{t-1}}[\| \w_{t+1} - \wstar \|^2]
&\leq \| \w_{t} - \wstar \|^2 - \dfrac{1}{2 \beta} (f(\w_t) - f(\wstar)) + \dfrac{\delta}{4 \beta^2} + \dfrac{\varepsilon}{2 \alpha}, \\
&\leq \| \w_{t} - \wstar \|^2 - \dfrac{\alpha}{4 \beta} \| \w_t - \wstar \|^2 + \dfrac{\delta}{4 \beta^2} + \dfrac{\varepsilon}{2 \alpha}. \eqcomment{by lemma \ref{lemma:strglycvx_fun_bound}}
\end{align*}

Now taking expectation over every $z_t$, we use a trivial induction over $t$ and write:
\begin{leftlinebox}[nobreak=false]
\begin{align*}
&\kern-1em \mathbb{E} [\| \w_{t+1} - \wstar \|^2] \\
    &\leq \left( 1 - \dfrac{\alpha}{4 \beta} \right)  \mathbb{E} [\| \w_{t} - \wstar \|^2] + \dfrac{\delta}{4 \beta^2} + \dfrac{\varepsilon}{2 \alpha},\\
    &\leq \left( 1 - \dfrac{\alpha}{4 \beta} \right)^t  \| \w_{0} - \wstar \|^2 + \sum\limits_{k=0}^{t} \left(1 - \dfrac{\alpha}{4 \beta} \right)^{t -k} \left( \dfrac{\delta}{4 \beta^2} + \dfrac{\varepsilon}{2 \alpha} \right), \\
    &\leq \left( 1 - \dfrac{\alpha}{4 \beta} \right)^t  \| \w_{0} - \wstar \|^2 + \sum\limits_{k=0}^{\infty} \left(1 - \dfrac{\alpha}{4 \beta} \right)^{k} \left( \dfrac{\delta}{4 \beta^2} + \dfrac{\varepsilon}{2 \alpha} \right), \\
    &= \left( 1 - \dfrac{\alpha}{4 \beta} \right)^t  \| \w_{0} - \wstar \|^2 + \dfrac{1}{\frac{\alpha}{4 \beta}} \left( \dfrac{\delta}{4 \beta^2} + \dfrac{\varepsilon}{2 \alpha} \right), \\
    &= \left( 1 - \dfrac{\alpha}{4 \beta} \right)^t  \| \w_{0} - \wstar \|^2 + \dfrac{4 \beta}{\alpha} \left( \dfrac{\delta}{4 \beta^2} + \dfrac{\varepsilon}{2 \alpha} \right).
    \stepcounter{equation}\tag{\theequation}
\end{align*}
\end{leftlinebox}
Given an arbitrary $\w \in \mathbb{R}^p$, we now wish to relate the distance $\|\w - \wstar \|^2$ to the function values $f(\w) - f(\wstar)$.
\begin{leftlinebox}
Since each $\ell_z$ is $\alpha$-strongly convex and $\beta$-smooth, so is $f = \mathbb{E}_z[\ell_z]$.
We introduce $\wbar$ the minimizer of $f$ on its unconstrained domain $\mathbb{R}^p$.
Then we can write that for any $\w \in \mathbb{R}^p$:
\begin{align*}
&\kern-1em f(\w) - f(\wstar) \\
    &\leq f(\w) - f(\wbar), \eqcomment{$f(\wbar) \leq f(\wstar)$} \\
    &\leq \nabla f(\wbar)^\top(\w - \wbar) + \dfrac{\beta}{2} \|\w - \wbar \|^2, \eqcomment{$f$ is $\beta$-smooth}\\
    &= \dfrac{\beta}{2} \|\w - \wbar \|^2, \eqcomment{$\nabla f(\wbar) = \bm{0}$} \\
    &\leq \beta (\|\w - \wstar \|^2 + \|\wstar - \wbar \|^2 ), \eqcomment{Lemma \ref{lemma:parallelogram_inequality}} \\
    &\leq \beta \|\w - \wstar \|^2 + \dfrac{2 \beta}{\alpha} \left(f(\wstar) - f(\wbar) \right), \eqcomment{$f$ is $\alpha$-strongly convex} \\
    &\leq \beta \|\w - \wstar \|^2 + \dfrac{2 \beta}{\alpha} f(\wstar), \eqcomment{$0 \leq f(\wbar)$} \\
    &\leq \beta \|\w - \wstar \|^2 + 2\dfrac{\beta \varepsilon}{\alpha}, \eqcomment{definition of $\varepsilon$}
    \label{eq:strgly_cvx_iterate_distance_to_function_distance} \stepcounter{equation}\tag{\theequation}
\end{align*}
\end{leftlinebox}

Taking the expectation, we can combine the results to obtain the final result:
\begin{leftlinebox}[nobreak=false]
\begin{align*}
\kern-1em \mathbb{E} [f(\w_{t+1})] - f(\wstar)
&\leq \beta \mathbb{E} [\|\w_{t+1} - \wstar \|^2] + 2\dfrac{\beta \varepsilon}{\alpha}, \\
&\leq \beta \left( \left( 1 - \dfrac{\alpha}{4 \beta} \right)^t  \| \w_{0} - \wstar \|^2 + \dfrac{4 \beta}{\alpha} \left(\dfrac{\delta}{4 \beta^2} + \dfrac{\varepsilon}{2 \alpha} \right) \right) + 2\dfrac{\beta \varepsilon}{\alpha}, \\
&= \beta \left( 1 - \dfrac{\alpha}{4 \beta} \right)^t  \| \w_{0} - \wstar \|^2 + \dfrac{4 \beta}{\alpha} \left(\dfrac{\delta}{4 \beta} + \dfrac{\varepsilon \beta}{2 \alpha} \right) + 2\dfrac{\beta \varepsilon}{\alpha}, \\
&= \beta \left( 1 - \dfrac{\alpha}{4 \beta} \right)^t  \| \w_{0} - \wstar \|^2 + \dfrac{\delta}{\alpha} + 2 \dfrac{\beta}{\alpha} \varepsilon + 2 \dfrac{\beta^2}{\alpha^2} \varepsilon, \\
&\leq \beta \exp\left(- \dfrac{\alpha t}{4 \beta} \right)  \| \w_{0} - \wstar \|^2 + \dfrac{\delta}{\alpha} + 2 \dfrac{\beta}{\alpha} \varepsilon + 2 \dfrac{\beta^2}{\alpha^2} \varepsilon.
\end{align*}
\end{leftlinebox}

\end{proof}

\begin{restatable}{theorem}{thaligstronglycvxsmalleta}\label{th:alig_cvx_strongly_small_eta}
We assume that $\Omega$ is a convex set, and that for every $z \in \Z$, $\ell_z$ is $\alpha$-strongly convex and $\beta$-smooth.
Let $\wstar$ be a solution of (\ref{eq:main_problem}) such that $\forall z \in \Z, \: \ell_z(\wstar) \leq \varepsilon$, and suppose that $\delta > 2 \beta \varepsilon$.
Further assume that $\eta \leq \frac{1}{2 \beta}$.
Then if we apply ALI-G with a maximal learning-rate of $\eta$ to $f$, we have:
\begin{equation}
\mathbb{E}[f(\w_{T+1})] - \fstar
    \leq \beta \exp \left(\dfrac{-\alpha \eta T }{2} \right)  \| \w_{0} - \wstar \|^2 + \dfrac{\delta}{ \alpha} + \dfrac{4 \varepsilon \beta}{\alpha}.
\end{equation}
\end{restatable}

\begin{proof}
Re-using inequalities (\ref{eq:smooth_iterate_smalleta_bound}) and (\ref{eq:smooth_smalleta_bound_first_term}) from the proof of Theorem \ref{th:alig_cvx_smooth_small_eta}, we can write:
\begin{equation}
\begin{split}
\kern-1em \| \w_{t+1} - \wstar \|^2
    &\leq  \| \w_{t} - \wstar \|^2 - \eta (\ell_{z_t}(\w_t) - \ell_{z_t}(\wstar)) + \dfrac{\eta \delta}{2 \beta} + \gamma_t \ell_{z_t}(\wstar), \\
    &\leq  \| \w_{t} - \wstar \|^2 - \eta (\ell_{z_t}(\w_t) - \ell_{z_t}(\wstar)) + \dfrac{\eta \delta}{2 \beta} + \eta \varepsilon \\
    &\quad \eqcomment{using $\gamma_t \leq \eta$, $0 \leq \ell_{z_t}(\wstar) \leq \varepsilon$}.
\end{split}
\end{equation}
Taking the expectation over $z_t|z_{t-1}$, we obtain:
\begin{equation}
\kern-1em \mathbb{E}_{z_t|z_{t-1}} [\| \w_{t+1} - \wstar \|^2]
    \leq \| \w_{t} - \wstar \|^2 - \eta (f(\w_t) - f(\wstar)) + \dfrac{\eta \delta}{2 \beta} + \eta \varepsilon.
\end{equation}
Therefore, we can write:
\begin{equation}
\begin{split}
\kern-1em \mathbb{E}_{z_t|z_{t-1}} [\| \w_{t+1} - \wstar \|^2]
    &\leq  \| \w_{t} - \wstar \|^2 - \dfrac{\alpha \eta}{2} \| \w_t - \wstar \|^2 + \dfrac{\eta \delta}{2 \beta} + \eta \varepsilon, \eqcomment{Lemma \ref{lemma:strglycvx_fun_bound}} \\
    &= \left(1 - \dfrac{\alpha \eta}{2} \right)  \| \w_{t} - \wstar \|^2 + \dfrac{\eta \delta}{2 \beta} + \eta \varepsilon.
\end{split}
\end{equation}
Then a trivial induction gives that:
\begin{equation}
\begin{split}
\mathbb{E}[\| \w_{T+1} - \wstar \|^2]
    &\leq \left(1 - \dfrac{\alpha \eta}{2} \right)^T  \| \w_{0} - \wstar \|^2 + \left( \dfrac{\eta \delta}{2 \beta} + \eta \varepsilon \right)\sum\limits_{t=0}^T \left(1 - \dfrac{\alpha \eta}{2} \right)^t, \\
    &\leq \left(1 - \dfrac{\alpha \eta}{2} \right)^T  \| \w_{0} - \wstar \|^2 + \left( \dfrac{\eta \delta}{2 \beta} + \eta \varepsilon \right)\sum\limits_{t=0}^\infty \left(1 - \dfrac{\alpha \eta}{2} \right)^t, \\
    &= \left(1 - \dfrac{\alpha \eta}{2} \right)^T  \| \w_{0} - \wstar \|^2 + \left( \dfrac{\eta \delta}{2 \beta} + \eta \varepsilon \right) \dfrac{1}{1 - \left(1 - \dfrac{\alpha \eta}{2} \right)}, \\
    &= \left(1 - \dfrac{\alpha \eta}{2} \right)^T  \| \w_{0} - \wstar \|^2 + \dfrac{\delta}{\alpha \beta} + \dfrac{2 \varepsilon}{\alpha}. \\
\end{split}
\end{equation}
We now re-use the inequality (\ref{eq:strgly_cvx_iterate_distance_to_function_distance}) in expectation to write:
\begin{equation}
\begin{split}
\mathbb{E}[f(\w_{T+1})] - \fstar
    &\leq \beta \mathbb{E}[\| \w_{T+1} - \wstar \|^2] + \dfrac{2 \beta \varepsilon}{\alpha}, \\
    &\leq \beta \left(1 - \dfrac{\alpha \eta}{2} \right)^T  \| \w_{0} - \wstar \|^2 + \dfrac{\delta}{\alpha} + \dfrac{4 \varepsilon \beta}{\alpha}, \\
    &\leq \beta \exp \left(\dfrac{-\alpha \eta T }{2} \right)  \| \w_{0} - \wstar \|^2 + \dfrac{\delta}{\alpha} + \dfrac{4 \varepsilon \beta}{\alpha}. \\
\end{split}
\end{equation}
\end{proof}

\section{Detailed Non-Convex Results}
\label{app:sec:detailed_noncvx_results}

The Restricted Secant Inequality (RSI) is a milder assumption than convexity.
It can be defined as follows:
\begin{definition}
    Let $f: \mathbb{R}^p \to \mathbb{R}$ be a lower-bounded differentiable function achieving its minimum at $\wstar$.
    We say that $f$ satisfies the RSI if there exists $\alpha > 0$ such that:
    \begin{equation}
        \forall \vw \in \mathbb{R}^p, \: \nabla f (\w)^\top (\vw - \wstar) \geq \alpha \| \vw - \wstar \|^2.
    \end{equation}
\end{definition}

The RSI is sometimes used to prove convergence of optimization algorithms without assuming convexity \citep{Vaswani2019a}.

As we prove below, the Polyak step-size may fail to converge under the RSI assumption, even in a non-stochastic setting with the exact minimum known.

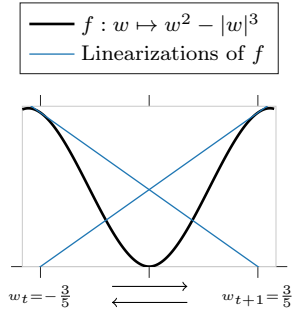
\begin{figure}[ht]
\centering
\footnotesize
\input{include/rsi_plot_arxiv.tex}
\caption{\em Illustration of the function $f$, which satisfies the RSI. When starting at $w =- 3 / 5$, gradient descent with the Polyak step-size oscillates between $w= - 3 / 5$ and $w = 3 / 5$.}
\label{fig:rsi2}
\end{figure}

\begin{restatable}{proposition}{rsi_example}\label{prop:rsi_example}
Let $f: w \in [\frac{-3}{5}; \frac{3}{5}] \mapsto w^2 - |w|^3$.
Then $f$ satisfies the RSI with $\alpha = \frac{1}{5}$.
\end{restatable}

\begin{proof}
First we note that $f$ achieves its minimum at $w_{\star} = 0$, and that $f(w_{\star}) = 0$.
In addition, we introduce the sign function $\sigma(w)$, which is equal to $1$ if $w \geq 0$, and $-1$ otherwise.
Now let $w \in [\frac{-3}{5}; \frac{3}{5}]$.
Then we have that:
\begin{equation}
\begin{split}
\nabla f (w) (w - w_{\star}) - \frac{1}{5} (w -  w_{\star})^2,
    &= (2 w - 3 \sigma(w) w^2) (w - 0) - \frac{1}{5} (w -  0)^2, \\
    &= \frac{9}{5} w^2 - 3 \sigma(w) w^3, \\
    &= 3 w^2 (\frac{3}{5} - \sigma(w) w), \\
    &\geq 0.
\end{split}
\end{equation}
\end{proof}

\begin{restatable}{proposition}{polyak_oscillation}\label{prop:polyak_oscillation}
Assume that we apply the Polyak step-size to $f: w \in [\frac{-3}{5}; \frac{3}{5}] \mapsto w^2 - |w|^3$, starting from the initial point $w_0 = -3/5$.
Then the iterates oscillate between $-3/5$ and $3/5$.
\end{restatable}

\begin{proof}
We show that, starting with $w_0 = -\frac{3}{5}$, we obtain $w_1 = \frac{3}{5}$.
This will prove oscillation of the iterates by symmetry of the problem.
Since $w_0 = \frac{-3}{5}$, we have $f(w_0) = \frac{9}{25} - \frac{27}{125} = \frac{18}{125}$.
Furthermore, $\nabla f(w_0) = 2 (\frac{-3}{5}) + 3 (\frac{9}{25}) = \frac{-3}{25}$.
Therefore:
\begin{equation}
\begin{split}
w_1
    &= w_0 - \dfrac{f(w_0)}{(\nabla f (w_0))^2} \nabla f (w_0), \\
    &= w_0 - \dfrac{f(w_0)}{\nabla f (w_0)}, \\
    &= \frac{-3}{5} + \dfrac{\frac{18}{125}}{\frac{3}{25}}, \\
    &= \frac{-3}{5} + \frac{6}{5}, \\
    &= \frac{3}{5}.
\end{split}
\end{equation}
\end{proof}

\begin{restatable}{theorem}{thrsismooth}\label{th:alig_rsi}
We assume that $\Omega = \mathbb{R}^p$, and that for every $z \in \Z$, $\ell_z$ is $\beta$-smooth and satisfies the RSI with constant $\alpha$.
We further assume that there exists $\wstar$ a solution of (\ref{eq:main_problem}) such that $\forall z \in \Z, \: \ell_z(\wstar) = 0$.
Let $\eta$ be such that $\frac{1}{2 \beta} \leq \eta \leq \frac{2 \alpha}{\beta^2}$.
Then if we apply ALI-G with a maximal learning-rate of $\eta$ to $f$, we have:
\begin{equation}
f(\w_{T+1}) - \fstar
    \leq  \frac{\beta}{2} \exp \left( \left(-\dfrac{\alpha}{\beta} + \dfrac{\eta \beta}{2} \right) T \right) \| \w_{0} - \wstar \|^2.
\end{equation}

Note: this result assumes perfect interpolation, and thus we set $\delta = 0$ (no small constant for numerical stability).
\end{restatable}
\begin{proof}
We consider the update at time $t$, which we condition on the draw of $z_t \in \Z$.
Since we consider $\delta=0$, we have $\gamma_t = \min \left\{\frac{\ell_{z_t}(\w_t)}{\| \nabla \ell_{z_t}(\w_t) \|^2 }, \eta \right\}$. We suppose $\nabla \ell_{z_t}(\w_t) \neq \bm{0}$.
\begin{equation} \label{eq:rsi}
\begin{split}
\| \w_{t+1} - \wstar \|^2
&= \| \Pi_\Omega(\w_t - \gamma_t \nabla \ell_{z_t}(\w_t)) - \wstar \|^2, \\
&\leq \| \w_t - \gamma_t \nabla \ell_{z_t}(\w_t) - \wstar \|^2, \eqcomment{$\Pi_\Omega$ projection} \\
&= \| \w_{t} - \wstar \|^2 - 2 \gamma_t \nabla \ell_{z_t}(\w_t)^\top(\w_{t} - \wstar) + \gamma_t^2 \| \nabla \ell_{z_t}(\w_t) \|^2, \\
&\leq \| \w_{t} - \wstar \|^2 - 2 \gamma_t \nabla \ell_{z_t}(\w_t)^\top(\w_{t} - \wstar) + \gamma_t \ell_{z_t}(\w_t), \eqcomment{since $\gamma_t \leq \frac{\ell_{z_t}(\w_t)}{\| \nabla \ell_{z_t}(\w_t) \|^2}$}\\
&\leq \| \w_{t} - \wstar \|^2 - 2 \gamma_t \nabla \ell_{z_t}(\w_t)^\top(\w_{t} - \wstar) + \gamma_t \dfrac{\beta}{2} \| \w_{t} - \wstar \|^2, \eqcomment{Lemma 3.4 of \cite{Bubeck2015}} \\
&\leq \| \w_{t} - \wstar \|^2 - 2 \gamma_t \alpha \| \w_{t} - \wstar \|^2 + \gamma_t \dfrac{\beta}{2} \| \w_{t} - \wstar \|^2, \eqcomment{RSI inequality} \\
&= \Big( 1 - 2 \gamma_t \alpha + \gamma_t \dfrac{\beta}{2} \Big) \| \w_{t} - \wstar \|^2.
\end{split}
\end{equation}
Since we know that $\frac{\ell_{z_t}(\w_t)}{\| \nabla \ell_{z_t}(\w_t) \|^2} \geq \frac{1}{2 \beta}$ (Lemma \ref{lemma:smooth_bound}) and $\eta \geq \frac{1}{2 \beta}$, we have that $\gamma_t \geq \frac{1}{2 \beta}$.
Then, using both $\gamma_t \geq \frac{1}{2 \beta}$ and $\gamma_t \leq \eta$, we can write:
\begin{equation}
\| \w_{t+1} - \wstar \|^2
\leq \left(1 - \dfrac{\alpha}{\beta} + \dfrac{\eta \beta}{2} \right) \| \w_{t} - \wstar \|^2.
\end{equation}
With a trivial induction we obtain:
\begin{equation}
\begin{split}
\| \w_{T+1} - \wstar \|^2
&\leq \left(1 - \dfrac{\alpha}{\beta} + \dfrac{\eta \beta}{2} \right)^T \| \w_{0} - \wstar \|^2, \\
&\leq \exp \left( \left(-\dfrac{\alpha}{\beta} + \dfrac{\eta \beta}{2} \right) T \right) \| \w_{0} - \wstar \|^2. \\
\end{split}
\end{equation}
Since $f$ is $\beta$-smooth and the problem is unconstrained by assumption, we have $f(\w_{T+1}) \leq \frac{\beta}{2} \| \w_{T+1} - \wstar \|^2$ (by Lemma 3.4 of \cite{Bubeck2015}), and we obtain the desired result.
\end{proof}

\begin{restatable}{theorem}{thrsismooth_smalleta}\label{th:alig_rsi_small_eta}
We assume that $\Omega = \mathbb{R}^p$, and that for every $z \in \Z$, $\ell_z$ is $\beta$-smooth and satisfies the RSI with constant $\alpha$.
We further assume that there exists $\wstar$ a solution of (\ref{eq:main_problem}) such that $\forall z \in \Z, \: \ell_z(\wstar) = 0$.
Let $\eta$ be such that $0 < \eta \leq \frac{1}{2 \beta}$.
Then if we apply ALI-G with a maximal learning-rate of $\eta$ to $f$, we have:
\begin{equation}
f(\w_{T+1}) - \fstar
    \leq  \frac{\beta}{2} \exp \left( \left(- \eta \left(2 \alpha - \frac{\beta}{2} \right) \right) T \right) \| \w_{0} - \wstar \|^2.
\end{equation}

Note: this result assumes perfect interpolation, and thus we set $\delta = 0$ (no small constant for numerical stability).
\end{restatable}
\begin{proof}
We consider the update at time $t$, which we condition on the draw of $z_t \in \Z$.
Since we consider $\delta=0$, we have $\gamma_t = \min \left\{\frac{\ell_{z_t}(\w_t)}{\| \nabla \ell_{z_t}(\w_t) \|^2 }, \eta \right\}$.
We suppose $\nabla \ell_{z_t}(\w_t) \neq \bm{0}$.
We re-use equation (\ref{eq:rsi}) to write:
\begin{equation}
\| \w_{t+1} - \wstar \|^2 \leq \Big( 1 - 2 \gamma_t \alpha + \gamma_t \dfrac{\beta}{2} \Big) \| \w_{t} - \wstar \|^2.
\end{equation}
Since we know that $\frac{\ell_{z_t}(\w_t)}{\| \nabla \ell_{z_t}(\w_t) \|^2} \geq \frac{1}{2 \beta}$ (Lemma \ref{lemma:smooth_bound}) and $\eta \leq \frac{1}{2 \beta}$, we have that $\gamma_t = \eta$ necessarily.
Thus we obtain:
\begin{equation*}
\| \w_{t+1} - \wstar \|^2
\leq \Big( 1 - 2 \eta \alpha + \eta \dfrac{\beta}{2} \Big) \| \w_{t} - \wstar \|^2.
\end{equation*}
With a trivial induction we obtain:
\begin{align*}
\| \w_{T+1} - \wstar \|^2
&\leq \left(1 - \eta \left(2 \alpha - \frac{\beta}{2} \right) \right)^T \| \w_{0} - \wstar \|^2, \\
&\leq \exp \left( \left(- \eta \left(2 \alpha - \frac{\beta}{2} \right) \right) T \right) \| \w_{0} - \wstar \|^2. \\
\end{align*}
Since $f$ is $\beta$-smooth and the problem is unconstrained by assumption, we have $f(\w_{T+1}) \leq \frac{\beta}{2} \| \w_{T+1} - \wstar \|^2$ (by Lemma 3.4 of \cite{Bubeck2015}), and we obtain the desired result.
\end{proof}

\vfill

\section{Additional Experimental Details}

\subsection{Standard Deviation of CIFAR Results}

\input{std_results}

\subsection{Additional Details About Training Protocol on ImageNet}

\paragraph{Data Processing.}
We use 1.23M images for training.
As mentioned in the paper, we do not use any data augmentation on this task.
Our data processing can be described as follows.
Each training image is resized so that its smaller dimension is of 224 pixels, after which we take a centered square crop of 224 by 224.
The cropped image is then centered and normalized per channel (for this, the mean and standard deviation per channel is computed across all training images), before being fed to the neural network.

\paragraph{Loss Function.}
We use the top-k truncated cross-entropy \cite{Lapin2016} as our loss function for training the model on ImageNet.
In particular, we use $k=5$ so that we optimize for the commonly used top-5 error, and we use the default temperature parameter $\tau=1$.

Our PyTorch code re-uses the implementation from \url{https://github.com/locuslab/lml}.

%% file: std_results.tex
\begin{table}[H]
\centering
\begin{tabular}{llrr}
\toprule
Task & Optimizer &   Avg &   Std \\
\midrule
DN10 &     ADAMW &  92.6 &  0.08 \\
DN10 &      ALIG &  95.0 &  0.16 \\
DN10 &   AMSGRAD &  91.7 &  0.25 \\
DN10 &       DFW &  94.6 &  0.22 \\
DN10 &    L4ADAM &  90.8 &  0.09 \\
DN10 &     L4MOM &  91.9 &  0.17 \\
DN10 &       SGD &  95.1 &  0.21 \\
DN10 &      YOGI &  92.1 &  0.38 \\
\midrule
DN100 &     ADAMW &  69.5 &  0.54 \\
DN100 &      ALIG &  76.3 &  0.14 \\
DN100 &   AMSGRAD &  69.4 &  0.41 \\
DN100 &       DFW &  73.2 &  0.29 \\
DN100 &    L4ADAM &  60.5 &  0.64 \\
DN100 &     L4MOM &  62.6 &  1.98 \\
DN100 &       SGD &  76.3 &  0.22 \\
DN100 &      YOGI &  69.6 &  0.34 \\
\midrule
WRN10 &     ADAMW &  92.1 &  0.34 \\
WRN10 &      ALIG &  95.2 &  0.09 \\
WRN10 &   AMSGRAD &  90.8 &  0.31 \\
WRN10 &       DFW &  94.2 &  0.19 \\
WRN10 &    L4ADAM &  90.5 &  0.09 \\
WRN10 &     L4MOM &  91.6 &  0.24 \\
WRN10 &       SGD &  95.3 &  0.31 \\
WRN10 &      YOGI &  91.2 &  0.27 \\
\midrule
WRN100 &     ADAMW &  69.6 &  0.51 \\
WRN100 &      ALIG &  75.8 &  0.29 \\
WRN100 &   AMSGRAD &  68.7 &  0.70 \\
WRN100 &       DFW &  76.0 &  0.24 \\
WRN100 &    L4ADAM &  61.7 &  2.17 \\
WRN100 &     L4MOM &  61.4 &  0.86 \\
WRN100 &       SGD &  77.8 &  0.13 \\
WRN100 &      YOGI &  68.7 &  0.47 \\
\bottomrule
\end{tabular}
\caption{\em
    Test Accuracy (\%) on CIFAR including standard deviations.
    Each experiment was run three times.}
\label{tab:cifar_std}
\end{table}